\documentclass{article}
\usepackage{amsmath,amssymb,amsthm}
\usepackage{mathrsfs}
\usepackage{mathabx}\changenotsign
\usepackage{dsfont}
\usepackage{bbm}
\usepackage{scalerel}
\usepackage{algorithm}
\usepackage[noend]{algpseudocode}
\usepackage{nicefrac}
\usepackage{array}
\usepackage{booktabs} 

\usepackage[suppress]{color-edits} 
\addauthor{as}{purple} 
\addauthor{ks}{red}
\addauthor{ac}{red} 

\newcommand{\SGLD}{\textsc{SGLD}~}  

\newcommand{\IF}{\text{ if }{}}
\newcommand{\OTHERWISE}{\text{ otherwise }{}}
\newcommand{\OTHERWISEIF}{\text{ otherwise if }{}}
\newcommand{\FORALLTEXT}{\text{ for all }{}}
\newcommand{\DERIV}{\mathrm{d}}

\newcommand{\LCHAR}{\L\text{ }}

\newcommand{\GLDTEXT}{\textsc{GLD}}
\newcommand{\GLDTEXTSPACE}{\textsc{GLD }}
\newcommand{\SGLDTEXT}{\textsc{SGLD}}
\newcommand{\SGLDTEXTSPACE}{\textsc{SGLD }}
\newcommand{\GFTEXT}{\textsc{GF}}
\newcommand{\GFTEXTSPACE}{\textsc{GF }}

\newcommand{\OPNORM}{\mathrm{op}{}}

\newcommand{\vecW}{\mathbf{w}}
\newcommand{\vecZ}{\mathbf{z}}
\newcommand{\vecV}{\mathbf{v}}
\newcommand{\vecB}{\mathbf{B}}
\newcommand{\vecU}{\mathbf{u}}
\newcommand{\vecOrigin}{\vec{\mathbf{0}}}
\newcommand{\vecEps}{\pmb{\varepsilon}}

\newcommand{\CPI}{\normalfont \textsf{C}_{\textsc{pi}}}

\usepackage{etoolbox}
\usepackage{xfrac}
\newcommand{\arxiv}[1]{\iftoggle{neurips}{}{#1}}
\newcommand{\neurips}[1]{\iftoggle{neurips}{#1}{}}
\newtoggle{neurips}

\usepackage{adjustbox}
\arxiv{\date{}} 
\arxiv{
\usepackage[letterpaper, left=1in, right=1in, top=1in, bottom=1in]{geometry}
\PassOptionsToPackage{hypertexnames=false}{hyperref}  
\usepackage{parskip}
\usepackage[dvipsnames]{xcolor}
\usepackage[colorlinks=true, linkcolor=blue!70!black, citecolor=blue!70!black,urlcolor=black,breaklinks=true]{hyperref}
\hypersetup{linktoc=all}

\usepackage{microtype}
\usepackage{hhline}

\makeatletter
\newcommand{\neutralize}[1]{\expandafter\let\csname c@#1\endcsname\count@}
\makeatother

\usepackage{algorithm}

\arxiv{
\usepackage{natbib}
\bibliographystyle{plainnat}
\bibpunct{(}{)}{;}{a}{,}{,}
}

\usepackage{amsthm}
\usepackage{mathtools}
\usepackage{amsmath}
\usepackage{bbm}
\usepackage{amsfonts}
\usepackage{amssymb}

\usepackage{xpatch}

\usepackage{thmtools}
\usepackage{thm-restate}
\declaretheorem[name=Theorem,parent=section]{theorem}
\declaretheorem[name=Lemma,parent=section]{lemma}
\declaretheorem[name=Assumption, parent=section]{assumption}
\declaretheorem[name=Condition, parent=section]{condition}
\declaretheorem[qed=$\triangleleft$,name=Example,style=definition, parent=section]{example}
\declaretheorem[name=Remark, parent=section]{remark}
\declaretheorem[name=Proposition, parent=section]{proposition}

\makeatletter
  \renewenvironment{proof}[1][Proof]%
  {%
   \par\noindent{\bfseries\upshape {#1.}\ }%
  }%
  {\qed\newline}
  \makeatother

\newtheorem{definition}{Definition}[section]

\xpatchcmd{\proof}{\itshape}{\normalfont\proofnameformat}{}{}

\usepackage[nameinlink,capitalize]{cleveref}

\renewcommand{\eqref}[1]{\texorpdfstring{\hyperref[#1]{Eq.~(\ref*{#1})}}{Eq.~(\ref*{#1})}}

\crefformat{equation}{#2(#1)#3}
\Crefformat{equation}{#2(#1)#3}

\Crefformat{figure}{#2Figure #1#3}
\Crefname{assumption}{Assumption}{Assumptions}
\Crefformat{assumption}{#2Assumption #1#3}
\Crefname{subsubsection}{Section}{Sections}
\crefformat{subsubsection}{#2Section #1#3}
\Crefformat{subsubsection}{#2Section #1#3}

\usepackage{crossreftools}
\pdfstringdefDisableCommands{%
    \let\Cref\crtCref
    \let\cref\crtcref
}

\usepackage{xparse}

\ExplSyntaxOn
\DeclareDocumentCommand{\XDeclarePairedDelimiter}{mm}
 {
  \__egreg_delimiter_clear_keys: 
  \keys_set:nn { egreg/delimiters } { #2 }
  \use:x 
   {
    \exp_not:n {\NewDocumentCommand{#1}{sO{}m} }
     {
      \exp_not:n { \IfBooleanTF{##1} }
       {
        \exp_not:N \egreg_paired_delimiter_expand:nnnn
         { \exp_not:V \l_egreg_delimiter_left_tl }
         { \exp_not:V \l_egreg_delimiter_right_tl }
         { \exp_not:n { ##3 } }
         { \exp_not:V \l_egreg_delimiter_subscript_tl }
       }
       {
        \exp_not:N \egreg_paired_delimiter_fixed:nnnnn 
         { \exp_not:n { ##2 } }
         { \exp_not:V \l_egreg_delimiter_left_tl }
         { \exp_not:V \l_egreg_delimiter_right_tl }
         { \exp_not:n { ##3 } }
         { \exp_not:V \l_egreg_delimiter_subscript_tl }
       }
     }
   }
 }

\keys_define:nn { egreg/delimiters }
 {
  left      .tl_set:N = \l_egreg_delimiter_left_tl,
  right     .tl_set:N = \l_egreg_delimiter_right_tl,
  subscript .tl_set:N = \l_egreg_delimiter_subscript_tl,
 }

\cs_new_protected:Npn \__egreg_delimiter_clear_keys:
 {
  \keys_set:nn { egreg/delimiters } { left=.,right=.,subscript={} }
 }

\cs_new_protected:Npn \egreg_paired_delimiter_expand:nnnn #1 #2 #3 #4
 {
  \mathopen{}
  \mathclose\c_group_begin_token
   \left#1
   #3
   \group_insert_after:N \c_group_end_token
   \right#2
   \tl_if_empty:nF {#4} { \c_math_subscript_token {#4} }
 }
\cs_new_protected:Npn \egreg_paired_delimiter_fixed:nnnnn #1 #2 #3 #4 #5
 {
  \mathopen{#1#2}#4\mathclose{#1#3}
  \tl_if_empty:nF {#5} { \c_math_subscript_token {#5} }
 }
\ExplSyntaxOff

\XDeclarePairedDelimiter{\supnorm}{
  left=\lVert,
  right=\rVert,
  subscript=\infty
  }
} 

\neurips{
\usepackage{neurips_2024}
\usepackage{xcolor}
\usepackage[backref]{hyperref}
\hypersetup{ 
	colorlinks,
	linkcolor={red!60!black},
	citecolor={green!60!black},
	urlcolor={blue!60!black}
}
\newtheorem{theorem}{Theorem}
\newtheorem{lemma}{Lemma}
\newtheorem{remark}{Remark}
\newtheorem{corollary}{Corollary}

\newtheorem{definition}{Definition}
}

\neurips{
\bibliographystyle{abbrv}
\title{Langevin Dynamics: A Unified Perspective on  Optimization via Lyapunov Potentials}
}
\arxiv{\title{Langevin Dynamics: A Unified Perspective on \\ Optimization via Lyapunov Potentials}}
\arxiv{
\author{%
August Y. Chen$^{\dag}$  \qquad Ayush Sekhari\(^{\diamondsuit}\)  \qquad Karthik Sridharan$^{\dag}$ 
\vspace{10pt} 
\\
\small{$^\dag$Cornell University \quad \(^\diamondsuit\)MIT} 
}}

\usepackage{amsmath} 

\usepackage{mathtools}
{\usepackage{amsthm}}
\usepackage{amsmath}
\usepackage{amsxtra}
\usepackage{enumitem}      
\usepackage{mathtools}
\usepackage{bbm}
\usepackage{amsfonts}
\usepackage{amssymb}
\let\vec\undefined
\usepackage{MnSymbol} 

\usepackage{xpatch}

\newcommand\blfootnote[1]{%
  \begingroup
  \renewcommand\thefootnote{}\footnote{#1}%
  \addtocounter{footnote}{-1}%
  \endgroup
}

{\theoremstyle{definition}  
\theoremstyle{plain}
}

{} 

{\newtheorem{corollary}{Corollary}} 

\newtheorem*{theorem*}{Theorem}

\xpatchcmd{\proof}{\itshape}{\normalfont\proofnameformat}{}{}
\newcommand{\proofnameformat}{\bfseries}

\usepackage{prettyref}
\newcommand{\pref}[1]{\prettyref{#1}}

\newcommand{\savehyperref}[2]{\texorpdfstring{\hyperref[#1]{#2}}{#2}}
\newrefformat{eq}{\savehyperref{#1}{\textup{(\ref*{#1})}}}
\newrefformat{eqn}{\savehyperref{#1}{Equation~\ref*{#1}}}
\newrefformat{con}{\savehyperref{#1}{Conjecture~\ref*{#1}}}
\newrefformat{lem}{\savehyperref{#1}{Lemma~\ref*{#1}}}
\newrefformat{def}{\savehyperref{#1}{Definition~\ref*{#1}}}
\newrefformat{line}{\savehyperref{#1}{line~\ref*{#1}}}
\newrefformat{thm}{\savehyperref{#1}{Theorem~\ref*{#1}}}
\newrefformat{corr}{\savehyperref{#1}{Corollary~\ref*{#1}}}
\newrefformat{sec}{\savehyperref{#1}{Section~\ref*{#1}}}
\newrefformat{app}{\savehyperref{#1}{Appendix~\ref*{#1}}}
\newrefformat{ass}{\savehyperref{#1}{Assumption~\ref*{#1}}}
\newrefformat{ex}{\savehyperref{#1}{Example~\ref*{#1}}}
\newrefformat{fig}{\savehyperref{#1}{Figure~\ref*{#1}}}
\newrefformat{alg}{\savehyperref{#1}{Algorithm~\ref*{#1}}}
\newrefformat{rem}{\savehyperref{#1}{Remark~\ref*{#1}}}
\newrefformat{conj}{\savehyperref{#1}{Conjecture~\ref*{#1}}}
\newrefformat{prop}{\savehyperref{#1}{Proposition~\ref*{#1}}}
\newrefformat{proto}{\savehyperref{#1}{Protocol~\ref*{#1}}}
\newrefformat{prob}{\savehyperref{#1}{Problem~\ref*{#1}}}
\newrefformat{claim}{\savehyperref{#1}{Claim~\ref*{#1}}}
\newrefformat{subsec}{\savehyperref{#1}{Subsection~\ref*{#1}}}
\newrefformat{sec}{\savehyperref{#1}{Section~\ref*{#1}}}
\newrefformat{ineq}{\savehyperref{#1}{Inequality~\ref*{#1}}}

\newcommand\numberthis{\addtocounter{equation}{1}\tag{\theequation}}
\allowdisplaybreaks

\DeclarePairedDelimiter{\abs}{\lvert}{\rvert} 
\DeclarePairedDelimiter{\brk}{[}{]}
\DeclarePairedDelimiter{\crl}{\{}{\}}
\DeclarePairedDelimiter{\prn}{(}{)}
\DeclarePairedDelimiter{\nrm}{\|}{\|}
\DeclarePairedDelimiter{\tri}{\langle}{\rangle}

\DeclarePairedDelimiter{\floor}{\lfloor}{\rfloor}

\def\ddefloop#1{\ifx\ddefloop#1\else\ddef{#1}\expandafter\ddefloop\fi}
\def\ddef#1{\expandafter\def\csname bb#1\endcsname{\ensuremath{\mathbb{#1}}}}
\ddefloop ABCDEFGHIJKLMNOPQRSTUVWXYZ\ddefloop
\def\ddefloop#1{\ifx\ddefloop#1\else\ddef{#1}\expandafter\ddefloop\fi}
\def\ddef#1{\expandafter\def\csname b#1\endcsname{\ensuremath{\mathbf{#1}}}}
\ddefloop ABCDEFGHIJKLMNOPQRSTUVWXYZ\ddefloop
\def\ddef#1{\expandafter\def\csname c#1\endcsname{\ensuremath{\mathcal{#1}}}}
\ddefloop ABCDEFGHIJKLMNOPQRSTUVWXYZ\ddefloop
\def\ddef#1{\expandafter\def\csname h#1\endcsname{\ensuremath{\widehat{#1}}}}
\ddefloop ABCDEFGHIJKLMNOPQRSTUVWXYZabcdefghijklmnopqrsuvwxyz\ddefloop    
\def\ddef#1{\expandafter\def\csname hc#1\endcsname{\ensuremath{\widehat{\mathcal{#1}}}}}
\ddefloop ABCDEFGHIJKLMNOPQRSTUVWXYZ\ddefloop
\def\ddef#1{\expandafter\def\csname t#1\endcsname{\ensuremath{\widetilde{#1}}}}
\ddefloop ABCDEFGHIJKLMNOPQRSTUVWXYZ\ddefloop
\def\ddef#1{\expandafter\def\csname tc#1\endcsname{\ensuremath{\widetilde{\mathcal{#1}}}}}
\ddefloop ABCDEFGHIJKLMNOPQRSTUVWXYZ\ddefloop

\newcommand{\ball}{\mathbb{B}}

\usepackage{tikz}

\newcommand{\grad}{\nabla}

\renewcommand{\epsilon}{\varepsilon}

\usepackage{tikz}
\usetikzlibrary{shapes.geometric, arrows.meta, positioning, bending}

\begin{document}
\maketitle 

\begin{abstract}
We study the problem of non-convex optimization using Stochastic Gradient Langevin Dynamics (SGLD).
SGLD is a natural and popular variation of stochastic gradient descent where at each step, appropriately scaled Gaussian noise is added.
To our knowledge, the only strategy for showing global convergence of SGLD on the loss function is to show that SGLD can sample from a stationary distribution which assigns larger mass when the function is small (the Gibbs measure), and then to convert these guarantees to optimization results. 

We employ a new strategy to analyze the convergence of SGLD to global minima, based on Lyapunov potentials and optimization. We convert the same mild conditions from previous works on SGLD into geometric properties based on Lyapunov potentials. 
This adapts well to the case with a stochastic gradient oracle, which is natural for machine learning applications where one wants to minimize population loss but only has access to stochastic gradients via minibatch training samples.
Here we provide 1) improved rates in the setting of previous works studying SGLD for optimization, 2) the first finite gradient complexity guarantee for SGLD where the function is Lipschitz and the Gibbs measure defined by the function satisfies a Poincar\'e Inequality, and 3) prove if continuous-time Langevin Dynamics succeeds for optimization, then discrete-time SGLD succeeds under mild regularity assumptions.
\end{abstract}
\label{sec:abstract}

\section{Introduction}\label{sec:introduction}
We consider the minimization problem
\[ \arg \min_{\vecW \in \mathbb{R}^d} F(\vecW).\]
More specifically we are interested in returning a vector $\vecW$ such that $F(\vecW)-\min_{\vecW}F(\vecW) \le \epsilon$ for some desired sub-optimality $\epsilon>0$. In Machine Learning (ML) settings, $F$ can be thought of as population loss and $\vecW$ as the parameters of a model we are using for the learning problem. Additionally, in ML one does not have direct access to $F$ but only via samples $\vecZ_1,\ldots,\vecZ_n$ drawn iid from some unknown but fixed distribution $D$ and we assume that $\mathbb{E}_{\vecZ\sim D}[f(\vecW;\vecZ)] =  F(\vecW)$. Here the $\vecZ_i$ can be thought of as input-output pairs and $f(\vecW;\vecZ)$ can be thought of as the loss of the model parametrized by weights $\vecW$ on instance $\vecZ$. When the objective function/loss function is differentiable (or sub-differentiable), then a common method of choice in practice is to use gradient descent (GD), stochastic gradient descent (SGD) and its variants to perform the optimization. To understand their properties theoretically, we aim to understand how many gradient computations are necessary to find an $\epsilon$-suboptimal $\vecW$, and for which functions $F$ this is possible. Under geometric conditions such as convexity, the properties of GD and SGD are well-understood. For convex functions, methods from acceleration to variance reduction have been developed to speed up runtime in a variety of settings. Matching lower and upper bounds exist for both exact and stochastic gradients for convex functions and smaller classes such as strongly convex functions \citep{bubeck2015convex}.\blfootnote{ 
{\scriptsize~~\texttt{Emails:}~\{ayc74@cornell.edu, sekhari@mit.edu, ks999@cornell.edu\}}} 

In recent years, machine learning has seen an explosion of success employing non-convex models. However, despite intensive study, the empirical success of optimizing non-convex functions to global optima is not at all well-understood theoretically. Beyond convexity, GD/SGD converges to global minima under general conditions such as Polyak-\L ojasiewicz (P\L) \citep{polyak1963gradient} \citep{lojasiewicz1963topological} and Kurdyka-\L ojasiewicz (K\L) \citep{kurdyka1998gradients} functions. Much more general geometric properties where GD/SGD can converge to global minima were found in \citep{priorpaper}, by considering what properties hold if and only if gradient flow succeeds. Additionally, researchers have proved GD/SGD with appropriate initialization can find global minima of particular non-convex problems such as matrix square root \citep{jain2017global} \citep{priorpaper}, matrix completion \citep{jin2016provable}, phase retrieval \citep{candes2015phase} \citep{chen2019gradient} \citep{tan2023online} \citep{priorpaper}, and dictionary learning \citep{arora2015simple}. 

While gradient descent/stochastic gradient descent has been shown to be successful in the aforementioned cases, there are well-known cases where GD/SGD does not work. 
A natural variant of gradient descent that is used for optimization is \textit{perturbed} gradient descent, where Gaussian noise is added to the iterates of stochastic gradient descent -- known as Langevin Dynamics -- is frequently analyzed. Formally, the iterates of \textit{Gradient Langevin Dynamics} (GLD) are given as follows:
\[ \vecW_{t+1} \leftarrow \vecW_t - \eta \nabla F(\vecW_t)+\sqrt{2\eta\beta^{-1}}\vecEps_t\numberthis\label{eq:SGLDiterates}.\]
Here $\eta>0$ is the step size, $\vecEps_t \sim \mathcal{N}(0,\bbI_d)$ is a $d$-dimensional standard Gaussian, and $\beta>0$ is the \textit{inverse temperature parameter} (when larger, noise is weighted less). When we use a stochastic gradient oracle $\nabla f(\vecW_t;\vecZ_t)$ in place of $\nabla F(\vecW_t)$, these iterates become those of \textit{Stochastic Gradient Langevin Dynamics} (SGLD). Langevin Dynamics has been shown to work in several highly non-convex settings where even gradient descent fails \citep{raginsky2017non}.

The continuous time version of \pref{eq:SGLDiterates} is the following Stochastic Differential Equation (SDE):
\[ \DERIV \vecW(t) = -\grad F(\vecW(t))\DERIV t+\sqrt{2\beta^{-1}}\DERIV \vecB(t)\numberthis\label{eq:LangevinSDE}.\]
Here $\vecB(t)$ denotes a standard Brownian motion in $\mathbb{R}^d$. This is known as the \textit{Langevin Diffusion}. Broadly, all of these recursions are known as \textit{Langevin Dynamics}. Note as $\beta\rightarrow\infty$, these iterates become exactly those of GD/SGD (for \pref{eq:SGLDiterates}) or Gradient Flow (for \pref{eq:LangevinSDE}). 

The only strategy in literature we know for proving \textit{global optimization guarantees} for GLD is by first showing sampling guarantees, and then connecting it back to optimization. 
Consider the Gibbs measure $\mu_{\beta} = e^{-\beta F}/Z$, where $Z$ denotes the partition function. 
It is well known that the continuous-time Langevin Diffusion with inverse temperature $\beta$ \pref{eq:LangevinSDE} converges to $\mu_{\beta}$ \citep{chiang1987diffusion} (although this is in fact false in discrete time).
When $\beta$ is sufficiently large, one can use this convergence to  get optimization guarantees. 
This was exactly the strategy of the works \citet{raginsky2017non}, \citet{xu2018global}, and \citet{zou2021faster}.
These works prove that under their conditions, this measure $\mu_{\beta}$ can be sampled from, and therefore non-convex optimization can succeed. Sampling from $\mu_{\beta}$ is generally known as \textit{Langevin Monte Carlo} (LMC).

The most general condition under which LMC has been proven to be successful is when $\mu_{\beta}$ satisfies a \textit{Poincar\'e Inequality} \citep{chewi21analysis}. 
A Poincar\'e Inequality is defined as follows:
\begin{definition}
A measure $\mu$ on $\mathbb{R}^d$ satisfies a Poincar\'e Inequality with Poincar\'e constant $\CPI(\mu)$ if for all infinitely differentiable functions $f:\mathbb{R}^d\rightarrow \mathbb{R}$, we have 
\[ \int_{\mathbb{R}^d} f^2 \DERIV\mu - \prn*{\int_{\mathbb{R}^d} f \DERIV\mu}^2 \le \CPI(\mu) \int_{\mathbb{R}^d} \nrm*{\grad f}^2 \DERIV\mu.\]
If the above is not satisfied, following the convention, we set  $\CPI(\mu)=\infty$. 
\end{definition}
There is evidence that in several cases, LMC does not succeed efficiently under looser conditions on $\mu_{\beta}$ such as a weak Poincar\'e Inequality \citep{mousavi2023towards}. Ultimately, a Poincar\'e Inequality being satisfied by $\mu_{\beta}$ is a geometric condition on $F$. 

A Poincar\'e Inequality is quite natural.
For instance, when $F$ is convex ($\mu_{\beta}$ is log-concave), $\mu_{\beta}$ satisfies a Poincar\'e Inequality. This is a famous result of Bobkov \citep{bobkov1999isoperimetric}.
But a Poincar\'e Inequality is in fact much more general. It is stable under bounded perturbations, hence covering a wide range of cases that log-concave measures (when $F$ is convex) does not (see Proposition 4.2.7, \citet{bakry2014analysis}). Thus starting with a convex $F_{\text{old}}$ and creating $F$ by arbitrarily perturbing $F_{\text{old}}$, perhaps creating \textit{exponentially many local minima or maxima}, the resulting measure $\mu_{\beta}=e^{-\beta F(\vecW)}/Z$ still satisfies a Poincar\'e Inequality (at the expense of worsening the Poincar\'e constant). Poincar\'e Inequalities are also stable under convolutions and mixtures, in the sense that for distributions which all satisfy a Poincar\'e Inequality, their mixture or convolutions between any two of them will also satisfy a Poincar\'e Inequality (again, at the expense of worsening the Poincar\'e constant; see Propositions 2.3.7 and 2.3.8, \citet{chewi2024log}). What a Poincar\'e Inequality fundamentally says is the existence of a spectral gap (in terms of variance) for the Langevin Diffusion \pref{eq:LangevinSDE}, see Theorem 4.2.5, \citet{bakry2014analysis}. It should be noted that the Poincar\'e constant $\CPI(\mu_{\beta})$ can behave in many ways, including as a constant. For example when $\mu_{\beta}$ is isotropic and $F$ is convex, the famous conjecture of Kannan, Lov\'asz, and Simonovits claims that $\CPI(\mu_{\beta}) = O(1)$ \citep{kannan1995isoperimetric}. This has been resolved up to polylog by the series of works \citet{leeeldan}, \citet{chen2021almost}, \citet{klartag2022bourgain}, \citet{jambulapati2022slightly}, and \citet{klartag2023logarithmic}, the best result known being $\CPI(\mu_{\beta}) = O(\sqrt{\log d})$ from \citet{klartag2023logarithmic}.
For further details on Poincar\'e Inequalities, we refer the reader to the excellent survey \citet{funcineqgeometrysurvey}. 

However, the approach of studying optimization guarantees for GLD/SGLD via sampling is not necessarily optimal. It does not handle stochastic gradients well (the more relevant setting for optimization), only works well when $F$ is approximately smooth, and converting sampling results back to optimization guarantees often incurs extra runtime. Moreover, it is not clear whether sampling, i.e. proving mixing, is necessary to study optimization. 

\subsection{Our Contributions}\label{subsec:ourcontributions}
In our work, we offer a different perspective: we aim to prove optimization results for GLD/SGLD through \textit{Lyapunov potentials} that are implied by Poincar\'e Inequalities. To our knowledge, this is the first time such a proof has been used to analyze global convergence of GLD/SGLD. Techniques to analyze sampling of GLD/SGLD generally go through a Girsanov change of measure style argument \citep{raginsky2017non, balasubramanian2022towards, chewi21analysis}. This is both fragile, and does not work as well for the more natural case of stochastic gradients (SGLD).
In contrast, our Lyapunov-potential based method is more direct, robust, and naturally handles stochastic gradients. Rather than in sampling or even expected suboptimality, our geometric properties allow us to study the \textit{hitting time} of GLD/SGLD. This leads to better bounds for optimizing non-convex functions, both in general and especially via SGLD.

Below we summarize our main contributions. The full statements are given in \pref{subsec:results}: 
\begin{enumerate}
    \item \pref{thm:poincareoptlipschitz} and \pref{thm:poincareoptlipschitzestimate}: Consider the case where $F$ is $s$-H\"{o}lder continuous for some $0 \le s \le 1$, there exists $\gamma \ge 2s$ such that for some $m,b>0$ we have $\tri*{\vecW, \grad F(\vecW)} \ge m\nrm*{\vecW}^{\gamma}-b$, and $\mu_{\beta}$ satisfies a Poincar\'e Inequality with constant $\CPI(\mu_{\beta})$ for $\beta = \widetilde{\Omega}\prn*{\frac{d}{\epsilon}}$. This is the setting of \citet{balasubramanian2022towards} and \citet{chewi21analysis}\footnote{Although these works do not make our assumption on the tail growth of $F$, this assumption is mild and natural for non-convex optimization problems motivated by machine learning.}. For both GLD and SGLD, with probability at least $1-\delta$ we will reach a $\vecW$ with $\epsilon$-suboptimality to the global minimum using at most
    \[ \widetilde{O}\prn*{\max\crl*{d^3 \max(\CPI(\mu_{\beta}), 1)^3, \frac{d^{2+\frac{s}2} \max(\CPI(\mu_{\beta}), 1)^{2+\frac{s}2}}{\epsilon^{2+\frac{s}2}}} \log\prn*{\nicefrac{1}{\delta}}} \]
gradient/stochastic gradient evaluations. Here, the $\tilde{O}$ hides universal constants and polynomial $\log$ factors in $\beta, d, \epsilon$.
    \item \pref{thm:poincareoptlipschitz} and \pref{thm:poincareoptlipschitzestimate}, special case: Consider the case where $F$ is Lipschitz and $\mu_{\beta}$ satisfies a Poincar\'e Inequality with constant $\CPI(\mu_{\beta})$ for $\beta = \widetilde{\Omega}\prn*{\frac{d}{\epsilon}}$. Here, unlike the above, we do not need lower bounds on the tails of $F$. For both GLD and SGLD, with probability at least $1-\delta$ we will reach a $\vecW$ with $\epsilon$-suboptimality to the global minimum using at most
    \[ \widetilde{O}\prn*{\max\crl*{d^3 \max(\CPI(\mu_{\beta}), 1)^3, \frac{d^2 \max(\CPI(\mu_{\beta}), 1)^2}{\epsilon^2}} \log\prn*{\nicefrac{1}{\delta}}} \]
gradient/stochastic gradient evaluations. 
    \item \pref{thm:smoothdissipativesettingpoincareopt}: Consider the case when $F$ is smooth ($\grad F$ is Lipschitz) and $(m,b)$-dissipative (that is, there exist $m, b > 0$ such that $\tri*{\vecW, \nabla F(\vecW)} \ge m\nrm*{\vecW}^2-b$; see \citet{raginsky2017non}, \citet{xu2018global}, \citet{zou2021faster}, and \citet{mou2022improved} for more details on dissipativeness). By $F$ smooth and dissipative, one can show that $\mu_{\beta}$ satisfies a Poincar\'e Inequality for $\beta = \widetilde{\Omega}\prn*{\frac{d}{\epsilon}}$; see Proposition 9 of \citet{raginsky2017non}. For both GLD and SGLD, with probability at least $1-\delta$ we will reach a $\vecW$ with $\epsilon$-suboptimality using
\[ \widetilde{O}\prn*{\max\crl*{d^3 \max(\CPI(\mu_{\beta}), 1)^3, \frac{d^2 \max(\CPI(\mu_{\beta}), 1)^2}{\epsilon^2}} \log\prn*{\nicefrac{1}{\delta}}} \]
gradient/stochastic gradient evaluations. 
    \item \pref{thm:maingeometriccondition}: We show a tight connection between $\mu_{\beta}$ satisfying a Poincar\'e Inequality and the hitting time of the Langevin Diffusion to the set of $\epsilon$-suboptimal global minima of $F$. This is a corollary of literature in probability theory and partial differential equations (PDEs) \citep{cattiaux2013poincare, cattiaux2017hitting}; we believe we are the first to connect these results to optimization.
    \item \pref{thm:gradexactFeps}, \ref{thm:gradestimateFeps}, and \ref{thm:discretizationloosenoracle}: A stronger condition is when the Langevin Diffusion works for optimization in the expected sense: when $\mathbb{E}\brk*{F(\vecW(t))}- \min_{\vecW}F(\vecW)$ is upper bounded by a rate that depends on $t$ and initialization. This is a stronger assumption than a Poincar\'e Inequality, which is tied to the hitting time of the Langevin Diffusion. Under this condition, we prove an optimization rate on the average suboptimality of the iterates for GLD/SGLD of $O \prn*{\nicefrac1{\epsilon^2}}$, which is dimension independent. This shows that when the continuous-time Langevin Diffusion works for optimization, discrete-time GLD/SGLD works as well. 
\end{enumerate}

Note that there are several caveats for using sampling as a way to show global optimization results. As mentioned earlier, $\beta$ must be sufficiently large relative to the tolerance $\epsilon>0$ we want to optimize $F$ to: we need $\beta = \widetilde{\Omega}(\frac{d}{\epsilon})$. 
Specifically, consider when $F(\vecW)=\nrm*{\vecW}^2$, thus $\mu_{\beta}$ is a Gaussian with covariance $\frac1{\beta}\bbI_d$. This is by no means an adversarial example: $F$ is strongly convex, smooth, and well-conditioned. By standard results on Gaussian concentration about mean \citep{vershynin2018high}, we see that we need $\beta = \widetilde{\Omega}(\frac{d}{\epsilon})$ in order for even exact oracle access to $\mu_{\beta}$ to succeed as an efficient optimization strategy. The expected number of queries is $\frac1{\mu_{\beta}\prn*{\{\vecW:F(\vecW)<\epsilon\}}}$ with exact oracle access; if $\epsilon = {o}(\frac{d}{\beta})$, then $\mu_{\beta}\prn*{\{\vecW:F(\vecW)<\epsilon\}}$ is exponentially small in $d$. Note this is reflected in our work, for instance due to the $\mu_{\beta}\prn*{\{\vecW:F(\vecW)<\epsilon\}}$ term in \pref{thm:maingeometriccondition} (see \pref{lem:measureoflargeF}). This is also reflected in \citet{raginsky2017non}, \citet{xu2018global}, and \citet{zou2021faster}, which use the sampling result to upper bound $\mathbb{E}_{\vecW \sim \mu_T}\brk*{F(\vecW)} - \mathbb{E}_{\vecW \sim \mu_{\beta}}\brk*{F(\vecW)}$, where $\mu_T$ denotes the distribution upon running \GLDTEXT/SGLD after $T$ iterations. The expected suboptimality of $F$ under $\mu_{\beta}$, $\mathbb{E}_{\vecW \sim \mu_{\beta}}\brk*{F(\vecW)}$, behaves like $\tilde{\Theta}\prn*{\frac{d}{\beta}}$, and hence $\beta = \Omega\prn*{\frac{d}{\epsilon}}$ is required for optimization. The upper bound here can be proven quite generally, and again consider the Gaussian example for the lower bound. 

Additionally, sampling and optimization runtime guarantees are \textit{not} the same. As mentioned above, as done in \citet{raginsky2017non}, \citet{ xu2018global}, and \citet{zou2021faster}, one uses the sampling result to upper bound $\mathbb{E}_{\vecW \sim \mu_T}\brk*{F(\vecW)} - \mathbb{E}_{\vecW \sim \mu_{\beta}}\brk*{F(\vecW)}$. However, techniques to do this can and often do pick up extra dependence in $d$, $\epsilon$, and isoperimetric constants such as $\CPI(\mu_{\beta})$, depending on the information metric the sampling guarantee is for. Moreover, for papers such as \citet{chewi21analysis} and \citet{balasubramanian2022towards} which study sampling in the constant temperature regime, when converting their results to optimization, we must scale their smoothness parameter by $\beta$, which again changes the runtime. Therefore, the runtime for optimization for other papers may not reflect the runtime written in said paper for sampling, as we compute the rate implied by the literature for our task of optimization (which requires low temperature, that is, large $\beta = \Omega\prn*{\frac{d}{\epsilon}}$): refer to \pref{subsec:poincareoptcomparisontoliteratureproofs} for full derivation of the rates of literature.

We summarize the comparison to literature in \pref{tab:resultscompare}. Note in our comparisons, we assume other results in literature are done with an $O(1)$ warm-start, which is the most favorable for pre-existing literature (i.e. the least favorable comparisons for our results).\footnote{For simplicity, in our comparisons we assume $\CPI(\mu_{\beta}) = \tilde{\Omega}(1)$, which is generally the case (for example this is true if $\mu_{\beta}$ is isotropic and $F$ is convex, and perturbations to $F$ will increase $\CPI(\mu_{\beta})$). All explicit expressions for our rates and those of the literature are given, so one can still perform these comparisons when $\CPI(\mu_{\beta}) = o(1)$.}
\begin{remark}\label{rem:cleversamplingtoopt}
We additionally note that unlike the strategy for converting sampling to optimization guarantees outlined in \citet{raginsky2017non} and followed in \citet{xu2018global} and \citet{zou2021faster}, which is to upper bound $\mathbb{E}_{\vecW \sim \mu_T}\brk*{F(\vecW)} - \mathbb{E}_{\vecW \sim \mu_{\beta}}\brk*{F(\vecW)}$ using sampling guarantees, there is a more elegant and faster approach. To our knowledge it has not been mentioned in literature. The approach is to simply sample until $TV\prn*{\mu_T, \mu_{\beta}} \le 0.1 = O(1)$. For any $\epsilon>0$, denote the set $\{\vecW:F(\vecW)\le\epsilon\}$ by $\cA_{\epsilon}$. For $\beta=\Omega\prn*{\frac{d}{\epsilon}}$, one can show (see \pref{lem:measureoflargeF}) that $\mu_{\beta}\prn*{\cA_{\epsilon}} \ge 0.5$. Therefore $\mu_T\prn*{\cA_{\epsilon}} \ge 0.4$ by definition of TV distance -- that is, the probability our iterate $\vecW_T \in \cA_{\epsilon}$ is at least 0.4. When $\beta=o\prn*{\frac{d}{\epsilon}}$, $\mu_{\beta}\prn*{\cA_{\epsilon}}$ can be exponentially small in $d$ as seen from the Gaussian example, so this strategy still requires large $\beta$. \pref{tab:resultscompare} shows the results using the strategy from \citet{raginsky2017non} known in the literature, but below we discuss the comparisons using both methods. While the rates of literature do improve, our rates are still more favorable.
\end{remark}
Here we expand on these comparisons:
\begin{enumerate}
    \item Consider the case where $F$ is $s$-H\"{o}lder continuous, there exists $\gamma \ge 2s$ such that $\tri*{\vecW, \grad F(\vecW)} \ge m\nrm*{\vecW}^{\gamma}-b$, and $\mu_{\beta}$ satisfies a Poincar\'e Inequality for $\beta = \widetilde{\Omega}\prn*{\frac{d}{\epsilon}}$. This case has been studied in \citet{chewi21analysis} and \citet{balasubramanian2022towards}.
    
    In the GLD case, using the strategy of \citet{raginsky2017non}, Theorem 7 of \citet{chewi21analysis} obtains a rate of $\tilde{O}\prn*{\frac{d^{2+\frac3{s}}\CPI(\mu_{\beta})^{1+\frac1{s}} }{\epsilon^{\frac{4}{s}}}}$. Following the method suggested by \pref{rem:cleversamplingtoopt}, the rate becomes $\tilde{O}\prn*{\frac{d^{1+\frac2{s}}\CPI(\mu_{\beta})^{1+\frac1{s}} }{\epsilon^{\frac{2}{s}}}}$. When $s \le \frac12$, our result from \pref{thm:poincareoptlipschitz} is always better or equal to both of these in all parameters. When $s \in (\frac12, 1]$, our result from \pref{thm:poincareoptlipschitz} is superior to the rate obtained following the strategy of \citet{raginsky2017non} when $\epsilon < \frac{d^{\frac14(3-s)}}{\CPI(\mu_{\beta})^{\frac12(s-\frac12)}}$. \pref{thm:poincareoptlipschitz} is superior to the rate obtained following \pref{rem:cleversamplingtoopt} when $\epsilon < \frac{d^{1-s}}{\CPI(\mu_{\beta})^{s-\frac12}}$.
    
    When $s \le \frac12$, Corollary 19 of \citet{balasubramanian2022towards} improves on \citet{chewi21analysis}. Using the strategy of \citep{raginsky2017non}, the rate is $\tilde{O}\prn*{\frac{d^{\frac{6}{1+s}+8-3s} \CPI(\mu_{\beta})^3}{\epsilon^{\frac{16-2s}{1+s}}}}$, which is using $s \le \frac12$ at least $\tilde{O}\prn*{\frac{d^{10.5} \CPI(\mu_{\beta})^3}{\epsilon^{10}}}$. Following \pref{rem:cleversamplingtoopt}, the rate becomes $\tilde{O}\prn*{\frac{d^{3+\frac{6}{1+s}-2s} \CPI(\mu_{\beta})^3 }{\epsilon^{\frac{6}{1+s}}}}$, which using $s \le \frac12$ is at least $\tilde{O}\prn*{\frac{d^{8} \CPI(\mu_{\beta})^3 }{\epsilon^{6}}}$. Our result from \pref{thm:poincareoptlipschitz} is superior or equal to both of these in all parameters, oftentimes by a significant amount.
    
    In the SGLD case, our rate from \pref{thm:poincareoptlipschitzestimate} is the \textit{first finite gradient complexity guarantee.} 
    \item Consider the case when $F$ is Lipschitz and $\mu_{\beta}$ satisfies a Poincar\'e Inequality for $\beta = \widetilde{\Omega}\prn*{\frac{d}{\epsilon}}$. This has not been well-studied in the sampling or optimization literature, and the only work we know of with finite gradient complexity is \citet{balasubramanian2022towards}, namely $s=0$ in Corollary 19, in the GLD case. The rate here using the strategy of \citet{raginsky2017non} is $\widetilde{O}\prn*{\frac{d^{14} \CPI(\mu_{\beta})^3}{\epsilon^{16}}}$, or following \pref{rem:cleversamplingtoopt}, is $\widetilde{O}\prn*{\frac{d^{9} \CPI(\mu_{\beta})^3}{\epsilon^{6}}}$. Our rate from \pref{thm:poincareoptlipschitz} is superior or equal to both of these in every parameter, oftentimes by a significant amount. Again, \pref{thm:poincareoptlipschitzestimate} is the first finite gradient complexity guarantee for the SGLD case. 
    \item Consider the case for SGLD and when $F$ is smooth and dissipative, which has been well-studied in the works \citet{raginsky2017non}, \citet{xu2018global}, and \citet{zou2021faster}. Theorem 1 of \citet{raginsky2017non} requires gradient noise $\delta$ to be potentially exponentially small in $d$, which does not make sense (we only require gradient noise of constant order, which is more realistic). 

    For using the results from \citet{xu2018global} and \citet{zou2021faster}, we must account for total gradient complexity for a stochastic gradient oracle with $O(1)$ noise. After doing so we obtain $\widetilde{O}\prn*{\frac{d^7}{\epsilon^5 \lambda_{*}^5}}$ for \citet{xu2018global} and a rate of $\widetilde{O}\prn*{\frac{d^5}{\lambda_{*}^4 \epsilon^4}}$ for variance-reduced SGLD from \citet{xu2018global}. Here, $\lambda_{*}$ is a quantity similar to $\frac1{\CPI(\mu_{\beta})}$ (but not directly comparable)\footnote{It is the spectral gap of discrete-time SGLD.}. Our rate from \pref{thm:smoothdissipativesettingpoincareopt} thus is generally superior to both of these in every parameter. (The results of \citet{xu2018global}, being phrased directly in optimization, can't be directly improved using \pref{rem:cleversamplingtoopt}.) The rate from \citet{zou2021faster} is, using Cheeger's Inequality, at least $\widetilde{O}\prn*{\frac{d^8 \CPI(\mu_{\beta})^2}{\epsilon^4}}$ using the strategy of \citet{raginsky2017non}. Thus our rate is superior when $\epsilon < \frac{d^{1.25}}{\CPI(\mu_{\beta})^{0.25}}$. Following \pref{rem:cleversamplingtoopt}, \citet{zou2021faster} yields a rate of $\widetilde{O}\prn*{\frac{d^{6} \CPI(\mu_{\beta})^2 }{\epsilon^2}}$; our rate is superior when $\epsilon < \frac{d^{1.5}}{\CPI(\mu_{\beta})^{0.5}}$.
    \item We additionally touch on other discretizations of the Langevin Diffusion. To our knowledge, the only other discretization of \pref{eq:LangevinSDE} successful beyond log-concavity is the Proximal Sampler first introduced in \citet{lee2021structured}.  With exact gradients, \citet{altschuler2023faster} showed it succeeds under a Poincar\'e Inequality when $F$ is smooth; the Proximal Sampler can only be implementable with smoothness for non-convex $F$. In the stochastic gradient setting, the only work we are aware of showing its success is Theorems 4.1 and 4.2 of \citet{huang2024faster}, showing the Proximal Sampler succeeds under smoothness and a Log-Sobolev Inequality (which is satisfied in the smooth and dissipative setting as shown in Proposition 9 of \citet{raginsky2017non}). The rate from there is, using the strategy of \citet{raginsky2017non}, $\tilde{O}\prn*{\frac{d^{5.5} \CPI(\mu_{\beta})^3}{\epsilon^5}}$. Or following \pref{rem:cleversamplingtoopt}, the rate is $\tilde{O}\prn*{\frac{d^{3.5} \CPI(\mu_{\beta})^3}{\epsilon^3}}$. Our rate from \pref{thm:smoothdissipativesettingpoincareopt} is superior or equal in every parameter, often by a significant amount.
\end{enumerate}
\arxiv{ 
\begin{table}[h!] 
\centering 
\renewcommand{\arraystretch}{1.15}  
\begin{tabular}{|>{\centering\arraybackslash}m{1.5in}|>{\centering\arraybackslash}m{2.1in}|>{\centering\arraybackslash}m{2.1in}|}
\toprule 
\textbf{Problem Setting} & \textbf{Our Result} & \textbf{Best in Literature} \\
\hline  
GLD Poincar\'e \& Lipschitz & $\widetilde{O}\prn*{\max\crl*{d^3 \CPI(\mu_{\beta})^3, \frac{d^2 \CPI(\mu_{\beta})^2}{\epsilon^2}}}$ & \begin{minipage}{1in}  
        \centering 
     \vspace{5pt}
$\widetilde{O}\prn*{\frac{d^{14} \CPI(\mu_{\beta})^3}{\epsilon^{16}}}$ \\ \vspace{2pt} \citep{balasubramanian2022towards}   
\vspace{5pt}
\end{minipage} \\
\hline 
        
SGLD Poincar\'e \& Lipschitz& $\widetilde{O}\prn*{\max\crl*{d^3 \CPI(\mu_{\beta})^3, \frac{d^2 \CPI(\mu_{\beta})^2}{\epsilon^2}}}$&  \begin{minipage}{1.5in}  
        \centering 
     \vspace{8pt}
No finite guarantee   
\vspace{8pt} 
\end{minipage}
 \\
 \hline
 SGLD smooth \& dissipative & {\centering $\widetilde{O}\prn*{\max\crl*{d^3 \CPI(\mu_{\beta})^3, \frac{d^2 \CPI(\mu_{\beta})^2}{\epsilon^2}}}$} &
 \begin{minipage}{1.6in}  
        \centering 
     \vspace{10pt}
$\widetilde{O}\prn*{\min\crl*{\frac{ d^{8}\CPI(\mu_{\beta})^2}{\epsilon^4},\frac{d^7}{\epsilon^5 \lambda_{*}^5}}}$ \\     \vspace{2pt} \citep{xu2018global, zou2021faster} 
\vspace{10pt}
\end{minipage} \\
\bottomrule 
\end{tabular}
\vspace{0.5mm} 
\caption{Gradient complexity comparisons. In the table, $d$ refers to dimension and $\epsilon$ refers to tolerance. $\beta = \widetilde{\Theta}\prn*{\frac{d}{\epsilon}}$, and $\CPI(\mu_{\beta})$ denotes the Poincar\'e constant of $\mu_{\beta}$. $\lambda_{*}$ is a spectral gap comparable to $\CPI(\mu_{\beta})$.
} 
\label{tab:resultscompare} 
\end{table}
}

\paragraph{Notation.} Unless otherwise specified the domain is $\mathbb{R}^d$, with origin $\vecOrigin$. We denote the Laplacian (sum of second derivatives) of a twice-differentiable function $f$ by $\Delta f$. Here $\ball(p,R)$ denotes the Euclidean $l_2$ ball centered at $p\in\mathbb{R}^d$ with radius $R \ge 0$. $\mathcal{S}^{d-1}$ denotes the surface of the $d$-dimensional unit sphere. $\widetilde{\Omega}$, $\widetilde{\Theta}$, $\widetilde{O}$ hide universal constants, $\log$ factors in $\beta, d, \epsilon$, as well as $\vecW_0$-dependence. Sometimes we will write exponentials as $\text{exp}$ for readability. When we write vectors $\vecW_t$ this denotes time $t$ in discrete time, and when we write $\vecW(t)$ this denotes time $t$ in continuous time. 
Unless indicated otherwise, $\mathbb{E}$ refers to expectation over the Brownian motion/random variables $\vecEps_t$ (as well as the data samples $\vecZ_t$ in the SGLD case), and $\mathbb{E}_{\vecW}$ denotes the same expectation when the stochastic processes is initialized at $\vecW$. For any set $\cU \subset\mathbb{R}^d$, let the hitting time of the Langevin Diffusion \pref{eq:LangevinSDE} initialized at $\vecW$ to $\cU$ be $\tau_{\cU}(\vecW)$. 
We assume that first order tensors, i.e. vectors, are equipped with $l_2$ Euclidean norm and that all second order tensors (i.e. matrices) and above are equipped with operator norm. When we write $\nrm*{\cdot}$ without specifying the norm, we implicitly mean the $l_2$ Euclidean norm of a vector. For some $f$ differentiable to $k$ orders, we will let $\nabla^{k} f$ denote the tensor of all the $k$-th order derivatives of $f$, and $\nrm*{\cdot}_{\OPNORM}$ denotes the corresponding tensor's operator norm. 

\section{Lyapunov Potentials and Optimization}\label{sec:contributions}
In the rest of this paper, suppose $F$ has a global minimum $\vecW^{\star}$, which need not be unique (thus $\vecW^{\star}$ can refer to any of these). Furthermore, without loss of generality, assume that $F(\vecW^{\star})=0$. 
\subsection{Our Strategy}\label{subsec:strategy}
Optimization under Langevin Dynamics can ultimately be posed as a question of \textit{hitting time}: how long does it take to reach a point $\vecW$ such that  $F(\vecW)\le \epsilon$? In the probability theory and stochastic partial differential equations (PDEs) literature, an extensive program has been devoted to studying the connection between isoperimetric inequalities such as a Poincar\'e Inequality, hitting times of the Langevin Diffusion to sets $A \subset \mathbb{R}^d$, and \textit{Lyapunov potentials}.
A subset of this literature includes \citet{carmona1983exponential, meyn1993stability, down1995exponential, bakry2008simple, cattiaux2009lyapunov, cattiaux2010functional, meyn2012markov, cattiaux2013poincare, cattiaux2017hitting}. 
As mentioned in \pref{sec:introduction}, Poincar\'e Inequalities are the loosest conditions under which sampling and in turn global optimization guarantees for Langevin Dynamics have been well-studied. This literature connects these inequalities to the \textit{geometry of $F$}.

\begin{definition}
Say a non-negative function $\Phi:\mathbb{R}^d\rightarrow\mathbb{R}$ is a \textit{Lyapunov potential} (for Langevin Dynamics at inverse temperature $\beta$ given in \pref{eq:LangevinSDE}) if $\Phi \ge 1$ and on the set $\{\vecW:F(\vecW) > \epsilon\}$ we have
\[ \tri*{\grad \Phi(\vecW), \grad F(\vecW)} \ge \lambda \Phi(\vecW)+\frac1{\beta}\Delta\Phi(\vecW),  \numberthis\label{eq:lyapunovbaby}\]
where  $\beta$ refers to the inverse temperature of \pref{eq:LangevinSDE}.
\end{definition}

Our main method to study optimization is to track the progress of GLD/SGLD using the Lyapunov potential $\Phi(\vecW)$. Suppose such a Lyapunov potential existed: from here, we can study the hitting time of GLD/SGLD to the set $\cA_{\epsilon}=\{\vecW: F(\vecW)\le\epsilon\}$.  

The fundamental idea is as follows. Consider $\tau_{\cA_{\epsilon}}(\vecW_0)$, the hitting time of GLD/SGLD initialized at $\vecW_0$ to $\cA_{\epsilon}$. Denote this by $\tau$ for short in the following. Consider the random variable $X :=\frac1{\tau} \sum_{t=0}^{\tau-1} \lambda\Phi(\vecW_t)$. 
Suppose that $\Phi$ is $L$-smooth and $L$-Hessian Lipschitz. The idea is that, by the following, we can make $X$ relatively small if $\tau$ is relatively large, by Taylor expanding $\Phi$ to third order and using \pref{eq:poincaregeomcondition} (it turns out to be possible to control the higher order discretization terms). 

However, by definition none of $\vecW_0, \ldots, \vecW_{\tau-1}$ lie in $\cA_{\epsilon}$. Clearly $X$ is lower-bounded by $\lambda$, since $\Phi \ge 1$. But we just showed $X$ is small if $\tau$ is relatively large. This gives contradiction! Hence, we can upper bound $\tau$. This idea, while currently informal, can be made rigorous (using discrete-time Dynkin's formula, Theorem 11.3.1, page 277 of \citet{meyn2012markov}). See  \pref{sec:PIoptproofs} for details. 

To show how $X$ can be made small, using definition \pref{eq:SGLDiterates}, we Taylor expand $\Phi$ to third order (using that it is $L$-smooth and $L$-Hessian Lipschitz) to obtain
\begin{align*}
\Phi(\vecW_{t+1}) &= \Phi\prn*{\vecW_t - \eta \grad F(\vecW_t) + \sqrt{2\eta\beta^{-1}} \vecEps_t} \\
&\le \Phi(\vecW_t) + \tri*{-\eta \grad F(\vecW_t), \grad \Phi(\vecW_t)} + \tri*{\sqrt{2\eta\beta^{-1}} \vecEps_t, \grad \Phi(\vecW_t)}\\
&\hspace{1in}+ \frac12 \tri*{\grad^2 \Phi(\vecW_t)\prn*{- \eta \grad F(\vecW_t) + \sqrt{2\eta\beta^{-1}} \vecEps_t}, - \eta \grad F(\vecW_t) + \sqrt{2\eta\beta^{-1}} \vecEps_t} \\
&\hspace{1in}+ \frac{L}6 \nrm*{- \eta \grad F(\vecW_t) + \sqrt{2\eta\beta^{-1}} \vecEps_t}^3.
\end{align*}
We first use \pref{eq:lyapunovbaby}, which gives 
\[ \tri*{-\eta \grad F(\vecW_t), \grad \Phi(\vecW_t)} \le -\eta \lambda \Phi(\vecW_t) -\frac{\eta}{\beta}\Delta\Phi(\vecW_t). \]
Now, take expectations with respect to $\vecEps_t$. The term $\tri*{\sqrt{2\eta\beta^{-1}} \vecEps_t, \grad \Phi(\vecW_t)}$ disappears, in addition to the cross term $-2\eta \sqrt{2\eta\beta^{-1}} \tri*{\grad^2 \Phi(\vecW_t) \vecEps_t, \grad F(\vecW_t)}$ from the second-order term. Note now that 
\[ \mathbb{E}\brk*{\frac12\tri*{\grad^2 \Phi(\vecW_t) \cdot \sqrt{2\eta\beta^{-1}} \vecEps_t,\sqrt{2\eta\beta^{-1}} \vecEps_t}} = \frac{\eta}{\beta}\Delta\Phi(\vecW_t).\]
Therefore, the Laplacian terms $\frac{\eta}{\beta}\Delta\Phi(\vecW_t)$ cancel in the above after taking expectations, and what we obtain is (upon dividing by $\eta$)
\[ \lambda\mathbb{E}\brk*{\Phi(\vecW_t)} \le \mathbb{E}\brk*{\Phi(\vecW_t)}-\mathbb{E}\brk*{\Phi(\vecW_{t+1})} + \crl{\text{higher order discretization error terms}}.\] 
Summing and telescoping this relation, and using that $\Phi$ is non-negative, we obtain
\[ \mathbb{E}\brk*{X} = \frac1{\tau} \sum_{t=0}^{\tau-1} \lambda\mathbb{E}\brk*{\Phi(\vecW_t)} \le \frac{\Phi(\vecW_0)}{\tau} + \frac1{\tau} \cdot \crl{\text{higher order discretization error terms}}.\] 
If we can control higher order discretization error terms, which it turns out we can do as discussed in \pref{sec:PIoptproofs}, then if $\tau$ is large then $\mathbb{E}\brk*{X}$ will be small. But as discussed earlier $X \ge \lambda$ pointwise, hence $\mathbb{E}\brk*{X} \ge \lambda$. This lets us control $\tau$, the hitting time of GLD/SGLD to the set $\cA_{\epsilon}$. 

One might note this idea of considering the hitting time of SGLD to $\cA_{\epsilon}$ bears resemblance to the style of proof from \citet{zhang2017hitting}. However, \citet{zhang2017hitting} considered the hitting time to second-order stationary points, and so our results (in addition to the techniques) are fairly different. 

To fully generalize this, using \pref{lem:thirdordersmoothfromregularity}, this idea can be extended to cover essentially all Lyapunov functions of interest (far beyond when $\Phi$ is smooth and Hessian Lipschitz). Due to the stochasticity already present in GLD, our analysis for GLD vs SGLD is extremely similar.

The geometric condition \pref{eq:lyapunovbaby} turns out to be closely linked to a Poincar\'e Inequality: as a corollary of \citet{cattiaux2017hitting} we obtain the following: 
\begin{theorem}\label{thm:maingeometriccondition}
Assume that $\mu_{\beta}$ satisfies a Poincar\'e inequality with constant $\CPI(\mu_{\beta})$ and has finite second second moment for some $\beta = \widetilde{\Omega}\prn*{\frac{d}{\epsilon}}$. Then on $\cA_{\epsilon}^c=\{\vecW:F(\vecW) > \epsilon\}$,
\[ \tri*{\grad F(\vecW), \grad \Phi(\vecW)} \ge\lambda\Phi(\vecW) + \frac1{\beta} \Delta\Phi(\vecW)  \quad \text{ for } \quad \lambda \in \brk*{\frac1{8\beta} \min\prn*{\frac1{{\CPI}(\mu_{\beta})}, \frac12}, \frac1{4\beta} \min\prn*{\frac1{{\CPI}(\mu_{\beta})}, \frac12}},\numberthis\label{eq:poincaregeomconditioninformal}\] 
for some non-negative $\Phi$ that is differentiable to all orders such that on $\cA_{\epsilon}^c$, $\Phi$ takes the explicit form
\[ \Phi(\vecW') = \mathbb{E}_{\vecW'}\brk*{\exp\prn*{\lambda \tau_{\cA_{\epsilon}}}}.\] 
\end{theorem}

\begin{remark}\label{rem:philowerboundremark}
Note that on $\cA_{\epsilon}^c$, $\Phi \ge 1$. Also note $\Phi$ generally behaves in a `dimension free' manner, depending on how $\tau_{\cA_{\epsilon}}(\vecW')$ behaves, as $\lambda \le \frac1{4\beta} \min\prn*{\frac1{{\CPI}(\mu_{\beta})}, \frac12}$ is very small. 
\end{remark}
We note that in \citet{bakry2008simple} and \citet{cattiaux2013poincare}, the condition \pref{eq:poincaregeomconditioninformal} is shown to imply that $\mu_{\beta}$ satisfies a Poincar\'e Inequality if $\cA_{\epsilon}$ is connected. The proof of this direction requires connectedness. However, note \pref{eq:poincaregeomconditioninformal} implies the moment generating function $\mathbb{E}_{\vecW'}\brk*{\exp\prn*{\lambda \tau_{\cA_{\epsilon}}}} < \infty$. Make the very mild assumption that $\cA_{\epsilon}$ lies in $\ball(\vecOrigin, R)$ for $R<\infty$ large enough. Since pointwise $\tau_{\cA_{\epsilon}} \ge \tau_{\ball(\vecOrigin, R)}$, this implies $\mathbb{E}_{\vecW'}\brk*{\exp\prn*{\lambda \tau_{\ball(\vecOrigin, R)}}} < \infty$. From here, \citet{cattiaux2013poincare} shows a geometric property analogous to \pref{eq:poincaregeomconditioninformal} holds where the Lyapunov function is now $\mathbb{E}_{\vecW'}\brk*{\exp\prn*{\lambda \tau_{\ball(\vecOrigin, R)}}}$, and in turn that $\mu_{\beta}$ satisfies a Poincar\'e Inequality. The moment generating function satisfying \pref{eq:lyapunovbaby} (that is, the moment generating function being a valid Lyapunov potential in the sense here) and isoperimetric inequalities are thus linked very tightly, as equivalent for the Langevin diffusion.

\subsection{Results}\label{subsec:results}
Now, we state our results in full detail. Complete statements and proofs, including all explicit dependencies, are in \pref{sec:PIoptproofs}. For all of our results, recall from \pref{subsec:ourcontributions} that the desired tolerance $\epsilon= \widetilde{\Omega}\prn*{\frac{d}{\beta}}$; no results so far in literature yield meaningful optimization guarantees for smaller tolerance levels. 

Before we state our results more explicitly, we state our assumptions, which are in fact necessary. Our first assumption, generalized to higher order derivatives from \citet{priorpaper}, is that the Lyapunov potential $\Phi$ satisfies `self-bounding regularity' in the following sense:
\begin{definition}\label{def:selfboundingregularity}
A $k$ times differentiable function $f:\mathbb{R}^d\rightarrow\mathbb{R}$ satisfies \textit{$k$-th order self-bounding regularity} if 
\[ \nrm*{\nabla^{k} f(\vecW)}_{\OPNORM}\le \rho_{f,k}(\abs{f(\vecW)})\] 
for some increasing function $\rho_{f,k}:\mathbb{R}\rightarrow\mathbb{R}_{\ge 0}$.

We say $f$ satisfies \textit{polynomial-like self-bounding regularity at order $k$} if we can express $\rho_{f,k}(z)=\sum_{j=0}^{n} c_j z^{d_j}$ where all $d_j \ge 0$. Note without loss of generality we can assume all $c_j, d_j \ge 0$ and $\rho_{f,k}(z)=A(z+1)^p$ or $\rho_{f,k}(z)=A+Az^p$ by the AM-GM Inequality.
\end{definition}
\textit{Such an assumption is necessary} for discrete-time optimization to succeed: Theorem 3 from \citet{priorpaper} shows even for Gradient Flow/Gradient Descent, there are examples where discrete-time optimization fails when continuous-time optimization succeeds. This is exactly what allow for control of higher order discretization terms arising in discrete-time optimization. As such we will assume the following: 
\begin{assumption}\label{ass:selfboundingPhipoly}
Suppose $\Phi$ satisfies first, second, and third order polynomial-like self-bounding regularity where the monomials in the self-bounding regularity functions have degree at most 1.
\end{assumption}
Note \pref{ass:selfboundingPhipoly} is satisfied by many Lyapunov functions, e.g. when the Lyapunov function $\Phi$ has tail growth polynomial in $\nrm*{\vecW}$ or of the form $e^{r\nrm*{\vecW}^s}$ for $s\le 1$, going well beyond smoothness. Since we have explicit knowledge of $\Phi$ via \pref{thm:maingeometriccondition}, this is just saying the MGF of $\tau_{\{\vecW:F(\vecW) < \epsilon\}}$ is reasonably well-behaved as a function of the initialization $\vecW$ of the continuous-time Langevin Diffusion.

Now we state our assumptions on $F$. We consider the most general setting of previous works \citep{chewi21analysis,balasubramanian2022towards} for analyzing LMC where we assume $F$ is H\"{o}lder continuous with parameter $0 \le s \le 1$: 
\begin{assumption}[H\"{o}lder continuity]\label{ass:holderF}
Suppose $\nabla F$ satisfies $L$-H\"{o}lder continuity for some $0 \le s \le 1$:
\[ \nrm*{\nabla F(\mathbf{u})-\nabla F(\mathbf{v})} \le L\nrm*{\mathbf{u}-\mathbf{v}}^s.\]
\end{assumption}
When $s>0$, that is $F$ is not Lipschitz, we also require an assumption on the growth of $F$. This significantly generalizes the dissipation assumption (when $s=1$ and $\gamma=2$) made in several previous works studying non-convex optimization \citep{raginsky2017non,xu2018global,zou2021faster,mou2022improved}.
\begin{assumption}\label{ass:weakdissipation}
There exists $\gamma \ge 2s$ such that for some $m,b>0$ and all $\vecW\in\mathbb{R}^d$,
\[ \tri*{\vecW, \grad F(\vecW)} \ge m\nrm*{\vecW}^{\gamma}-b.\]
\end{assumption}
Analyzing growth rates, we can see $\gamma \le s+1$, which leads to no issues for $0 \le s \le 1$. Note this assumption is quite reasonable: in some sense it states that the gradient will push us towards the origin when we are sufficiently far away. Moreover, all critical points of $F$ are in $\ball(\vecOrigin,(b/m)^{1/\gamma})$. However, we allow for arbitrary non-convexity inside this ball. In fact, by adding a suitable regularizer penalizing solutions lying outside $\ball(\vecOrigin,(b/m)^{1/\gamma})$, we can ensure $F$ satisfies the above, which is discussed on page 15 of \citet{raginsky2017non}.

\begin{theorem}\label{thm:poincareoptlipschitz}
Suppose that $F$ satisfies \pref{ass:holderF} and \pref{ass:weakdissipation}, $\mu_{\beta}$ satisfies a Poincar\'e Inequality with constant $\CPI(\mu_{\beta})$ for $\beta = \widetilde{\Omega}\prn*{\frac{d}{\epsilon}}$, and $\mu_{\beta}$ has finite second moment $S<\infty$. (In our results dependence on $S$ will be logarithmic.) Suppose $\Phi$ (from \pref{thm:maingeometriccondition}) satisfies \pref{ass:selfboundingPhipoly} with $p\le 1$. Then running \GLDTEXT, with probability at least $1-\delta$, across all the runs we will reach a $\vecW$ with $F(\vecW)\le\epsilon$ in at most 
\[ \widetilde{O}\prn*{\max\crl*{d^3 \max(\CPI(\mu_{\beta}), 1)^3, \frac{d^{2+\frac{s}2} \max(\CPI(\mu_{\beta}), 1)^{2+\frac{s}2}}{\epsilon^{2+\frac{s}2}}} \log\prn*{\frac1{\delta}}}\numberthis\label{eq:runtimeguaranteelipschitz}\]
gradient evaluations.
\end{theorem}
We note considering \pref{ass:holderF} for any $s \ge 0$ and a Poincar\'e Inequality is quite natural. In terms of growth of $F$, a Poincar\'e Inequality implies at least linear tail growth of $F$ but nothing further, as discussed on page 7 of \cite{chewi21analysis}. Thus, \pref{ass:holderF} for any $s \ge 0$ and a Poincar\'e Inequality are not only compatible but natural to study in tandem.

We now move on to the stochastic gradient oracle case. Some control over the stochastic gradient estimates is necessary: if they are very inaccurate, following them will be meaningless.
\begin{assumption}[Bound of variance of gradient estimates]\label{ass:gradnoiseassumption}
The unbiased gradient estimate $\nabla f(\vecW;\vecZ)$ of $\nabla F(\vecW)$ satisfies the sub-Gaussian property that for all $\vecW\in\mathbb{R}^d$ and $t\ge 0$,
\[ \mathbb{P}_{\vecZ}\prn*{\nrm*{\nabla f(\vecW;\vecZ)-\nabla F(\vecW)}_2 \ge t} \le e^{-t^2/\sigma_F^2}. \numberthis\label{eq:gradnoisestandard}\]
\end{assumption}
This covers the classic setting of stochastic optimization where $\nabla f(\vecW;\vecZ) = \nabla F(\vecW)+\vecEps_t$ where $\vecEps_t$ is sub-Gaussian with mean 0 and variance $\sigma_F^2$ \citep{nemirovskistochastic}. We expect our techniques to hold when gradient noise scales in function value, a more general setting discussed in \citet{priorpaper}, but for simplicity we work with \pref{ass:gradnoiseassumption}. 

We also need the following assumption made in \citet{raginsky2017non} studying stochastic optimization in this setting. This is quite reasonable: it essentially says the stochastic gradients contain reasonable signal and also will push us towards the origin when sufficiently far away.
\begin{assumption}\label{ass:stochasticgradcontrol}
For every $\vecZ$, $\grad f(\vecW;\vecZ)$ satisfy \pref{ass:holderF} and \pref{ass:weakdissipation}. (Note they may be satisfied with larger $L$ and $b$ and smaller $m$.)
\end{assumption}
Then, we have the following:
\begin{theorem}\label{thm:poincareoptlipschitzestimate}
Suppose $\mu_{\beta}$, $F$, $\Phi$ satisfy the same assumptions as in \pref{thm:poincareoptlipschitz}. Then running \SGLDTEXTSPACE with a stochastic gradient oracle satisfying \pref{ass:gradnoiseassumption} and \pref{ass:stochasticgradcontrol}, we obtain the same guarantee \pref{eq:runtimeguaranteelipschitz} of the query complexity of our stochastic gradient oracle as in \pref{thm:poincareoptlipschitz}.
\end{theorem}
To our knowledge, our result \pref{thm:poincareoptlipschitzestimate} is the first finite iteration guarantee for the setting of $F$ H\"{o}lder-continuous and $\mu_{\beta}$ satisfying a Poincar\'e Inequality with a stochastic gradient oracle. The stronger assumption of smoothness is not satisfied by many canonical non-convex optimization problems \citep{priorpaper}, so analyzing optimization with a stochastic gradient oracle in this more general setting is highly relevant to study.

Recall from our conditions \pref{ass:holderF} and \pref{ass:weakdissipation} that by analyzing the implied growth rates of $F$, we have $2s \le \gamma \le s+1$. Thus when $s=1$, $\gamma=1$ is forced, so this recovers as a special case of our assumption the smooth and dissipative setting from \citet{raginsky2017non}, \citet{xu2018global}, and \citet{zou2021faster}. In turn, $s=1$, $\gamma=1$ actually implies $\mu_{\beta}$ satisfies a Poincar\'e Inequality for all $\beta \ge \frac2m$ \citep{raginsky2017non}. In this setting we have the following result which is stronger than directly applying \pref{thm:poincareoptlipschitz}: 
\begin{theorem}\label{thm:smoothdissipativesettingpoincareopt}
Suppose $F$ is $L$-smooth and $(m,b)$-dissipative (that is, there exist $m, b > 0$ such that $\tri*{\vecW, \nabla F(\vecW)} \ge m\nrm*{\vecW}^2-b$). Running either \GLDTEXTSPACE or \SGLDTEXTSPACE with a stochastic gradient oracle satisfying \pref{ass:gradnoiseassumption} and \pref{ass:stochasticgradcontrol}, with probability at least $1-\delta$, across all the runs we will reach a $\vecW$ with $F(\vecW)\le\epsilon$ in at most 
\[ \widetilde{O}\prn*{\max\crl*{d^3 \max(\CPI(\mu_{\beta}), 1)^3, \frac{d^{2} \max(\CPI(\mu_{\beta}), 1)^{2}}{\epsilon^{2}}} \log\prn*{\frac1{\delta}}} \]
gradient/stochastic gradient evaluations.
\end{theorem} 


\section{Dimension Free Rates for SGLD Under Stronger Conditions}\label{sec:beyondpoincare}
Our results in the above upper bound the hitting-time of GLD/SGLD to $\{\vecW:F(\vecW)<\epsilon\}$. As discussed in \pref{subsec:strategy}, there is a close connection between a Poincar\'e Inequality, geometric properties, and hitting times of the Langevin Diffusion, so those results make intuitive sense. However, hitting times are a weaker guarantee than average suboptimality which is commonly studied in optimization. It is thus natural to ask if there are stronger conditions where we have guarantees not about not the hitting time but the average suboptimality $\frac1T \sum_{t=1}^T F(\vecW_t)$? This is indeed the case as we now discuss.

This idea is similar to that of \citet{priorpaper}. In \citet{priorpaper}, an analogous setup was considered where continuous-time gradient flow led to geometric properties and in discrete-time, the average suboptimality of gradient descent/stochastic gradient descent was studied under these geometric properties. This is the It\^{o} calculus analogy of the work in \citet{priorpaper}. 

Define a rate function for average suboptimality as follows: fix a desired tolerance $\epsilon=\widetilde{\Omega}\prn*{\frac{d}{\beta}}$. Define 
\[ F_{\epsilon}(\vecW)=F(\vecW) \mathbf{1}_{\{\vecW:F(\vecW) \ge \epsilon\}}.\numberthis\label{eq:Fepsdef}\]
Here $\mathbf{1}_{\cA}$ denotes the indicator function of $\cA \subset\mathbb{R}^d$. Suppose we had some non-negative rate function $R(\vecW,t)$ upper bounding the suboptimality of $F$, namely such that $R(\vecW, t) \ge \mathbb{E}\brk*{F_{\epsilon}(\vecW(t))}$ for all $t \ge 0$. By rate function, we mean that the following conditions, natural for optimization, hold: $\lim_{t\rightarrow\infty}R(\vecW(0),t)=0$ for all $\vecW(0)$, and: 
\[ \mathbb{E}\brk*{R(\vecW(s),t)} \le R(\vecW(0),s+t) \FORALLTEXT s, t \ge 0, \vecW(0)\in\mathbb{R}^d. \numberthis\label{eq:ldrateineq}\]
That is, in expectation more information about the Langevin Dynamics path improves the rate. Note from any rate function $R(\vecW,t)$, there is an equivalent rate function satisfying \pref{eq:ldrateineq}; see \citet{priorpaper} for more discussion. 

It is important to note the following: we cannot have $R(\vecW,t) \ge \mathbb{E}\brk*{F(\vecW(t))}$: say after reaching $\epsilon$-suboptimality, gradients are very small and Langevin dynamics approximates a random walk. Then for all $t$ large enough, $\mathbb{E}\brk*{F(\vecW(t))} \approx c_0\epsilon$ for some $c_0=\Theta\prn*{1}$. We thus have for all $t$ large enough,
\[ R(\vecW, t)\ge \mathbb{E}\brk*{R(\vecW(t), 0)}\ge \mathbb{E}\brk*{F(\vecW(t))} \approx c_0\epsilon,\]
contradicting that $\lim_{t\rightarrow\infty}R(\vecW,t)=0$. However, our definition \pref{eq:Fepsdef} \textit{resolves this problem}: now for large enough $t$, $\mathbb{E}\brk*{F_{\epsilon}\prn*{\vecW(t)}}=0$, so \pref{eq:ldrateineq} now holds. 
Moreover, such a rate function clearly implies Langevin Dynamics works as an optimization strategy. In fact, such a rate function implies similar geometric properties to \pref{eq:lyapunovbaby} from \pref{thm:maingeometriccondition}:
\begin{definition}[Admissible potential]
A non-negative function $\Phi(\vecW)$ is an admissible potential with respect to the cost function $F_{\epsilon}(\vecW)$ if 
\[ \tri*{ \nabla\Phi(\vecW),\nabla F(\vecW)} \ge F_{\epsilon}(\vecW)+\frac1{\beta} \Delta \Phi(\vecW).\numberthis\label{eq:admissablepotentialF}\]
\end{definition}
\begin{theorem}[From rate functions to potentials]  \label{thm:constructadmissablepotential}
Under mild assumptions on $F$ and $R$, suppose $R(\vecW,t)$ satisfies the relationship \pref{eq:ldrateineq} and $\int_0^\infty R(\vecW,t) \DERIV t<\infty$ always holds true. Then $\Phi(\vecW)=\int_0^\infty R(\vecW,t) \DERIV t$ is an admissable potential.
\end{theorem}
\begin{remark}
Note the Langevin Diffusion with $\beta=\infty$ becomes gradient flow (\GFTEXT). In \citet{priorpaper}, the success of an analogous rate function for GF implied
the very similar condition
\[ \tri*{\nabla F(\vecW), \nabla \Phi(\vecW)} \ge F(\vecW).\numberthis\label{eq:oldconditiongf}\]
Note \pref{eq:oldconditiongf} implies \pref{eq:admissablepotentialF} with $\beta=\infty$. The confirms the intuition that Langevin Dynamics optimizes a larger class of functions than \GFTEXTSPACE can.
\end{remark}
It turns out in the idealized continuous-time setting, \pref{eq:admissablepotentialF} is enough to show that $F$ can be optimized: see \pref{subsec:optrategeometryformal} and  \pref{thm:conditionimpliesrateproofformal}.
But as Theorem 3 from \citet{priorpaper} showed, to go to discrete time, we need assumptions on $\Phi$ such as self-bounding regularity assumptions.
\begin{assumption}\label{ass:polyselfbounding}
$\Phi$ satisfies polynomial self-bounding regularity (\pref{def:selfboundingregularity}) for degrees one through three.
\end{assumption}

Moreover, our potential function needs to capture reasonable information (e.g. it cannot remains small while the iterates $\vecW_t$ escapes to infinity). 
We make this precise as follows.
\begin{assumption}\label{ass:iteratesinballassumption}
Suppose we initialize $\vecW_0$ in $\ball(\vecOrigin, R_1')$ for some $R_1'>0$. Letting $\sup_{\vecW\in \ball(\vecOrigin, R_1')} \Phi(\vecW) = B'$, suppose that there is some $\kappa'>1$ and $R_1 > 0$ such that $\{\vecW:\Phi(\vecW)<\kappa' B'\} \subset \ball(\vecOrigin, R_1)$.
\end{assumption}
For $F$, we slightly loosen the assumptions compared to \pref{sec:contributions}, and just assume \pref{ass:holderF} holds for any $s \ge 0$.

From here, we have the following results for GLD and SGLD:
\begin{theorem}\label{thm:gradexactFeps}
Suppose \pref{eq:admissablepotentialF} holds for some $\beta=\tilde{\Omega}\prn*{\frac{d}{\epsilon}}$, where the Lyapunov function $\Phi \ge 0$ and where $F_{\epsilon}$ is defined from \pref{eq:Fepsdef}. Suppose $F$ satisfies \pref{ass:holderF}, and $\Phi$ satisfies \pref{ass:polyselfbounding} and \pref{ass:iteratesinballassumption}. Then running discrete-time \GLDTEXTSPACE with constant step size $\eta$ will, with probability at least $1-\delta$, yield a sequence of iterates $\vecW_t$ with $\frac1T \sum_{t=1}^T F(\vecW_t) = \widetilde{O}\prn*{\epsilon}$ in at most $T=O\prn*{\frac{1}{\epsilon^2}\log\prn*{\nicefrac1{\delta}}}$ iterations.
\end{theorem}
\begin{theorem}\label{thm:gradestimateFeps}
Suppose $F$, $\Phi$ satisfy the same assumptions as in \pref{thm:gradexactFeps}. Then running \SGLDTEXTSPACE with a stochastic gradient oracle satisfying \pref{ass:gradnoiseassumption}, we obtain the same guarantees as \pref{thm:gradexactFeps}.
\end{theorem}
We defer the proofs to \pref{sec:corediscreteproofs}. The idea is similar to the sketch from \pref{subsec:strategy}; we again use self-bounding regularity to control the higher order discretization terms. Note this implies the following result: \textit{using GLD/SGLD, we can not only optimize (via GLD) but also learn (via SGLD) any function for which Langevin Dynamics can optimize with a rate function well-behaved for optimization} (one that is admissible). 

However, we would like to loosen our condition \pref{eq:admissablepotentialF} for it to hold for function classes of interest in optimization. By modifying our proofs, we can show success if we have the looser condition
\[ \tri*{ \nabla \Phi(\vecW), \nabla F(\vecW)} \ge F_{\epsilon}(\vecW) +\min\prn*{0, \frac1{\beta} \Delta\Phi(\vecW)}\numberthis\label{eq:loosenedcondition}.\]
whenever we can query $\tri*{ \nabla \Phi(\vecW), \nabla F(\vecW)} - F_{\epsilon}(\vecW)$. Note this is realistic, for example, if $\Phi = F$. This difference represents local gradient domination; if non-negative, gradient descent locally succeeds and we should not add noise. Otherwise, we add noise, and \pref{eq:loosenedcondition} guarantees the noise will cancel the Laplacian (as in \pref{subsec:strategy}). This algorithm is described formally as \pref{alg:sgldmodlaplacian} in \pref{subsec:loosenconditionproofs}. Note \pref{eq:loosenedcondition} subsumes the condition \pref{eq:oldconditiongf} implied by the success of gradient flow: in particular it contains P\LCHAR functions and K\LCHAR functions (see \pref{subsec:loosenconditiondiscussion}).

\begin{theorem}\label{thm:discretizationloosenoracle}
Suppose that we have \pref{eq:loosenedcondition} for some $\beta =\tilde{\Omega}\prn*{\frac{d}{\epsilon}}$, and $F$, $\Phi$ satisfy the same assumptions as \pref{thm:gradexactFeps}.
Moreover suppose we have query access to $\tri*{ \nabla \Phi(\vecW), \nabla F(\vecW)} - F_{\epsilon}(\vecW)$. Now run modified Langevin Dynamics as described above with constant step size $\eta$. This will, with probability at least $1-\delta$, yield a sequence of iterates $\vecW_t$ with $\frac1T \sum_{t=1}^T F(\vecW_t) = \widetilde{O}\prn*{\epsilon}$ in at most $O\prn*{\frac{1}{\epsilon^2}\log\prn*{\nicefrac1{\delta}}}$ iterations. 
\end{theorem}
We defer the proof to \pref{subsec:loosenconditionproofs}, since it heavily relies on the analysis done in \pref{sec:corediscreteproofs}. Again, note in all these results \pref{thm:gradexactFeps}, \ref{thm:gradestimateFeps}, and \ref{thm:discretizationloosenoracle} that $\beta = \widetilde{\Omega}\prn*{\frac{d}{\epsilon}}$.

\subsubsection*{Acknowledgements} 
AS thanks Adam Block and Sasha Rakhlin for useful discussions and acknowledges support from the Simons Foundation and 
NSF through award DMS-2031883, as well as from DOE through the award DE-SC0022199. 

\clearpage


\bibliography{sources}

\clearpage 

\clearpage 
\renewcommand{\contentsname}{Contents of Appendix}
\tableofcontents
\addtocontents{toc}{\protect\setcounter{tocdepth}{3}}





\clearpage 

\section{Setup for Rest of Paper}\label{sec:setup}
The appendix is organized as follows. We derive our `continuous time' results (\pref{thm:maingeometriccondition} and \pref{thm:constructadmissablepotential}) in \pref{sec:continuoustimeproofs}. 
We present our proofs for \pref{sec:beyondpoincare} first in \pref{sec:corediscreteproofs}, proving \pref{thm:gradexactFeps}, \ref{thm:gradestimateFeps}, and \ref{thm:discretizationloosenoracle}; later our proofs of \pref{thm:poincareoptlipschitz}, \ref{thm:poincareoptlipschitzestimate}, and \ref{thm:smoothdissipativesettingpoincareopt} rely on some of this work. Finally, in \pref{sec:PIoptproofs} we prove \pref{thm:poincareoptlipschitz}, \ref{thm:poincareoptlipschitzestimate}, and \ref{thm:smoothdissipativesettingpoincareopt}.

\subsection{Additional Notation} 
In the following, $\log$ always denotes natural logarithm. The notation $U\prn*{[a,b]}$ refers to the uniform distribution on $[a,b]$. The notation $\delta_{\cA}$ denotes the Dirac Delta on some event $\cA$. The notation $\Gamma$ refers to the Gamma function.

The notation $d(p, \cA)$ refers to the minimum distance from a point $p \in \mathbb{R}^d$ to a set $\cA \subset \mathbb{R}^d$. For a set $\cU \subset \mathbb{R}^d$, $\partial \cU$ denotes its boundary. For a vector $\vecW\in\mathbb{R}^d$, $\vecW_i$ refers to its $i$-th coordinate. For a $k$-th order tensor operator $T$ and $\vecV_1, \ldots, \vecV_k\in\mathbb{R}^d$, $T\brk*{\vecV_1,\ldots,\vecV_k}$ refers to applying $T$ to the $k$-th order tensor $\vecV_1 \otimes \dots \otimes \vecV_k$, that is, $\tri*{T, \vecV_1 \otimes \dots \otimes \vecV_k}$. 

Again, we will refer to the measure on $\mathbb{R}^d$ proportional to $e^{-\beta F(\vecW)}$ by $\mu_{\beta}$ (the subscript shows the dependence on the temperature, which is crucial for optimization). When we write $Z$, it refers to the normalizing constant $\int_{\mathbb{R}^d} e^{-\beta F(\vecW)} \DERIV\vecW$ of the measure, unless specified otherwise (so it may change line-to-line if we refer to different measures). For any set $\cU \subset\mathbb{R}^d$, let the hitting time of the SDE \pref{eq:langevindiffusioncorrect} initialized at $\vecW$ to $\cU$ be $\tau'_{\cU}(\vecW)$.

Before we apply results from probability regarding the continuous-time Langevin Diffusion, consider the SDE 
\[ \DERIV \vecW(t) = -\beta\grad  F(\vecW(t)) \DERIV t + \sqrt{2} \DERIV \vecB(t). \numberthis\label{eq:langevindiffusioncorrect}\]
We refer to this SDE when we directly use results from \citet{cattiaux2013poincare} and \citet{cattiaux2017hitting}, so that our convention for Poincar\'e and Log-Sobolev constants will match theirs. Note \pref{eq:langevindiffusioncorrect} is equivalent to \pref{eq:LangevinSDE}. For a given realization of a Brownian motion driving both SDEs, both SDEs will trace out the same path. However in \pref{eq:langevindiffusioncorrect} time passes `$\beta$ times faster' than in \pref{eq:LangevinSDE}. Hence for any set $\cU \subset\mathbb{R}^d$, the hitting time of the SDE \pref{eq:langevindiffusioncorrect} to $\cU$ is $\frac1{\beta}$ (i.e. faster if $\beta \ge 1$) than that of the hitting time of \pref{eq:LangevinSDE} to $\cU$, if both SDEs are driven by the same Brownian motion. That is, using our notation, we have $\tau'_{\cU} = \frac1{\beta} \tau_{\cU}$ for all $\cU \subset \mathbb{R}^d$. 

\section{Proofs for Continuous Time} \label{sec:continuoustimeproofs}
\subsection{Proof of \pref{thm:maingeometriccondition} and Related Results}\label{subsec:probabilitylyapunovresultsappendix}
Now we restate \pref{thm:maingeometriccondition} formally here. Note \pref{thm:maingeometricconditionformal} requires us to control $\mu\prn*{\cA_{\epsilon}}$ in the $\lambda$ from \pref{thm:maingeometriccondition}, for which we need \pref{lem:measureoflargeF}. \pref{lem:measureoflargeF} is precisely where we need $\epsilon = \widetilde{\Omega}\prn*{\frac{d}{\beta}}$. This leads to consistency between our results and our discussion from \pref{subsec:ourcontributions}. We defer \pref{lem:measureoflargeF} to later in this section and note \pref{thm:maingeometriccondition} follows immediately from combining \pref{thm:maingeometricconditionformal} and \pref{lem:measureoflargeF}.
\begin{theorem}\label{thm:maingeometricconditionformal}
Assume that $\mu_{\beta}$ satisfies a Poincar\'e inequality with constant $\CPI(\mu_{\beta})$. Then there exists a non-negative Lyapunov function $\Phi$ differentiable to all others such that on $\cA_{\epsilon}^c$, we have $\Phi \ge 1$ and
\[ -\tri*{\grad F(\vecW), \grad \Phi(\vecW)} + \frac1{\beta} \Delta\Phi(\vecW) \le-\lambda\Phi(\vecW),\numberthis\label{eq:poincaregeomcondition}\] 
where
\[ \lambda = \frac1{\beta}\mu_{\beta}\prn*{\cA_{\epsilon}} \min\prn*{\frac1{4\CPI(\mu_{\beta})}, \frac18}.\]
In fact, on $\cA_{\epsilon}^c$, $\Phi$ has the explicit form
\[ \Phi(\vecW') = \mathbb{E}_{\vecW'}\brk*{\exp\prn*{\lambda \tau_{\cA_{\epsilon}}}}.\]
\end{theorem}
\begin{proof}
We first need to introduce some concepts from Markov processes and Partial Differential Equations (PDEs). First, we introduce the concept of the (infinitesimal) generator of a Markov process, which will make this exposition much more natural. We give only what is needed for our proof and refer the reader to \citet{chewi2024log} for more details. 
\begin{definition}
The (infinitesimal) generator of a Markov process $\vecW(t)$ is the operator $\mathcal{L}$ defined on all (sufficiently differentiable) functions $f$ by 
\[ \mathcal{L}f(\vecW) = \lim_{t\rightarrow0}\frac{\mathbb{E}\brk*{f(\vecW(t))}-f(\vecW)}{t}.\]
\end{definition}
It is well-known and can be easily checked that for the Langevin Diffusion given in the form \pref{eq:langevindiffusioncorrect}, the generator
\[ \mathcal{L}f(\vecW) = -\tri*{\beta \grad F(\vecW), \grad f(\vecW)} + \Delta f(\vecW).\numberthis\label{eq:generatorlangevin}\]
For example, this calculation can be found in Example 1.2.4 of \citet{chewi21analysis}.

Note the similarity of the above to \pref{eq:lyapunovbaby}. This is no coincidence; our discrete-time proofs, specifically \pref{lem:onesteprecursion} and \pref{lem:onesteprecursionstochastic}, are essentially re-deriving the generator of the Langevin diffusion. In \pref{lem:onesteprecursion} and \pref{lem:onesteprecursionstochastic} we Taylor expand to third order (so we have the full second order quadratic form); intuitively that is all that is needed by It\^{o}'s Lemma.

We also need to introduce the idea of symmetry of the measure $\mu_{\beta}$ with respect to the stochastic process. In particular, we say $\mu_{\beta}$ is \textit{symmetric} (with respect to the Langevin Diffusion \pref{eq:langevindiffusioncorrect}) if for all infinitely differentiable $f, g$, 
\[ \int f \cL g \DERIV\mu_{\beta} = \int \cL f g \DERIV\mu_{\beta}.\]
Here $\cL$ refers to the generator \pref{eq:generatorlangevin} for the Langevin Diffusion \pref{eq:langevindiffusioncorrect}. It is well-known and can be easily checked again that $\mu_{\beta}$ is symmetric, see Example 1.2.18 of \citet{chewi21analysis} or the discussion on page 3 of \citet{cattiaux2017hitting}.

Finally, we need to introduce some ideas from PDE theory. Consider a second-order differential operator 
\[ \mathcal{P} = \frac12\sum_{1 \le i < j \le d} a_{ij}\frac{\partial^2}{\partial \vecW_i \partial \vecW_j} + \sum_{1 \le i \le d} b_i \frac{\partial}{\partial \vecW_i} + c.\]
The following definitions generalize far beyond second-order differential operators, but this is all we need for our work. We say that $\mathcal{P}$ is \textit{elliptic} if, for every $\vecW \neq 0 \in \mathbb{R}^d$, 
\[ \sum_{1 \le i, j \le d} a_{ij} \vecW_i \vecW_j \neq 0.\]
We say $\mathcal{P}$ is uniformly elliptic if we can write 
\[\mathcal{P} = \frac12\sum_{1 \le i < j \le d} \prn*{\mathbf{\sigma} \mathbf{\sigma}^T}_{ij} \frac{\partial^2}{\partial \vecW_i \partial \vecW_j} + \sum_{1 \le i \le d} b_i \frac{\partial}{\partial \vecW_i} +c,\]
for some $\mathbf{\sigma} \in \mathbb{R}^d$ where uniformly on $\mathbb{R}^d$ we have 
\[ \mathbf{\sigma} \mathbf{\sigma}^T\succcurlyeq a > 0\]
in the PSD order \citep{street2018else, cattiaux2017hitting}.

A canonical example of $\mathcal{P}$ that is uniformly elliptic is the Laplacian, where $a_{ij} = 2\delta_{i=j}$ \citep{yang2020hypoellipticity}. Beyond this, note for the Langevin Diffusion \pref{eq:langevindiffusioncorrect}, we have $a_{ij}=2\delta_{i=j}$ as well, from \pref{eq:generatorlangevin}. Thus, it is clear that $\cL$ for the Langevin Diffusion \pref{eq:langevindiffusioncorrect} is uniformly elliptic.

Ellipticity is well-known to imply that solutions $u$ to the Dirichlet problem $\mathcal{P} u = f$ in some open domain $\Omega \subset \mathbb{R}^d$ are smooth, which is all we need here \citep{yang2020hypoellipticity}.\footnote{For this, ellipticity is sufficient but not necessary. The loosest such condition for this is hypoellipticity \citep{street2018else, yang2020hypoellipticity}, which is not relevant for this work.} Ellipticity implies maximal hypoellipticity, which in turn implies strong hypoellipticity/Hormander's condition from \citet{cattiaux2013poincare}, as discussed in \citet{yang2020hypoellipticity}. Thus uniform ellipticity implies strong uniform hypoellipticity as defined in \citet{cattiaux2013poincare}. Using the results of \citet{cattiaux2013poincare} requires strong uniform hypoellipticity and symmetry with respect to the stochastic process, and \citet{cattiaux2017hitting} requires uniform ellipticity and symmetry. We have uniform ellipticity and symmetry, and so can use all those results.

Now we move to the main proof. Our main tool is Theorem 2.1 of \citet{cattiaux2017hitting}, which connects Poincar\'e Inequalities to more explicit geometric conditions that we can use in an `optimization-styled' proof analysis later.\footnote{We presume here $F$ is sufficiently differentiable to use the results of \citet{cattiaux2013poincare} and \citet{cattiaux2017hitting}, for example this holds if $F$ is infinitely differentiable. The careful reader will notice that $F$ can be approximated by an infinitely differentiable function to arbitrary precision. We also assume the boundary $\partial \cA_{\epsilon} = \{\vecW:F(\vecW)=\epsilon\}$ is differentiable to all orders, non-characteristic for \pref{eq:langevindiffusioncorrect} in the sense described in \citet{cattiaux2013poincare} and \citet{cattiaux2017hitting}, and has Lebesgue measure 0. In the $F$ Lipschitz case we assume this set is bounded and hence compact; boundedness and hence compactness follows from \pref{ass:weakdissipation} in all other cases. Since we can approximate $F$ by an infinitely differentiable function to arbitrary precision, this boundary in turn will be infinitely differentiable.} Specialized to the Langevin Diffusion \pref{eq:langevindiffusioncorrect} on the domain $\cD=\mathbb{R}^d$, it states the following:
\begin{theorem}[Theorem 2.1 of \citet{cattiaux2017hitting}]\label{thm:cattiauxopenresult}
Suppose that $\mu_{\beta}$ satisfies a Poincar\'e Inequality with constant $\CPI(\mu_{\beta})$. Then for all open subsets $\cU$ of $\mathbb{R}^d$, there exists a function $\Phi$ differentiable to all orders such that on $\cU^c$ we have $\Phi \ge \delta' > 0$ for some $\delta'$, as well as
\begin{align*} \mathcal{L}\Phi(\vecW)=-\tri*{\beta \grad F(\vecW), \grad \Phi(\vecW)} + \Delta\Phi(\vecW) \le -\lambda'\Phi(\vecW),\numberthis\label{eq:originalphiineq}\end{align*}
where $\lambda'=\mu_{\beta}\prn*{\mathcal{U}} \min\prn*{\frac1{4\CPI(\mu_{\beta})}, \frac18}$.
\end{theorem}
Note to prove this result in $\cD=\mathbb{R}^d$ all that is needed is ellipticity, which is clearly satisfied here in the case of the Langevin diffusion (following the discussion on page 9 of \citet{cattiaux2017hitting}). Hence, applying \pref{thm:cattiauxopenresult} with $\cU=\{\vecW:F(\vecW)<\epsilon\}$ which is clearly open, this gives the existence of such a $\Phi$.

Suppose $\{\vecW:\Phi(\vecW) \le \frac{\delta'}2\} \neq \emptyset$. In this case, consider $\{\vecW:\Phi(\vecW) \le \frac{\delta'}2\} \subset \{\vecW:\Phi(\vecW) < \frac{3\delta'}4\} \subset \{\vecW:F(\vecW)<\epsilon\}$. Apply the standard construction of bump functions to the compact set $\{\vecW:\Phi(\vecW) \le \frac{\delta'}2\}$ contained in the open set $\{\vecW:\Phi(\vecW) < \frac{3\delta'}4\}$ to obtain a function $\chi$ differentiable to all orders supported on $\{\vecW:\Phi(\vecW) < \frac{3\delta'}4\}$ and identically 1 on $\{\vecW:\Phi(\vecW) \le \frac{\delta'}2\}$. Let $B = \inf \Phi \le \frac{\delta'}2$. It is easy to check that $\Phi + \prn*{\frac{\delta'}2 + \max\prn*{0, -B}}\chi \ge \frac{\delta'}2$, and differentiable to all orders as $\Phi$ and $\chi$ are, and is identical to $\Phi$ on $\{\vecW:F(\vecW)\ge\epsilon\}$. Taking $\Phi \leftarrow \Phi + \prn*{\frac{\delta'}2 + \max\prn*{0, -B}}\chi \ge \frac{\delta'}2$, this gives us the existence of $\Phi \ge \frac{\delta'}2$ differentiable to all orders where we know on $\{\vecW:F(\vecW)\ge\epsilon\}$, it satisfies \pref{eq:originalphiineq}.

Notice $\mu_{\beta}\prn*{\{\vecW:F(\vecW)<\epsilon\}} = \mu_{\beta}\prn*{\cA_{\epsilon}}$, since $\mu_{\beta}\prn*{\partial \cA_{\epsilon}}=\mu_{\beta}\prn*{\{\vecW:F(\vecW)=\epsilon\}}$ is simply a positive constant times the Lebesgue measure of $\partial \cA_{\epsilon}$, and hence is 0. Therefore we know for this $\Phi$,
\[\mathcal{L}\Phi(\vecW)=-\tri*{\beta \grad F(\vecW), \grad \Phi(\vecW)} + \Delta\Phi(\vecW) \le-\lambda'\Phi(\vecW)=-\beta \lambda \Phi(\vecW).\numberthis\label{eq:geomconditionproof}\]
We claim with such a $\Phi$, the moment generating function $\mathbb{E}_{\vecW'}\brk*{\exp\prn*{\beta\lambda \tau'_{\cA_{\epsilon}}}}$ exists (i.e. is finite). The argument is done explicitly on page 8 of \citet{cattiaux2013poincare} (connectivity of $\cA$ is not necessary, as one will see below). We write it here explicitly here for the reader. Clearly this MGF is finite for $\vecW' \in \cA_{\epsilon}$, so consider any $\vecW' \in \cA_{\epsilon}^c$. Consider any $t<\infty$, any $R < \infty$ and consider the hitting time $\tau'_{\cA_{\epsilon} \cup \ball\prn*{\vecOrigin, R}^c}$. Denote $\tau'_{t,\epsilon,R} := t \land \tau'_{\cA_{\epsilon} \cup \ball\prn*{\vecOrigin, R}^c}$ for short, which is clearly a stopping time. Apply Dynkin's Formula to the map $(s,\vecW) \rightarrow e^{\beta \lambda s} \Phi(\vecW)$ with the stopping time $\tau'_{t,\epsilon,R}$; thus for all $s<\tau'_{t,\epsilon,R}$, we know $\Phi\prn*{\vecW(s)}$ satisfies \pref{eq:geomconditionproof}. We obtain: 
\begin{align*}
\frac{\delta'}2 \mathbb{E}_{\vecW'}\brk*{\text{exp}\prn*{\beta \lambda \tau'_{t,\epsilon,R}}} &\le \mathbb{E}_{\vecW'}\brk*{\text{exp}\prn*{\beta \lambda\tau'_{t,\epsilon,R}} \Phi\prn*{\vecW(\tau'_{t,\epsilon,R})}} \\
&= \Phi(\vecW') + \mathbb{E}_{\vecW'}\brk*{ \int_0^{\tau'_{t,\epsilon,R}} \text{exp}\prn*{\beta \lambda s}\prn*{\beta \lambda\Phi\prn*{\vecW(s)} + \mathcal{L}\Phi\prn*{\vecW(s)}} \DERIV s} \\
&\le \Phi(\vecW') + \mathbb{E}_{\vecW'}\brk*{ \int_0^{\tau'_{t,\epsilon,R}} \text{exp}\prn*{\beta \lambda s}\prn*{\beta \lambda\Phi\prn*{\vecW(s)} - \beta \lambda \Phi\prn*{\vecW(s)}} \DERIV s} \\
&= \Phi(\vecW').
\end{align*}
For justification, the first line above follows as $\Phi(\vecW) \ge \frac{\delta'}2$. Dynkin's Formula and then Chain Rule and It\^{o}'s Lemma are used in the second line (an analogous calculation is done formally on page 121, \citet{peskir2006optimal}). The third line uses the geometric condition \pref{eq:geomconditionproof} that we know $\Phi\prn*{\vecW(s)}$ satisfies for $s<\tau'_{t,\epsilon,R}$. 

Thus, we have for all $t<\infty$, $R<\infty$ that 
\[ \mathbb{E}_{\vecW'}\brk*{\text{exp}\prn*{\beta \lambda\tau'_{t,\epsilon,R}}} \le \frac{2\Phi(\vecW')}{\delta'}<\infty.\]
Recalling $\delta'>0$ is independent of $R,t$, letting first $R \rightarrow \infty$ and then $t \rightarrow \infty$, Dominated Convergence Theorem gives the result $\mathbb{E}_{\vecW'}\brk*{\exp\prn*{\beta\lambda \tau'_{\cA_{\epsilon}}}} \le \frac{2\Phi(\vecW')}{\delta'} < \infty$ (since the right hand side above is a finite upper bound independent of $R,t$). 

We now claim the moment generating function $\mathbb{E}_{\vecW'}\brk*{\exp\prn*{\beta\lambda \tau'_{\cA_{\epsilon}}}}$, which we now know exists, satisfies \pref{eq:poincaregeomcondition}. In fact this holds as an \textit{equality} on $\cA_{\epsilon}^c$ (although we don't need this). This is shown on page 8 of \citet{cattiaux2013poincare} and discussed on page 12 of \citet{cattiaux2017hitting}. Thus, here we just give a sketch; it follows by literature on PDEs, specifically Dirichlet problems. The result used to prove this is result 1 of Section 7.2 of \citet{peskir2006optimal}:
\begin{theorem}[Result 1 of Section 7.2 of \citet{peskir2006optimal}]\label{thm:fundamentaldirichletproblem}
Let $\cU$ be a bounded, open subset of $\mathbb{R}^d$. Given a continuous function $L:\cU \rightarrow\mathbb{R}$ define
\[ F(\vecW) = \mathbb{E}_{\vecW}\brk*{\int_0^{\tau'_{\cU^c}} L\prn*{\vecW(t)}\DERIV t},\]
where $\vecW(t)$ here denotes the iterates of any diffusion process and $\tau'_{\cU^c}$ denotes the hitting time of $\vecW(t)$ to $\cU^c$. Then $F$ solves the Dirichlet problem 
\[ \mathcal{L} F = -L \text{ in }\cU, F|_{\partial \cU}=0.\]
Here, $\mathcal{L}$ is the generator of this diffusion.
\end{theorem}
Consider any $R<\infty$. Consider $\cU_{\epsilon,R} := \cA_{\epsilon}^c \cap \{\vecW:\nrm*{\vecW}<R\}$, which is clearly open. Now, we apply the same reasoning as Result 4 of Section 7.2 of \citet{peskir2006optimal} (the \textit{killed} version of the Dirichlet problem), except now we want to study the \textit{created} version of the Dirichlet problem\footnote{See Section 5.4, \citep{peskir2006optimal}.}. There is not much difference, thus we just give a sketch and refer the reader to Result 4 of Section 7.2 of \citet{peskir2006optimal} and again page 8 of \citet{cattiaux2013poincare}. Let $L \equiv \beta \lambda$ be a constant function and now let $\vecW(t)$ denotes the iterates of the Langevin diffusion \pref{eq:langevindiffusioncorrect}. Consider
\[ F(\vecW) = \mathbb{E}_{\vecW}\brk*{\int_0^{\tau'_{\cU_{\epsilon,R}^c}} e^{\beta \lambda t} \beta \lambda \DERIV t} = \mathbb{E}_{\vecW}\brk*{\int_0^{\tau'_{\cU_{\epsilon,R}^c}} e^{\beta \lambda t} L\prn*{\vecW(t)}\DERIV t}. \]
where $\tau'_{\cU^c}$ now is consistent with our definition from \pref{sec:setup}, being for the Langevin Diffusion \pref{eq:langevindiffusioncorrect}. Observe that
\[ F(\vecW)+1=\mathbb{E}_{\vecW}\brk*{1+\int_0^{\tau'_{\cU_{\epsilon,R}^c}} e^{\beta \lambda t} \beta \lambda \DERIV t} = \mathbb{E}_{\vecW}\brk*{e^{\beta \lambda \tau'_{\cU_{\epsilon,R}^c}}} \le \mathbb{E}_{\vecW}\brk*{e^{\beta \lambda \tau'_{\cA_{\epsilon}}}} < \infty,\]
since $\frac{\partial}{\partial t} e^{\beta \lambda t} = \beta \lambda e^{\beta \lambda t}$, $\tau'_{\cU_{\epsilon,R}^c} \le \tau'_{\cA_{\epsilon}}$. Hence, $F(\vecW)<\infty$ and so we may continue to analyze it.

Now consider $\tilde{\vecW}(t) := e^{\beta \lambda t}\vecW(t)$ (the created process). By the same reasoning as in Result 4 of Section 7.2 of \citet{peskir2006optimal} but for the created rather than killed process, we have $F(\vecW) = \mathbb{E}_{\vecW}\brk*{\int_0^{\tilde{\tau'}_{\cU_{\epsilon,R}^c}} L\prn*{\vecW(t)}\DERIV t}$ where $\tilde{\tau'}_{\cU_{\epsilon,R}^c}$ denotes the hitting time of $\tilde{\vecW}(t)$ to $\cU_{\epsilon,R}^c$. Let the generator of $\tilde{\vecW}(t)$ be $\tilde{\cL}$. Now, \pref{thm:fundamentaldirichletproblem} implies that $F(\vecW)$ solves the Dirichlet problem 
\[ \tilde{\cL}F = -L = -\beta \lambda \text{ in }\cU_{\epsilon,R}, F|_{\partial \cU_{\epsilon,R}} = 0.\]
It can be readily seen that by Chain Rule that $\tilde{\cL} = \cL + \beta\lambda$; this calculation is done formally on page 121, \citet{peskir2006optimal}. Therefore, we have
\[ -\beta \lambda = \tilde{\cL}F = \cL F + \beta \lambda F\text{ in } \cU_{\epsilon,R}, F|_{\partial \cU_{\epsilon,R}} = 0.\]
Therefore, $\Phi_R = F+1$ satisfies (note $\cL \Phi_R = \cL F$)
\[ \cL \Phi_R = \cL F = -\beta \lambda (F+1) = -\beta \lambda \Phi_R\text{ in } \cU_{\epsilon,R}, \Phi_R|_{\partial \cU_{\epsilon,R}}=1.\]
Note we showed earlier
\[ \Phi_R(\vecW)=F(\vecW)+1=\mathbb{E}_{\vecW}\brk*{1+\int_0^{\tau'_{\cU_{\epsilon,R}^c}} e^{\beta \lambda t} \beta \lambda \DERIV t} = \mathbb{E}_{\vecW}\brk*{e^{\beta \lambda \tau'_{\cU_{\epsilon,R}^c}}}.\]
Finally, since we've already shown $\mathbb{E}_{\vecW}\brk*{e^{\beta \lambda \tau'_{\cA_{\epsilon}}}}<\infty$, the same argument of page 8 of \citet{cattiaux2013poincare} shows that the pointwise limit
\[ \Phi(\vecW) := \mathbb{E}_{\vecW}\brk*{e^{\beta \lambda \tau'_{\cA_{\epsilon}}}} = \lim_{R\rightarrow\infty} \mathbb{E}_{\vecW}\brk*{e^{\beta \lambda \tau'_{\cU_{\epsilon,R}^c}}} \]
exists and solves the Dirichlet Problem
\[ \cL \Phi = -\beta \lambda \Phi \text{ in }\lim_{R\rightarrow\infty}\cU_{\epsilon,R} \cap \{\vecW:\nrm*{\vecW}<R\} =\cA_{\epsilon}^c.\]
Thus, it satisfies \pref{eq:poincaregeomcondition}. Moreover, since $\cL$ is elliptic (and therefore hypoelliptic), the resulting solution
\[ \Phi(\vecW) = \mathbb{E}_{\vecW}\brk*{e^{\beta \lambda \tau'_{\cA_{\epsilon}}}}\]
is differentiable to all orders in $\lim_{R\rightarrow\infty}\cA_{\epsilon}^c \cap \{\vecW:\nrm*{\vecW}<R\} = \cA_{\epsilon}^c$. Note since the quantity in the exponential is always non-negative pointwise, $\Phi(\vecW) \ge 1$ on $\cA_{\epsilon}^c$. 

Since the boundary $\partial \cA_{\epsilon}=\{\vecW:F(\vecW)=\epsilon\}$ is compact and differentiable to all orders, through a standard compactness and $\delta-\epsilon$ argument we can show by defining
\[ \Phi(\vecW) = \lim_{\vecW' \rightarrow \vecW, \vecW' \in \cA_{\epsilon}^c} \Phi(\vecW') \FORALLTEXT \vecW \in \partial \cA_{\epsilon},\]
the resulting $\Phi$ is differentiable to all orders on $\cA_{\epsilon}^c \cup \partial \cA_{\epsilon}$ (when we define derivatives as the limits coming from outside $\cA_{\epsilon}^c$). (Compactness here is important.) As $\cA_{\epsilon}^c \cup \partial \cA_{\epsilon}$ is closed, applying Whitney's Extension Theorem as mentioned in \cite{cattiaux2013poincare}, $\Phi$ above can be extended to a function differentiable to all orders on all of $\mathbb{R}^d$ so that \pref{eq:poincaregeomcondition} holds on $\{\vecW:F(\vecW) \ge \epsilon\}$. Note $\Phi \ge 1$ on $\{\vecW:F(\vecW) \ge \epsilon\}$.

Suppose the resulting $\Phi$ from the extension was not non-negative. Let $B:=\inf \Phi <0$. Observe $\{\vecW:\Phi(\vecW) \le 0\} \subset \{\vecW:\Phi(\vecW) < \frac12\} \subset \cA_{\epsilon}$. Apply the standard construction of bump functions to the compact set $\{\vecW:\Phi(\vecW) \le 0\}$ contained in the open set $\{\vecW:\Phi(\vecW) < \frac12\}$ to obtain a function $\chi$ differentiable to all orders supported on $\{\vecW:\Phi(\vecW) < \frac12\}$ and identically 1 on $\{\vecW:\Phi(\vecW) \le 0\}$. Then $\Phi-B\chi$ is non-negative (recall $B<0$) and differentiable to all orders, and is identical to $\Phi$ on $\cA^c_{\epsilon}$. Taking $\Phi \leftarrow \Phi-B\chi$, this gives us the existence of $\Phi \ge 0$ differentiable to all orders where we have its explicit form and know it satisfies \pref{eq:geomconditionproof} and therefore \pref{eq:poincaregeomcondition} (upon dividing both sides by $\beta>0$) on $\cA^c_{\epsilon}$.

To conclude, note from our remarks from \pref{sec:setup} that
\[ \tau'_{\cA_{\epsilon}}(\vecW') = \frac1{\beta} \tau_{\cA_{\epsilon}}(\vecW').\]
Therefore on $\cA_{\epsilon}^c$ we can also write
\[ \Phi(\vecW') = \mathbb{E}_{\vecW'}\brk*{\exp\prn*{\lambda \tau_{\cA_{\epsilon}}}} \ge 1.\]
This completes the proof.
\end{proof}

Now we prove \pref{lem:measureoflargeF}.
\begin{lemma}\label{lem:measureoflargeF}
Suppose $F$ satisfies \pref{ass:holderF} and $\mu_{\beta}$ has finite second moment $S<\infty$. Then for $\epsilon \ge \frac{2d}{\beta} \log(4\pi e \beta L d S)$, we have $\mu_{\beta}\prn*{\cA_{\epsilon}} \ge \frac12$.
\end{lemma}
\begin{proof}
As $F(\vecW)$ is non-negative, by Markov's Inequality, we have
\[ \mu_{\beta}\prn*{\cA_{\epsilon}^c}=\mu_{\beta}\prn*{\{\vecW:F(\vecW) > \epsilon\}} \le \frac{\mathbb{E}_{\vecW\sim\mu_{\beta}}\brk*{F(\vecW)}}{\epsilon}.\]
Now we compute $\mathbb{E}_{\vecW\sim\mu_{\beta}}\brk*{F(\vecW)}$ with the same strategy as in the proof of Proposition 11 of \citet{raginsky2017non}. Write
\[ \mathbb{E}_{\vecW\sim\mu_{\beta}}\brk*{F(\vecW)} = \int_{\mathbb{R}^d} F(\vecW) \mu_{\beta}(\vecW)\DERIV \vecW = \frac1{\beta}\prn*{h(\mu_{\beta}) - \log Z}.\]
Here $Z$ is the partition function of $\mu_{\beta}$ and 
\[ h(\mu_{\beta}) = -\int_{\mathbb{R}^d}\mu_{\beta}(\vecW)\log\mu_{\beta}(\vecW) \DERIV \vecW\]
is the differential entropy of $\mu_{\beta}$.

To upper bound the differential entropy of $\mu_{\beta}$, we use the same derivation as the proof of Proposition 11 of \citet{raginsky2017non}. The assumption that $\int_{\mathbb{R}^d}\nrm*{\vecW}^2 \DERIV \mu_{\beta}(\vecW) \le S$, as well as the fact that the differential entropy of a measure with finite second moment is upper bounded by the differential entropy of a Gaussian with the same second moment, yields
\[ h(\mu_{\beta}) \le \frac{d}2 \log\prn*{\frac{2\pi e S}{d}}.\]
Now we aim to lower bound the partition function $Z$. Using \pref{lem:upperboundFholder} and \pref{lem:extensionofgaussianintegral}, we obtain
\begin{align*}
\log Z &= \log \int_{\mathbb{R}^d} e^{-\beta F(\vecW)} \DERIV \vecW \\
&\ge \log \int_{\mathbb{R}^d} e^{-\beta L \nrm*{\vecW-\vecW^{\star}}^{s+1}} \DERIV \vecW\\
&= \log \int_{\mathbb{R}^d} e^{-\beta L \nrm*{\vecW}^{s+1}} \DERIV \vecW \\
&= \log \prn*{\frac{2\pi^{d/2}}{\Gamma(d/2)} \cdot \frac1{s+1} \cdot \prn*{\beta L}^{-\frac{d}{s+1}}\cdot \Gamma\prn*{\frac{d}{s+1}}}.
\end{align*}
It is well known that on $\mathbb{R}_{>0}$, $\Gamma(\cdot)$ attains a constant lower bound of at least $\frac12$ (the real value is around 0.8856, but this is all we need for our purposes). Moreover, by well-known properties of $\Gamma(\cdot)$, we have $\Gamma(d/2) = \frac{d}{2} \cdot \frac{d-2}{2}\cdot \cdots\cdot \frac{d-2\lfloor d/2 \rfloor+r'+2}{2}\cdot \Gamma\prn*{\frac{d-2\lfloor d/2 \rfloor+r'}{2}}$, where $r' = 2\prn*{1-d\pmod{2}}$. Since $\frac{d-2\lfloor d/2 \rfloor+r'}{2}\in\{1/2,1\}$ and $\Gamma(1/2)=\sqrt{\pi}$, $\Gamma(1)\le 1$, this gives $\Gamma\prn*{\frac{d-2\lfloor d/2 \rfloor}{2}} \le d^{d/2} \sqrt{\pi}$. This implies (since $\beta L \ge 1$) the following very loose bound:
\begin{align*}
\log Z &\ge \log \prn*{\frac{2\pi^{d/2}}{\Gamma(d/2)} \cdot \frac1{s+1} \cdot \prn*{\beta L}^{-\frac{d}{s+1}}\cdot \Gamma\prn*{\frac{d}{s+1}}} \\
&\ge \log \prn*{\frac{\pi^{d/2}}{2\sqrt{\pi}\prn*{\beta L}^{d} d^{d/2}}} \\
&\ge -d\log\prn*{2\beta L d}.
\end{align*}
Hence, we see 
\[ \mathbb{E}_{\vecW\sim\mu_{\beta}}\brk*{F(\vecW)} = \frac1{\beta}\prn*{h(\mu_{\beta}) - \log Z} \le \frac{d}{\beta}\prn*{\frac12\log\prn*{\frac{2\pi e S}{d}} + \log\prn*{2\beta L d}} \le \frac{d}{\beta} \log(4\pi e \beta L d S).\]
The conclusion follows from our condition on $\beta$ and the original application of Markov's Inequality.

Note it suffices to just take $\epsilon \ge 2\mathbb{E}_{\vecW\sim\mu_{\beta}}\brk*{F(\vecW)}$ to make this proof work; most of our work was to find a suitable upper bound for $\mathbb{E}_{\vecW\sim\mu_{\beta}}\brk*{F(\vecW)}$. Also, $\epsilon =\Omega\prn*{\mathbb{E}_{\vecW\sim\mu_{\beta}}\brk*{F(\vecW)}}$ is necessary, as demonstrated by the Gaussian example in \pref{subsec:ourcontributions}.
\end{proof}

\subsection{Proof of \pref{thm:constructadmissablepotential}}\label{subsec:optrategeometryformal}
We derive the implication of a rate function $R(\vecW,t)$ satisfying the condition \pref{eq:ldrateineq} to a geometric condition, which we described in \pref{sec:beyondpoincare}. First we convert \pref{eq:ldrateineq} into a more tractable condition about the rate function:
\begin{lemma}\label{lem:deriveadmissibleformal}
Assume $R$ has continuous second partials and that
\[ \mathbb{E}\brk*{\abs*{\tri*{R(\vecW,t), \nabla F(\vecW)}}}, \mathbb{E}\brk*{\abs*{\Delta R(\vecW(s),t)}} < \infty, \mathbb{P}\prn*{\int_0^{\infty} \nrm*{R(\vecW(\tau),t)}_2^2 \DERIV\tau<\infty}=1. \]
Then, we have that $R$ satisfying \pref{eq:ldrateineq} implies the condition:
\[ \frac{\partial}{\partial t}R(\vecW,t) \ge -\langle \nabla R(\vecW,t), \nabla F(\vecW)\rangle+\frac1{\beta}\Delta R(\vecW,t) \FORALLTEXT \vecW\in\mathbb{R}^d, t \ge 0.\numberthis\label{eq:derivconditionadmissibile}\]
\end{lemma}
\begin{proof}
Note from $\mathbb{E}\brk*{R(\vecW(s),t)} \le R(\vecW,s+t) \FORALLTEXT s, t, \vecW = \vecW(0)$ we have
\[ \lim_{s\rightarrow 0} \frac{\mathbb{E}\brk*{R(\vecW(s),t)}-\mathbb{E}\brk*{R(\vecW(0),t)}}{s}\le \lim_{s\rightarrow 0}\frac{R(\vecW,s+t)-R(\vecW,t)}{s}= \frac{\partial}{\partial t}R(\vecW,t).\]
In the following we consider $t$ as a fixed constant. Note $\mathbb{E}\brk*{R(\vecW(\cdot),t)}$ is a deterministic function of the argument and so by definition of partial derivative
\[ \lim_{s\rightarrow 0} \frac{\mathbb{E}\brk*{R(\vecW(s),t)}-\mathbb{E}\brk*{R(\vecW(0),t)}}{s} = \frac{\partial}{\partial s} \mathbb{E}\brk*{R(\vecW(s),t)}\Big|_{s=0}.\]
Thus the above becomes
\[ \frac{\partial}{\partial s} \mathbb{E}\brk*{R(\vecW(s),t)}\Big|_{s=0} \le \frac{\partial}{\partial t}R(\vecW,t).\]
Now, recall Langevin Dynamics \pref{eq:LangevinSDE} is given by the SDE
\[ \DERIV\vecW(s) = -\nabla F(\vecW(s)) \mathrm{d}s + \sqrt{2/\beta} \mathrm{d}\vecB(s),\]
where $\vecB(s)$ is the standard Brownian motion in $\mathbb{R}^d$ and $\beta$ is the inverse temperature parameter. This SDE is a compact way of writing
\[ \vecW(s) = \vecW(0)-\int_0^s \nabla F(\vecW(\tau)) \mathrm{d}\tau + \int_0^s \sqrt{2/\beta} \mathrm{d}\vecB(\tau).\]
The next step is to find the corresponding It\^{o} process that describes $R(\vecW(s), t)$. By It\^{o}'s Lemma, 
\[ \mathrm{d}R(\vecW(s), t) =\sum_{i=1}^d \prn*{\frac{\partial R}{\partial \vecW_i}(\vecW(s),t)} \mathrm{d}\vecW(s)_i + \frac12 \sum_{1 \le i,j\le d} \prn*{\frac{\partial^2 R}{\partial \vecW_i\partial \vecW_j}(\vecW(s),t)} \mathrm{d}\vecW(s)_i \mathrm{d}\vecW(s)_j.\]
Straightforward calculation and the fact that $(\DERIV \vecB(s))_i (\DERIV \vecB(s))_j = \delta_{i=j} \DERIV s$ gives
\begin{align*}
\DERIV \vecW(s)_i \DERIV \vecW(s)_j &= \prn*{-\nabla F(\vecW(s))_i \DERIV s + \sqrt{2/\beta} (\DERIV \vecB(s))_i}\prn*{-\nabla F(\vecW(s))_j \DERIV s + \sqrt{2/\beta} (\DERIV \vecB(s))_j}\\
&=\begin{cases} 0 &\IF i \neq j \\ \frac2{\beta} \DERIV s &\OTHERWISEIF i = j.\end{cases}
\end{align*}
Substituting this into the above we get
\begin{align*}
\DERIV R(\vecW(s),t) &= \tri*{\nabla R(\vecW(s),t), \DERIV \vecW(s)} + \frac12 \sum_{i=1}^d \prn*{\frac{\partial^2 R}{\partial \vecW_i^2}(\vecW(s),t)} \frac2{\beta} \DERIV s \\
&= \tri*{\nabla R(\vecW(s),t), \prn*{-\nabla F(\vecW(s)) \DERIV s+\sqrt{2/\beta} \DERIV \vecB(s)}} + \frac1{\beta} \Delta R(\vecW(s),t) \DERIV s \\
&= -\tri*{\nabla R(\vecW(s),t), \nabla F(\vecW(s))} + \frac1{\beta}\Delta R(\vecW(s),t) \DERIV s \\
&\hspace{1in}+ \sqrt{2/\beta} \tri*{\nabla R(\vecW(s),t), \DERIV \vecB(s)}.
\end{align*}
We can rewrite this as
\begin{align*}
R(\vecW(s),t) &= R(\vecW(0),t) + \int_0^s \prn*{-\tri*{\nabla R(\vecW(\tau),t), \nabla F(\vecW(\tau))} + \frac1{\beta}\Delta R(\vecW(\tau),t)} \DERIV \tau \\
&\hspace{1in}+ \int_0^s \sqrt{2/\beta} \tri*{\nabla R(\vecW(\tau),t), \DERIV \vecB(\tau)}.
\end{align*}
We aim to find an expression for $\mathbb{E}\brk*{R(\vecW(s)),t}$. Note by our conditions, considering Definition 3.1.4, Theorem 3.2.1 and Definition 3.3.2 together in \citet{oksendal2003stochastic}, we see that in fact we have 
\[ \mathbb{E}\brk*{\int_0^s \sqrt{2/\beta} \tri*{\nabla R(\vecW(\tau),t), \DERIV\vecB(\tau)}} = 0 \FORALLTEXT s \ge 0.\]
Thus taking expectations gives
\[ \mathbb{E}\brk*{R(\vecW(s),t)} = R(\vecW(0),t)+\mathbb{E}\brk*{\int_0^s \prn*{-\tri*{\nabla R(\vecW(\tau),t), \nabla F(\vecW(\tau))} + \frac1{\beta}\Delta R(\vecW(\tau),t)} \DERIV\tau}. \]
Now we want to deal with this expectation. By Dominated Convergence, thanks to our assumptions, we can swap the order of expectation and integration. So we obtain
\[ \mathbb{E}\brk*{R(\vecW(s),t)} = R(\vecW(0),t)+\int_0^s\mathbb{E} \brk*{-\tri*{\nabla R(\vecW(\tau),t),\nabla F(\vecW(\tau))} + \frac1{\beta}\Delta R(\vecW(\tau),t)} \DERIV\tau. \] 
Now applying Leibniz Rule gives
\[ \frac{\partial}{\partial s} \mathbb{E}\brk*{R(\vecW(s),t)} = \frac{\partial}{\partial s}(s) \cdot \mathbb{E}\brk*{-\tri*{\nabla R(\vecW(s),t),\nabla F(\vecW(s))} + \frac1{\beta}\Delta R(\vecW(s),t)}.\]
And thus our condition becomes
\begin{align*}
\frac{\partial}{\partial t}R(\vecW,t) &\ge \frac{\partial}{\partial s} \mathbb{E}\brk*{R(\vecW(s),t)} \Big|_{s=0}\\ 
&= \lim_{s\rightarrow 0}\mathbb{E}\brk*{-\tri*{\nabla R(\vecW(s),t),\nabla F(\vecW(s))} + \frac1{\beta}\Delta R(\vecW(s),t)} \\
&= -\langle \nabla R(\vecW,t), \nabla F(\vecW)\rangle+\frac1{\beta}\Delta R(\vecW,t).
\end{align*}
This last step is justified as follows. Our formula for $\frac{\partial}{\partial s} \mathbb{E}\brk*{R(\vecW(s),t)}$ holds for all $s > 0$, and our expression for $\mathbb{E}\brk*{R(\vecW(s),t)}$ is continuous in $s$. Recall our assumptions that $R$ has continuous second partials and 
\[ \mathbb{E}\brk*{\abs*{\tri*{\nabla R(\vecW,t), \nabla F(\vecW)}}}, \mathbb{E}\brk*{\abs*{\Delta R(\vecW(s),t)}}< \infty.\]
Thus, Dominated Convergence Theorem may be used to swap the order of limit and expectation, so we may take the limit of both sides as $s \rightarrow 0^{+}$, yielding
\begin{align*}
\frac{\partial}{\partial t}R(\vecW,t) &\ge -\mathbb{E}[\tri*{\nabla R(\vecW,t), \nabla F(\vecW)}] + \frac1{\beta}\Delta R(\vecW,t)\\
&= -\langle \nabla R(\vecW,t), \nabla F(\vecW)\rangle+ \frac1{\beta}\Delta R(\vecW,t),
\end{align*}
where the expectation clearly drops since we took the limit. 

To justify the application of Dominated Convergence in more detail, note $\vecW = \vecW(0)$ and $t$ here are both fixed and so for $s > 0$ small enough, we have
\[ \abs*{\tri*{\nabla R(\vecW(s),t),\nabla F(\vecW(s)}} \le \abs*{\tri*{\nabla R(\vecW,t), \nabla F(\vecW)}}+1,\]
by continuity of the gradients of $R(\vecW, t)$ and $F$. Now under the expectation with respect to the probability measure given by the Brownian motion up to time $s$ we get
\[ \mathbb{E}\brk*{\abs*{\tri*{\nabla R(\vecW(s),t), \nabla F(\vecW(s)}}}\le \mathbb{E}\brk*{\abs*{\tri*{\nabla R(\vecW,t), \nabla F(\vecW)}}}+1.\]
The same argument, since we have the appropriate conditions, can be used for $\Delta R$.
\end{proof}

Recall that we claimed in \pref{sec:beyondpoincare} that a rate function $R(\vecW, t)$ satisfying \pref{eq:ldrateineq} implied $\Phi(\vecW)=\int_0^\infty R(\vecW,t) \DERIV t$ satisfies the definition of admissible potential \pref{eq:admissablepotentialF}. Using \pref{lem:deriveadmissibleformal}, we show this now.
\begin{theorem} [Constructing an admissible potential; analogy to Theorem 2 from \citet{priorpaper}]\label{thm:constructadmissablepotentialformal}
Assuming the conditions
\[\mathbb{E}\brk*{\abs*{\tri*{\nabla R(\vecW,t), \nabla F(\vecW)}}} < \infty, \nrm*{\int_0^\infty \nabla R(\vecW,t) \DERIV t} < \infty,\int_0^\infty R(\vecW,t) \DERIV t < \infty \]
for all $\vecW$, and the assumption that $R(\vecW, t)$ and $F$ have continuous gradients. Then we know that 
\[ \Phi(\vecW)=\int_0^\infty R(\vecW,t) \DERIV t\]
is an admissable potential if $R(\vecW,t)$ satisfies the relationship \pref{eq:derivconditionadmissibile} given in \pref{lem:deriveadmissibleformal}.
\end{theorem}
\begin{proof}
Rearrange the condition from \pref{lem:deriveadmissibleformal} to read
\[ -\frac{\partial}{\partial t}R(\vecW,t) + \frac1{\beta}\Delta R(\vecW,t)\le \tri*{ \nabla R(\vecW,t), \nabla F(\vecW)}.\]
Next integrate both sides from $0 \le t < \infty$, yielding
\begin{align*}
\int_0^\infty \prn*{-\frac{\partial}{\partial t}R(\vecW,t) + \frac1{\beta}\Delta R(\vecW,t))} \DERIV t &\le \int_0^\infty \tri*{ \nabla R(\vecW,t), \nabla F(\vecW)} \DERIV t \\
&= \tri*{\int_0^\infty \nabla R(\vecW,t) \DERIV t, \nabla F(\vecW)} \\
&= \tri*{ \nabla \prn*{\int_0^\infty R(\vecW,t) \DERIV t}, \nabla F(\vecW)}.
\end{align*}
The last step follows as for all $1 \le i \le d$ we have again by Dominated Convergence Theorem that
\[ \frac{\partial}{\partial \vecW_i} \int_0^\infty R(\vecW,t) \DERIV t = \int_0^\infty \frac{\partial}{\partial \vecW_i} R(\vecW,t) \DERIV t,\]
by assumption that $\int_0^\infty R(\vecW,t) \DERIV t < \infty \FORALLTEXT \vecW$ (as $R$ is non-negative).

Next, observe that
\begin{align*}
\int_0^\infty \prn*{-\frac{\partial}{\partial t}R(\vecW,t) + \frac1{\beta}\Delta R(\vecW,t))} \DERIV t &= \int_0^\infty -\frac{\partial}{\partial t}R(\vecW,t) 
\DERIV t + \frac{1}{\beta} \int_0^\infty \Delta R(\vecW,t) \DERIV t \\
&= R(\vecW,0) + \frac1{\beta}\int_0^\infty \Delta R(\vecW,t) \DERIV t \\
&\ge F_{\epsilon}(\vecW)+ \frac1{\beta}\int_0^\infty \Delta R(\vecW,t) \DERIV t,
\end{align*}
since $\lim_{t\rightarrow \infty} R(\vecW,t)=0$ and as $R(\vecW,0)\ge F_{\epsilon}(\vecW)$. To complete the proof, note by two applications of Dominated Convergence Theorem that
\[ \int_0^\infty \frac{\partial^2}{\partial \vecW_i^2} R(\vecW,t) \DERIV t = \frac{\partial}{\partial \vecW_i} \int_0^\infty \frac{\partial}{\partial \vecW_i} R(\vecW,t) \DERIV t 
= \frac{\partial^2}{\partial \vecW_i^2} \int_0^\infty R(\vecW,t) \DERIV t,\]
by assumption that $\nrm{\int_0^\infty \nabla R(\vecW,t) \DERIV t} < \infty$ and $\int_0^\infty R(\vecW,t) \DERIV t < \infty \FORALLTEXT \vecW$ (as $R$ is non-negative). Therefore we have
\[\int_0^\infty \Delta R(\vecW,t) \DERIV t = \Delta \int_0^\infty R(\vecW,t)\DERIV t \]
and so by definition of the potential $\Phi$ we have
\[ \tri*{ \nabla \Phi(\vecW), \nabla F(\vecW)} \ge F_{\epsilon}(\vecW)+\frac1{\beta} \Delta \Phi(\vecW),\]
as wanted.
\end{proof}

Recall in \pref{sec:beyondpoincare} we stated \pref{eq:admissablepotentialF} implies that the Langevin Diffusion succeeds as an optimization strategy for rate in expectation. Here we show this.
\begin{theorem} [Getting a rate; analogy to Theorem 1 from \citet{priorpaper}]\label{thm:conditionimpliesrateproofformal}
Suppose that we have
\[\mathbb{E}\brk*{\abs*{\tri*{\nabla \Phi(\vecW,t),\nabla F(\vecW)}}}, \mathbb{E}\brk*{\abs*{\Delta \Phi(\vecW,t)}} < \infty\]
for all $\vecW\in\mathbb{R}^d$, for an admissable potential $\Phi$ with respect to $F$. Then running Langevin Dynamics at temperature $\beta$ starting from $\vecW(0)$, we have that
\[ \frac1t \int_0^t \mathbb{E}\brk*{F_{\epsilon}(\vecW(s))} \DERIV s \le \frac{\Phi(\vecW(0))}{t}.\]
That is, if we uniformly choose a stopping time in $[0,t]$ we obtain a $O(t^{-1})$ rate.
\end{theorem}
\begin{proof}
Considering an arbitrary path of \SGLD and then taking expectation with respect to the Brownian motion, the admissability condition \pref{eq:admissablepotentialF} rearranges to
\[ \mathbb{E}\brk*{F_{\epsilon}(\vecW(s))} \le \mathbb{E}\brk*{\tri*{ \nabla\Phi(\vecW(s)), \nabla F(\vecW(s))} - \frac1{\beta}\Delta \Phi(\vecW(s))} \FORALLTEXT s \ge 0,\]
because we have such an inequality pointwise by definition of admissability.

By an analogous application of It\^{o}'s formula as the above in the proof of \pref{lem:deriveadmissibleformal}, since our conditions allow us to apply Dominated Convergence Theorem, we can compute
\[ \frac{\DERIV}{\DERIV s} \mathbb{E}\brk*{\Phi(\vecW(s))}= \mathbb{E}\brk*{-\tri*{\nabla \Phi(\vecW(s)),\nabla F(\vecW(s))} +\frac{1}{\beta} \Delta \Phi(\vecW(s))},\]
and so the condition actually becomes
\[\mathbb{E}\brk*{F_{\epsilon}(\vecW(s))} \le -\frac{\DERIV}{\DERIV s}\mathbb{E}\brk*{\Phi(\vecW(s))}. \]
Integrating this from all $0 \le s \le t$ gives
\[ \int_0^t \mathbb{E}\brk*{F_{\epsilon}(\vecW(s))} \DERIV s \le -\int_0^t \frac{\DERIV }{\DERIV s} \mathbb{E}\brk*{\Phi(\vecW(s))} \DERIV s \le \mathbb{E}\brk*{\Phi(\vecW(0))} = \Phi(\vecW(0)).\]
Dividing by $t$ yields the result.
\end{proof}


\section{Proofs for \pref{sec:beyondpoincare}}\label{sec:corediscreteproofs}
In this section we prove our results with constant probability guarantees; we can recover our results from \pref{sec:beyondpoincare} easily via the standard log-boosting trick. Moreover, in this section, $\rho_{\Phi}(z)$ is defined in terms of $\rho_{\Phi,1}(z), \rho_{\Phi,2}(z), \rho_{\Phi,3}(z)$ as i \pref{lem:thirdordersmoothfromregularity} (to be stated later in this section) for the general $p$ case.
\subsection{Proof of \pref{thm:gradexactFeps}}\label{subsec:gradientoraclediscretize}
In the exact gradient oracle setting, we have the following result for optimization in discrete time. This is a formal statement of \pref{thm:gradexactFeps}.
\begin{theorem}\label{thm:discretizationsimpleoracleformal}
Consider $F$ and suppose $F$ is differentiable. Suppose that we have \pref{eq:admissablepotentialF}:
\[ \tri*{ \nabla \Phi(\vecW), \nabla F(\vecW)} \ge F_{\epsilon}(\vecW)+\frac1{\beta} \Delta \Phi(\vecW),\]
for some $\beta>0$. Suppose \pref{ass:holderF}, \ref{ass:polyselfbounding}, and \ref{ass:iteratesinballassumption} hold. Moreover suppose $\beta \ge d\sqrt{\frac{\log10}{C(\vecW_0)}}$ where $C(\cdot)$ is defined below. 

Define the following quantities:
\[ A_0(\vecW_0) = \theta\prn*{\rho_{\Phi}^{-1}\prn*{\kappa' \rho_{\Phi}\prn*{\Phi(\vecW_0)}}}-\theta\prn*{\Phi(\vecW_0)} >0, A_1(\vecW_0) = 12\sqrt{2} CB_{\textsc{grad}}^3.\]
Here $C$ comes from \pref{lem:onesteprecursion} and $B_{\textsc{grad}}:=L\prn*{R_1+\nrm*{\vecW^{\star}}}^s$, $L$ comes from \pref{ass:holderF} and $R_1$ comes from \pref{ass:iteratesinballassumption}. (If necessary take $C\leftarrow\max(C,1)$, and $B_{\textsc{grad}}\leftarrow\max(B_{\textsc{grad}},1)$.) Now define
\[ r(\vecW_0) = \min\prn*{1, \frac3{4C}, \frac1{B_{\textsc{grad}}}}, C(\vecW_0) = \min\prn*{1,\frac{A_0(\vecW_0)^2 r(\vecW_0)}{128A_1(\vecW_0)^2}}.\] 
In terms of these define (where $\theta$ comes from \pref{lem:thirdordersmoothfromregularity}),
\[M(\vecW_0) = \frac{10\max\prn*{\theta\prn*{\Phi(\vecW_0)}, 6C B_{\textsc{grad}}^3}}{\theta'\prn*{\rho_{\Phi}^{-1}\prn*{\kappa' \rho_{\Phi}\prn*{\Phi(\vecW_0)}}}} \in (0,\infty).\]

\begin{algorithm}[h!]
\caption{Discrete Time Gradient Langevin Dynamics with slight modifications}
\label{alg:langevingradoracle}
\begin{algorithmic}[1]
\If{$\epsilon \le \min\prn*{1/e, C', \sqrt{\frac{C(\vecW_0)}{\log 10}}}$, where $C'$ is an absolute, dimension and temperature free constant given in \pref{lem:actuallypotentialsimplesetting}:}
    \State Consider some constant choice of $\eta$, $T$ given in \pref{lem:actuallypotentialsimplesetting} and run the following process with $\eta$ for $T$ steps:
\[ \vecW_{t+1} \leftarrow \vecW_t - \eta \nabla F(\vecW_t) + \sqrt{2\eta/\beta} \vecEps_t.\]
Here we sample $\vecEps_t \sim \sqrt{d}\cS^{d-1}$ uniformly (in spirit a Gaussian).
\ElsIf{$\epsilon > \min\prn*{1/e, C', \sqrt{\frac{C(\vecW_0)}{\log 10}}}$:} 
\State Run this process with $\epsilon\leftarrow\min\prn*{1/e, C', \sqrt{\frac{C(\vecW_0)}{\log 10}}}$.
\EndIf
\end{algorithmic}
\end{algorithm}

Now consider running \pref{alg:langevingradoracle}. We claim it has the following guarantees. First, its runtime $T$ is as follows:
\begin{enumerate}
    \item If $\epsilon \le \min\prn*{1/e, C', \sqrt{\frac{C(\vecW_0)}{\log 10}}}$: then
    \[ T \le \begin{cases} \frac{\beta^2}{d^2}&\IF \beta \le \frac{d}{\epsilon/\prn*{\log 1/\epsilon}^2} \\
    \frac{1}{\epsilon^2}\prn*{\log 1/\epsilon}^2&\IF \beta \ge \frac{d}{\epsilon/\prn*{\log 1/\epsilon}^2}. \end{cases} \]
    \item If $\epsilon > \min\prn*{1/e, C', \sqrt{\frac{C(\vecW_0)}{\log 10}}}$: then we have the same runtime guarantee as implied by above with $\min\prn*{1/e, C', \sqrt{\frac{C(\vecW_0)}{\log 10}}}$ in place of $\epsilon$.
\end{enumerate}
In terms of error, we have with probability at least 0.75 (taken over the $\{\epsilon_t\}_{0 \le t \le T-1}$) that 
\begin{align*}
\frac1T \sum_{t=0}^{T-1}F_{\epsilon}(\vecW_t)\le \begin{cases} 6M(\vecW_0) \prn*{\frac{\log(20)+2\log(\beta/d)}{r(\vecW_0)C(\vecW_0)} + 1} \frac{d}{\beta} &\IF \beta \le \frac{d}{\epsilon/(\log 1/\epsilon)^2} \\
6M(\vecW_0) \prn*{\frac1{r(\vecW_0) C(\vecW_0)}+1}\epsilon &\IF \beta \ge \frac{d}{\epsilon/(\log 1/\epsilon)^2}. \end{cases}
\end{align*}
Here $C'$ is an absolute, dimension and temperature free constant given in \pref{lem:actuallypotentialsimplesetting}.

Note now that logarithmic boosting tricks proves that a given one of these guarantees can occur with probability at least $1-\delta$ using at most $T \log(1/\delta)$ steps.
\end{theorem}
\begin{remark}\label{rem:costfunction}
In our proofs of \pref{thm:gradexactFeps}, \ref{thm:gradestimateFeps}, and \ref{thm:discretizationloosenoracle}, our results hold under the more general condition
\[ \tri*{\grad F(\vecW), \grad\Phi(\vecW)} -\frac1{\beta}\Delta F(\vecW) \ge A(\vecW),\]
for a general non-negative cost function $A(\vecW)$. Note $A(\vecW)$ need not be continuous (for example, $F_{\epsilon}(\vecW)$ is not continuous). 
This lets us use the Lemmas we develop here, in the proofs of \pref{thm:poincareoptlipschitz}, \ref{thm:poincareoptlipschitzestimate}, and \ref{thm:smoothdissipativesettingpoincareopt}.
\end{remark}
The proofs of this result is `optimization style'. We break it into parts. First we perform a one-step discretization bound in expectation by applying \pref{lem:thirdordersmoothfromregularity}, which gives \pref{lem:onesteprecursion}. Then we analyze a stochastic process naturally arising from this setup to show that $\Phi$ indeed is a potential function for the discrete-time algorithm, for appropriate choice of $\eta$ and $T$, which is detailed in \pref{lem:actuallypotentialsimplesetting}. After this, we can conclude upon using the resulting bound and telescoping. 

First, we need to show that with self-bounding regularity, by composing with the appropriate function, we can obtain some analogue of third-order smoothness in order to perform optimization-style discretization. We detail this as follows via the following Lemmas, which are also used later to prove \pref{thm:poincareoptformal}.
\begin{lemma}\label{lem:thirdordersmoothfromregularity}
Let $\Phi$ be any non-negative function that satisfies polynomial self-bounding regularity to first, second, and third orders\footnote{This implicitly assumes $\Phi$ is differentiable through third order.}, that is we have $\nrm*{\nabla^i \Phi(\vecW)}_{\OPNORM} \le \rho_{\Phi,i}\prn*{\Phi(\vecW)}$ for $1 \le i \le 3$, where $\rho_{\Phi,i}(z)= \sum_{j=1}^{n_i} c_{i,j} z^{d_{i,j}}$ for all $z \ge 0$ (where all the $d_{i,j} \ge 0$). Then there exists some $\theta:\mathbb{R}_{\ge 0} \rightarrow\mathbb{R}_{\ge 0}$ such that $\theta'(z) > 0$, $\theta''(z)< 0$, $\theta'''(z) \ge 0$ for all $z \ge 0$, and 
\[ \theta\prn*{\Phi(\vecW+\vecU)} \le \theta\prn*{\Phi(\vecW)}+\theta'\prn*{\Phi(\vecW)} \tri*{\nabla \Phi(\vecW), \vecU}+\frac12\theta'\prn*{\Phi(\vecW)}\tri*{ \nabla^2 \Phi(\vecW) \vecU, \vecU} + \frac{C}6 \nrm*{\vecU}^3,\]
for some constant $C$ that depends only on the form of the functions $\rho_1, \rho_2$, and $\rho_3$.

Moreover, we also have
\[ \theta\prn*{\Phi(\vecW+\vecU)} \le \theta\prn*{\Phi(\vecW)}+\theta'\prn*{\Phi(\vecW)} \tri*{ \nabla \Phi(\vecW), \vecU}+\frac12 \nrm*{\vecU}^2,\]
and
\[\nrm*{\nabla \Phi(\vecW)} \le \rho_{\Phi}\prn*{\Phi(\vecW)}\sqrt{2\theta\prn*{\Phi(\vecW)}}.\]
\end{lemma}
\begin{proof}
Note we can assume without loss of generality that all the $c_{i,j} \ge 0$, and thus again we can assume without loss of generality that for all $z \ge 0$ we have
\[ \max\prn*{\rho_{\Phi,1}(z), \rho_{\Phi,1}(z)^3,\rho_{\Phi,2}(z),\rho_{\Phi,3}(z), \rho_{\Phi,1}(z)\rho_{\Phi,2}(z)} \le A+Az^p \le 2A(z+1)^p\]
for some $A \ge 0, p \ge 0$. The last step follows from \pref{lem:upperboundpower}.

Next, define $\rho_{\Phi}(z) := 2A(z+1)^p$, which is clearly non-negative and increasing. Thus for all $z \ge 0$ we have
\[ \rho_{\Phi}(z) \ge \max\prn*{\rho_{\Phi,1}(z), \rho_{\Phi,1}(z)^3,\rho_{\Phi,2}(z),\rho_{\Phi,3}(z), \rho_{\Phi,1}(z)\rho_{\Phi,2}(z)}.\]
Now let $\theta(z)$ be defined by $\theta'(z)=\frac1{\rho_{\Phi}(z)}$ and $\theta(0)=0$. The potential $\Phi$ we consider is non-negative and so we only consider $z \ge 0$; thus, $\theta$ is differentiable to all orders. Clearly $\theta'(z) > 0$. We can also check that $\theta''(z)=-\frac{p}{2A} (z+1)^{-p-1} < 0$, thus
\[ \abs*{\theta''(z)}\rho_{\Phi}(z) = \frac{p}{2A}(z+1)^{-p-1} \cdot 
2A(z+1)^p \le p(z+1)^{-1} \le p.\]
for all $z \ge 0$. Finally, we can compute $\theta'''(z) = \frac{p(p+1)}{2A} (z+1)^{-p-2}$, thus
\[ \abs*{\theta'''(z)} \rho_{\Phi}(z) = \theta'''(z)\rho_{\Phi}(z) = \frac{p(p+1)}{2A}(z+1)^{-p-2} \cdot 2A(z+1)^p=\frac{p(p+1)}{(z+1)^2} \le p(p+1)\]
for all $z \ge 0$.

Now define for all $0 \le \alpha \le 1$,
\[l(\alpha) := \theta\prn*{\Phi(\vecW+\alpha \vecU)}.\]
Recall $\Phi$ is non-negative, so all the inputs here to $\theta$ are non-negative. $l(\alpha)$ is differentiable to third order, since $\Phi$ is and $\theta$ is for non-negative inputs.

By standard calculation using the Chain Rule (this is also done in the proof of Lemma 11 of \cite{priorpaper}),
\begin{align*}l'(\alpha) = \theta'\prn*{\Phi(\vecW + \alpha \vecU)} \tri*{ \nabla \Phi(\vecW + \alpha \vecU), \vecU}.\end{align*}
We also have, from similar calculation (also done in the proof of Lemma 11 of \cite{priorpaper}) and using that $\theta''(z) \le 0$ for all $z \ge 0$ which was established earlier,
\begin{align*}l''(\alpha) &= \theta''\prn*{\Phi(\vecW + \alpha \vecU)} \tri*{\grad \Phi(\vecW + \alpha \vecU), \vecU}^2+\theta'\prn*{\Phi(\vecW + \alpha \vecU)} \tri*{ \grad^2 \Phi(\vecW + \alpha \vecU) \vecU, \vecU} \\
&\le \theta'\prn*{\Phi(\vecW + \alpha \vecU)}\tri*{ \nabla^2 \Phi(\vecW + \alpha \vecU) \vecU, \vecU}.\end{align*}
Similar calculation, noting $\theta'(z) \ge 0$ and the bounds we established earlier on $\abs*{\theta''(z)}\rho_{\Phi}(z)$ and $\abs*{\theta'''(z)}\rho_{\Phi}(z)$, gives
\begin{align*}l'''(\alpha) &= \theta'''\prn*{\Phi(\vecW + \alpha \vecU)} \tri*{ \grad \Phi(\vecW + \alpha \vecU), \vecU}\cdot \tri*{ \grad \Phi(\vecW + \alpha \vecU), \vecU}^2 \\
&\hspace{1in}+ \theta''\prn*{\Phi(\vecW + \alpha \vecU)} \cdot 2\tri*{ \nabla \Phi(\vecW + \alpha \vecU), \vecU}\tri*{ \nabla^2 \Phi(\vecW + \alpha \vecU) \vecU, \vecU} \\
&\hspace{1in} + \theta''\prn*{\Phi(\vecW + \alpha \vecU)}\tri*{ \nabla \Phi(\vecW + \alpha \vecU), \vecU} \cdot \tri*{ \nabla^2 \Phi(\vecW + \alpha \vecU) \vecU, \vecU} \\
&\hspace{1in} + \theta'\prn*{\Phi(\vecW + \alpha \vecU)}\nabla^3 \Phi(\vecW + \alpha \vecU) [\vecU,\vecU,\vecU]\\
&=\theta'''\prn*{\Phi(\vecW + \alpha \vecU)} \tri*{ \grad \Phi(\vecW + \alpha \vecU), \vecU}^3\\
&\hspace{1in}+3\theta''\prn*{\Phi(\vecW + \alpha \vecU)}\tri*{ \nabla^2 \Phi(\vecW + \alpha \vecU) \vecU, \vecU} \tri*{ \nabla \Phi(\vecW + \alpha \vecU), \vecU}\\
& \hspace{1in} + \theta'\prn*{\Phi(\vecW + \alpha \vecU)}\nabla^3 \Phi(\vecW + \alpha \vecU) [\vecU,\vecU,\vecU] \\
&\le \abs*{\theta'''(\Phi(\vecW + \alpha \vecU))} \rho_{\Phi,1}(\Phi(\vecW + \alpha \vecU))^3 \nrm*{\vecU}^3 \\
&\hspace{1in}+ 3\abs*{\theta''(\Phi(\vecW + \alpha \vecU))} \rho_{\Phi,1}\prn*{\Phi(\vecW + \alpha \vecU)} \rho_{\Phi,2}\prn*{\Phi(\vecW + \alpha \vecU)} \nrm*{\vecU}^3\\
&\hspace{1in}+\theta'(\Phi(\vecW + \alpha \vecU)) \rho_{\Phi,3}(\Phi(\vecW + \alpha \vecU)) \nrm*{\vecU}^3 \\
&\le \rho_{\Phi}\prn*{\Phi(\vecW + \alpha \vecU)}\prn*{\abs*{\theta'''(\Phi(\vecW + \alpha \vecU))}+3\abs*{\theta''\prn*{\Phi(\vecW + \alpha \vecU)}}+\theta'\prn*{\Phi(\vecW + \alpha \vecU)}}\nrm*{\vecU}^3\\
&\le (p^2+p+3p+1) \nrm*{\vecU}^3.\end{align*}
From here, we consider Taylor expansion of $l(1)$ around $0$. By Taylor's formula for the remainder, we know for some $\alpha \in [0,1]$ that
\begin{align*}
l(1) &=l(0)+l'(0)+\frac12 l''(0)+\frac16 l'''(\alpha).
\end{align*}
Plugging in the above inequalities, we get
\[ \theta\prn*{\Phi(\vecW+\vecU)} \le \theta\prn*{\Phi(\vecW)}+\theta'\prn*{\Phi(\vecW)} \tri*{ \nabla \Phi(\vecW), \vecU}+\frac12\theta'\prn*{\Phi(\vecW)}\tri*{ \nabla^2 \Phi(\vecW) \vecU, \vecU} + \frac{p^2+4p+1}6 \nrm*{\vecU}^3.\]
The result follows since $C=p^2+4p+1$ only depends on the form of the functions $\rho_{\Phi,1}, \rho_{\Phi,2}$, and $\rho_{\Phi,3}$.

The second part follows from noticing that $\rho_{\Phi}$ as defined here is an upper bound on $\rho_{\Phi,1}$ and $\rho_{\Phi,2}$, so the same derivation as in the proof of Lemma 11 of \citet{priorpaper} suffices.

Finally, if $\max_j \prn*{d_{i,j}} \le 1$ for all $1 \le i \le 3$ (i.e. the max degree of the self-bounding regularity functions is at most 1), we can be a bit tighter in how we define $\theta$. Instead we can just say
\[ \max\prn*{\rho_{\Phi,1}(z), \rho_{\Phi,1}(z),\rho_{\Phi,1}(z)} \le A+Az^p \le 2A(z+1)^p\]
where $0\le p\le 1$, and we define $\rho_{\Phi}(z)=A(z+1)^p$. Defining $\theta$ by $\theta'(z)=\frac1{\rho_{\Phi}(z)}$, $\theta(0)=0$ analogously as before, note we have for any $z \ge 0$ that
\[ \theta'(z) > 0, \theta''(z) < 0, \theta'''(z)>0,\]
\[ \abs*{\theta'''(z)}\rho_{\Phi,1}(z)^3 = \frac{p(p+1)}{A} (z+1)^{-p-2} \cdot 8A^3 (z+1)^{3p} =8A^2 p(p+1) (z+1)^{2p-2} \le 8A^2 p(p+1),\]
\[ \abs*{\theta''(z)} \rho_{\Phi,1}(z) \rho_{\Phi,2}(z) = \frac{p}{A} (z+1)^{-p-1} \cdot 4A^2 (z+1)^{2p}= 4Ap (z+1)^{p-1} \le 4Ap,\]
\[ \abs*{\theta'(z)} \rho_{\Phi,3}(z) = \frac{1}{A(z+1)^p} \cdot 2A(z+1)^p = 2.\]
The above three lines all use $p \le 1$ in the last inequality of those lines. Therefore, an analogous derivation as above gives
\begin{align*}
\theta\prn*{\Phi(\vecW+\vecU)} &\le \theta\prn*{\Phi(\vecW)}+\theta'\prn*{\Phi(\vecW)} \tri*{ \nabla \Phi(\vecW), \vecU}+\frac12\theta'\prn*{\Phi(\vecW)}\tri*{ \nabla^2 \Phi(\vecW) \vecU, \vecU} \\
&\hspace{1in}+ \frac{4A^2 p(p+1)+2Ap+1}3 \nrm*{\vecU}^3.
\end{align*}
\end{proof}

\begin{lemma}\label{lem:onesteprecursion}
For one iteration of \GLDTEXTSPACE starting at arbitrary $\vecW_t$,
\begin{align*}
\mathbb{E}_{\vecEps_t} \brk*{\theta\prn*{\Phi(\vecW_{t+1})}} &\le \theta\prn*{\Phi(\vecW_t)} - \eta  \theta'\prn*{\Phi(\vecW_t)} F_{\epsilon}(\vecW_t)\\
&\hspace{1in}+ \frac12 \eta^2 \nrm*{\nabla F(\vecW_t)}^2+\frac{2C}3 \eta^3\nrm*{\nabla F(\vecW_t)}^3 + 2C(\eta d/\beta)^{3/2},
\end{align*}
where $p$ and $C$ are defined from \pref{lem:thirdordersmoothfromregularity}.
\end{lemma}
\begin{proof}
First, apply \pref{lem:thirdordersmoothfromregularity} with $\vecW=\vecW_t$ and $\vecU = -\eta \nabla F(\vecW_t)+\sqrt{2\eta/\beta} \vecEps_t$ to obtain
\begin{align*}
\theta\prn*{\Phi(\vecW_{t+1})}
&= \theta\prn*{\Phi(\vecW_t-\eta \nabla F(\vecW_t)+\sqrt{2\eta/\beta} \vecEps_t)} \\
&\le\theta\prn*{\Phi(\vecW_t)}+ \theta'\prn*{\Phi(\vecW_t)}\tri*{ \nabla \Phi(\vecW_t),-\eta \nabla F(\vecW_t)+\sqrt{2\eta/\beta} \vecEps_t} \\
&+\frac12 \theta'\prn*{\Phi(\vecW_t)} \tri*{ \nabla^2 \Phi(\vecW_t) \prn*{-\eta \nabla F(\vecW_t)+\sqrt{2\eta/\beta} \vecEps_t}, -\eta \nabla F(\vecW_t)+\sqrt{2\eta/\beta} \vecEps_t } \\
&\hspace{1in}+\frac{C}6 \nrm*{-\eta \nabla F(\vecW_t)+\sqrt{2\eta/\beta} \vecEps_t}^3
\end{align*}
where $C$ is defined in the proof of \pref{lem:thirdordersmoothfromregularity}.

We take expectations of this inequality with respect to $\vecEps_t$. Let's consider what each term of the upper bound becomes when we take expectations.
\begin{itemize}
\item First order term: Since $\vecEps_t$ has mean as the 0 vector,
\[ \mathbb{E}_{\vecEps_t}\brk*{\theta'\prn*{\Phi(\vecW_t)}\tri*{ \nabla \Phi(\vecW_t),-\eta \nabla F(\vecW_t)+\sqrt{2\eta/\beta} \vecEps_t}} =  -\eta \theta'\prn*{\Phi(\vecW_t)}\tri*{\grad \Phi(\vecW_t), \grad F(\vecW_t)}.\]
\item Second order term: Note
\begin{align*}
&\mathbb{E}_{\vecEps_t} \brk*{\theta'\prn*{\Phi(\vecW_t)} \tri*{ \nabla^2 \Phi(\vecW_t) (-\eta \nabla F(\vecW_t)+\sqrt{2\eta/\beta} \vecEps_t), -\eta \nabla F(\vecW_t)+\sqrt{2\eta/\beta} \vecEps_t }}\\
&= \eta^2\theta'\prn*{\Phi(\vecW_t)} \tri*{\nabla^2 \Phi(\vecW_t) \nabla F(\vecW_t), \nabla F(\vecW_t)} \\
&\hspace{1in}- 2\eta \prn*{2\eta/\beta}^{1/2} \theta'\prn*{\Phi(\vecW_t)} \tri*{ \nabla^2 \Phi(\vecW_t) \nabla F(\vecW_t),\mathbb{E}_{\vecEps_t}\brk*{\vecEps_t}} \\
&\hspace{1in} + \prn*{2\eta/\beta} \theta'\prn*{\Phi(w_t)} \mathbb{E}_{\vecEps_t}\brk*{\tri*{\nabla^2\Phi(\vecW_t)\vecEps_t,\vecEps_t}}.
\end{align*}
In the above, the cross terms cancel because $\vecEps_t$ has mean of the 0 vector.

Now, consider $\mathbb{E}_{\vecEps_t}\brk*{\tri*{\nabla^2\Phi(\vecW_t)\vecEps_t,\vecEps_t}}$. We perform similar analysis as in the derivation and application of Ito's Lemma. This is where we see how the Laplacian term here actually helps. To make the parallels and motivation to Stochastic Calculus clear, here $\eta$ corresponds to $\DERIV t$, and $\sqrt{\eta} (\vecEps_t)_i$ corresponds to $(\DERIV \vecB_t)_i$. Note 
\begin{align*}
\mathbb{E}_{\vecEps_t}\brk*{\tri*{\nabla^2\Phi(\vecW_t)\vecEps_t,\vecEps_t}} &= \sum_{1 \le i, j \le d} \mathbb{E}_{\vecEps_t}\brk*{(\vecEps_t)_i (\vecEps_t)_j \prn*{\nabla^2 \Phi(\vecW_t)}_{ij}} \\
&= \sum_{1 \le i, j \le d} \nabla^2 \Phi(\vecW_t)_{ij}\mathbb{E}_{\vecEps_t}\brk*{(\vecEps_t)_i (\vecEps_t)_j}.
\end{align*}
We break into cases:
\begin{enumerate}
    \item When $i \neq j$: Note by symmetry of the unit sphere that
    \[ \mathbb{E}_{\vecEps_t}\brk*{(\vecEps_t)_i (\vecEps_t)_j}= 0.\]
    In particular this follows because for any $\mathbf{x}\in\cS^{d-1}$, $(\vecEps_t)_j$ has equal probability of being $\mathbf{x}$ or $-\mathbf{x}$. 
    \item When $i=j$: This is where we pick up the Laplacian. Note by symmetry,
    \[\mathbb{E}_{\vecEps_t}\brk*{(\vecEps_t)^2_i}=\mathbb{E}_{\vecEps_t}\brk*{(\vecEps_t)^2_j} \FORALLTEXT i, j, \text{ and }d = \mathbb{E}_{\vecEps_t}[\sum_{i=1}^d (\vecEps_t)^2_i] =\sum_{i=1}^d \mathbb{E}_{\vecEps_t}\brk*{(\vecEps_t)^2_i}.\]
    Therefore, 
    \[ \mathbb{E}_{\vecEps_t}\brk*{(\vecEps_t)^2_i}=1\FORALLTEXT 1 \le i \le d.\]
\end{enumerate}
Hence, we obtain the Laplacian $\Delta \Phi(\vecW_t)$: we have plugging this into the above that
\[ \mathbb{E}_{\vecEps_t}\brk*{\tri*{\nabla^2\Phi(\vecW_t)\vecEps_t,\vecEps_t}} = \sum_{i=1}^d (\nabla^2 \Phi(\vecW_t))_{ii} = \Delta \Phi(\vecW_t).\]
Hence, 
\begin{align*}
&\mathbb{E}_{\vecEps_t} \brk*{\theta'\prn*{\Phi(\vecW_t)} \tri*{ \nabla^2 \Phi(\vecW_t) (-\eta \nabla F(\vecW_t)+\sqrt{2\eta/\beta} \vecEps_t), -\eta \nabla F(\vecW_t)+\sqrt{2\eta/\beta} \vecEps_t }}\\
&=\eta^2\theta'\prn*{\Phi(\vecW_t)} \tri*{\nabla^2 \Phi(\vecW_t) \nabla F(\vecW_t), \nabla F(\vecW_t)} + (2\eta/\beta) \theta'\prn*{\Phi(\vecW_t)} \Delta\Phi(\vecW_t).
\end{align*}
\item Third order term: By AM-GM, we can prove for all $\mathbf{a},\mathbf{b}\in\mathbb{R}^d$,
\[ \nrm*{\mathbf{a}+\mathbf{b}}^3 \le \prn*{\nrm*{\mathbf{a}}+\nrm{\mathbf{b}}}^3 \le 4\nrm*{\mathbf{a}}^3+4\nrm{\mathbf{b}}^3.\]
Thus using this inequality pointwise we obtain
\begin{align*}
\mathbb{E}_{\vecEps_t}\brk*{\nrm*{-\eta \nabla F(\vecW_t)+\sqrt{2\eta/\beta} \vecEps_t}^3} &\le 4\eta^3 \nrm*{\nabla F(\vecW_t)}^3+4(2\eta/\beta)^{3/2} d^{3/2}.
\end{align*}
The last step is because deterministically $\nrm*{\vecEps_t}\le\sqrt{d}$ always.
\end{itemize}
Using the geometric property \pref{eq:admissablepotentialF} and $\theta'\prn*{\Phi(\vecW_t)} \ge 0$ from \pref{lem:thirdordersmoothfromregularity},
\[ -\eta \theta'\prn*{\Phi(\vecW_t)}\tri*{\grad \Phi(\vecW_t), \grad F(\vecW_t)} \le -\eta \theta'\prn*{\Phi(\vecW_t)}\prn*{F_{\epsilon}(\vecW_t) + \frac1{\beta} \Delta \Phi(\vecW_t)}. \]
Putting these together, this gives
\begin{align*}
&\mathbb{E}_{\vecEps_t} \brk*{\theta\prn*{\Phi(\vecW_{t+1})}} \\
&\le \theta\prn*{\Phi(\vecW_t)} -\eta \theta'\prn*{\Phi(\vecW_t)}\tri*{ \nabla \Phi(\vecW_t),\nabla F(\vecW_t)}\\
&\hspace{1in}+\frac12\prn*{\eta^2\theta'\prn*{\Phi(\vecW_t)} \tri*{\nabla^2 \Phi(\vecW_t) \nabla F(\vecW_t), \nabla F(\vecW_t)} + (2\eta/\beta) \theta'\prn*{\Phi(\vecW_t)} \Delta\Phi(\vecW_t)} \\
&\hspace{1in}+\frac{C}6 (4\eta^3 \nrm*{\nabla F(\vecW_t)}^3+4(2\eta/\beta)^{3/2} d^{3/2})  \\
&\le \theta\prn*{\Phi(\vecW_t)} -\eta \theta'\prn*{\Phi(\vecW_t)}\prn*{F_{\epsilon}(\vecW_t) + \frac1{\beta} \Delta \Phi(\vecW_t)}\\
&\hspace{1in}+\frac12\prn*{\eta^2\theta'\prn*{\Phi(\vecW_t)} \tri*{\nabla^2 \Phi(\vecW_t) \nabla F(\vecW_t), \nabla F(\vecW_t)} + (2\eta/\beta) \theta'\prn*{\Phi(\vecW_t)} \Delta\Phi(\vecW_t)} \\ 
&\hspace{1in}+\frac{C}6 \prn*{4\eta^3 \nrm*{\nabla F(\vecW_t)}^3+4(2\eta/\beta)^{3/2} d^{3/2}}.
\end{align*}
Note the terms $\eta \theta'\prn*{\Phi(\vecW_t)} \cdot \frac1{\beta} \Delta \Phi(\vecW_t)$ and $\frac12 (2\eta/\beta) \theta'\prn*{\Phi(\vecW_t)}\Delta \Phi(\vecW_t)$ cancel out. Moreover, note by definition of operator norm and since we set $\theta'(z) = \frac1{\rho_{\Phi}(z)} \le \frac1{\rho_{\Phi,2}(z)}$, we obtain
\begin{align*}
\frac12 \eta^2\theta'\prn*{\Phi(\vecW_t)} \tri*{\nabla^2 \Phi(\vecW_t)\nabla F(\vecW_t), \nabla F(\vecW_t)} &\le \frac12 \eta^2\theta'\prn*{\Phi(\vecW_t)} \nrm*{\nabla F(\vecW_t)}^2 \rho_2\prn*{\Phi(\vecW_t)}\\ 
&\le \frac12 \eta^2\nrm*{\nabla F(\vecW_t)}^2.
\end{align*}
Thus our above bound becomes
\begin{align*}
\mathbb{E}_{\vecEps_t} \brk*{\theta\prn*{\Phi(\vecW_{t+1})}} &\le \theta\prn*{\Phi(\vecW_t)} - \eta  \theta'\prn*{\Phi(\vecW_t)} F_{\epsilon}(\vecW_t)\\
&\hspace{1in}+ \frac12 \eta^2 \nrm*{\nabla F(\vecW_t)}^2+\frac{2C}3 \eta^3\nrm*{\nabla F(\vecW_t)}^3 + 2C(\eta d/\beta)^{3/2}.
\end{align*}
This is the desired result.
\end{proof}

The above result gives us a way to upper bound $\theta\prn*{\Phi(\vecW_t)}$. To control this we will need to control the $\Phi(\vecW_t)$ which we do as follows.
\begin{lemma}\label{lem:actuallypotentialsimplesetting}
Suppose $\beta \ge d$ and $\epsilon \le 1/e$. Additionally suppose $\epsilon < C'$ where $z=\frac1{C'}$ is the largest solution to $\frac{\log(\sqrt{10}z(\log z))}{\log^2(z)}=1/2$ (the existence of finitely many such $C'>0$ is obvious). 

With probability at least $0.9$, we have that
\[ \rho_{\Phi}\prn*{\Phi(\vecW_t)} \le \kappa' \rho_{\Phi}\prn*{\Phi(\vecW_0)}\]
for all $0 \le t \le T-1$ if $T, \eta$ are chosen as follows, where $\kappa'$ comes from \pref{ass:iteratesinballassumption}. First define 
\[ A_0(\vecW_0) = \theta\prn*{\rho_{\Phi}^{-1}\prn*{\kappa' \rho_{\Phi}\prn*{\Phi(\vecW_0)}}}-\theta\prn*{\Phi(\vecW_0)} >0, A_1(\vecW_0) = 12\sqrt{2} CB_{\textsc{grad}}^3.\]
Here $C$ comes from \pref{lem:onesteprecursion} and $B_{\textsc{grad}}:=L\prn*{R_1+\nrm*{\vecW^{\star}}}^s$, $L$ comes from \pref{ass:holderF} and $R_1$ comes from \pref{ass:iteratesinballassumption}. (Again if necessary take $C\leftarrow\max(C,1)$ and $B_{\textsc{grad}}\leftarrow\max(B_{\textsc{grad}},1)$.)

Also define
\[ r(\vecW_0) = \min\prn*{1, \frac3{4C}, \frac1{B_{\textsc{grad}}}}, C(\vecW_0) = \min\prn*{1,\frac{A_0(\vecW_0)^2 r(\vecW_0)}{128A_1(\vecW_0)^2}}.\] 

Now we choose $T, \eta$ based on cases:
\begin{enumerate}
\item If $\beta \le \frac{d}{\epsilon/\prn*{\log 1/\epsilon}^2}$: Take $T, \eta$ such that
\[ \eta = r(\vecW_0) \frac{d}{\beta}\]
and where $T$ is the floor of the unique solution to the equation 
\[ z \log (10z) = C(\vecW_0) \frac{\beta^2}{d^2}.\]
Existence and uniqueness to this equation is clear since $z \log(10z)$ is surjective on $\mathbb{R}_{\ge 0}$ and for every positive real $t$, exactly one positive real $z$ is such that $z \log (10z)=t$. Now note if $C(\vecW_0) \cdot \frac{\beta^2}{d^2} \ge 1$ then this means $T \le C(\vecW_0)\frac{\beta^2}{d^2} \le \frac{\beta^2}{d^2}$, and otherwise we have $T < 1 \implies T=0$. However, recall $\beta \ge d$ so in all cases we have $T \le \frac{\beta^2}{d^2}$.

Also note this means 
\[ \eta \le \min\prn*{1, \frac3{4C}, \frac1{B_{\textsc{grad}}},  \frac{d}{\beta}}.\]
\item If $\beta \ge \frac{d}{\epsilon/\prn*{\log 1/\epsilon}^2}$: Take $T, \eta$ such that 
\[ \eta = r(\vecW_0)\frac{\epsilon}{\prn*{\log 1/\epsilon}^2}, T = \floor{ \frac{C(\vecW_0)}{\epsilon^2}\prn*{\log 1/\epsilon}^2 }.\]
Note this implies $T \le C(\vecW_0) \cdot \frac{1}{\epsilon^2}\prn*{\log 1/\epsilon}^2 \le \frac{1}{\epsilon^2}\prn*{\log 1/\epsilon}^2$.
\end{enumerate}
Note as $\epsilon \le 1/e$, if $\beta \ge d\sqrt{\log (10)/C(\vecW_0)}$ and $\epsilon \le \sqrt{C(\vecW_0)}$, then $T \ge 1$. 

Also note in all cases that $\eta \le \min(1, r(\vecW_0))$, since $\epsilon \le 1$, $\beta \ge d$.
\end{lemma}
\begin{proof}
Define $\mathfrak{F}_t$ by the natural filtration with respect to $\vecW_j, \vecEps_j$ for all $j \le t$. Let 
\[ \tau := \min\{T, \inf\{t:\rho_{\Phi}\prn*{\Phi(\vecW_t)} > \kappa' \rho_{\Phi}\prn*{\Phi(\vecW_0)}\}\] 
where $\kappa'$ comes from \pref{ass:iteratesinballassumption}. 

Define a stochastic process $Y_t$ by 
\[ Y_t := \begin{cases}\theta\prn*{\Phi(\vecW_t)}+\sum_{j=0}^{t-1} \prn*{\eta F_{\epsilon}(\vecW_j) - R(\vecW_j, \Phi, \eta, \beta, d)}&\IF t \le \tau \\ Y_{\tau}&\OTHERWISEIF t > \tau \end{cases}\]
where 
\[ R(\vecW_j, \Phi, \eta, \beta, d) := \frac12 \eta^2 \nrm*{\nabla F(\vecW_t)}^2+\frac{2C}3 \eta^3\nrm*{\nabla F(\vecW_t)}^3 + 2C(\eta d/\beta)^{3/2}.\]
We show properties of $Y_t$ in \pref{lem:Ytsupermartingalesimplesetting} and \pref{lem:azumahoeffding} for the sake of presentation. Together, these prove that with probability at $1-\delta$, we have
\[ Y_{t} - Y_0 \le \sqrt{\frac12 \prn*{\sum_{t=0}^{\tau-1} C(\eta, t, d, \beta)^2}\log(T/\delta)}\]
for all $1 \le t \le T$, where 
\[C(\eta, t, d, \beta) = 4\sqrt{\theta\prn*{\rho_{\Phi}^{-1}\prn*{\kappa' \rho_{\Phi}\prn*{\Phi(\vecW_0)}}}} \cdot \nrm*{-\eta \nabla F(\vecW_t)+\sqrt{\frac{2\eta}{\beta}} \vecEps_t}+ 4\nrm*{-\eta \nabla F(\vecW_t)+\sqrt{\frac{2\eta}{\beta}} \vecEps_t}^2.\]
Denote the event from \pref{lem:azumahoeffding} with $\delta=0.1$ by $E_1$, which occurs with probability at least $0.9$. We claim that if $E_1$ occurs, then for all $0 \le t \le T-1$ we have $\rho_{\Phi}\prn*{\Phi(\vecW_t)} \le\kappa' \rho_{\Phi}\prn*{\Phi(\vecW_0)}$. This clearly finishes the proof.

Suppose for the sake of contradiction that conditioned on $E_1$, there exists $0 \le t \le T-1$ where $\rho_{\Phi}\prn*{\Phi(\vecW_t)} > \kappa' \rho_{\Phi}\prn*{\Phi(\vecW_0)}$. Hence we have $\tau < T$ and for that $\tau$, $\rho_{\Phi}\prn*{\Phi(\vecW_{\tau})} > \kappa' \rho_{\Phi}\prn*{\Phi(\vecW_0)}$. First note if $T=0$ this is not possible, so suppose $T \ge 1$ from now on. Thus $\log(10T) > 1$. 

Then by \pref{ass:iteratesinballassumption}, for all $t < \tau$ we have $\vecW_t \in \ball(\vecOrigin,R_1')$. Hence for all $t<\tau$ we have by \pref{ass:holderF} that
\[ \nrm*{\grad F(\vecW_t)}=\nrm*{\grad F(\vecW_t)-\grad F(\vecW^{\star})} \le L\nrm*{\vecW_t-\vecW^{\star}}^s \le L\prn*{R_1+\nrm*{\vecW^{\star}}}^s = B_{\textsc{grad}}.\]
We assumed without loss of generality that $L, R_1 \ge 1$, thus the above upper bound $B_{\textsc{grad}}$ is at least 1.

\pref{lem:azumahoeffding} gives us a way to upper bound $Y_{\tau}-Y_0$ (since we condition on $E_1$), so now let's derive a lower bound on $Y_{\tau}-Y_0$. We will then show that these upper and lower bounds are contradictory to complete the proof.

By definition of $\tau$ and as $\rho_{\Phi}, \theta$ are increasing, we have $\theta\prn*{\Phi(\vecW_{\tau})} > \theta\prn*{\rho_{\Phi}^{-1}\prn*{\kappa' \rho_{\Phi}\prn*{\Phi(\vecW_0)}}}$. Moreover, by the above and as $B_{\textsc{grad}} \ge 1$, we have for all $t<\tau$ that
\[ R(\vecW_j, \Phi, \eta, \beta, d) \le \frac12 \eta^2 B_{\textsc{grad}}^3 + \frac{2C}3 \eta^3 B_{\textsc{grad}}^3 + 2C(\eta d/\beta)^{3/2}.\]
Thus as $Y_0=\theta\prn*{\Phi(\vecW_0)}$, we have
\begin{align*}
Y_{\tau}-Y_0 &= \theta\prn*{\Phi(\vecW_{\tau})}+\sum_{j=0}^{\tau-1} \prn*{\eta F_{\epsilon}(\vecW_j) - R(\vecW_j, \Phi, \eta, \beta, d)} - \theta\prn*{\Phi(\vecW_0)} \\
&\ge \theta\prn*{\rho_{\Phi}^{-1}\prn*{\kappa' \rho_{\Phi}\prn*{\Phi(\vecW_0)}}}-\theta\prn*{\Phi(\vecW_0)}-\prn*{\prn*{\frac12 \eta^2 + \frac{2C}3 \eta^3} B_{\textsc{grad}}^3 + 2C(\eta d / \beta)^{3/2}} \tau\\
&\ge \theta\prn*{\rho_{\Phi}^{-1}\prn*{\kappa' \rho_{\Phi}\prn*{\Phi(\vecW_0)}}}-\theta\prn*{\Phi(\vecW_0)}-\prn*{\frac32 C B_{\textsc{grad}}^3 \eta^2 + 2C (\eta d / \beta)^{3/2}}T.\numberthis\label{eq:lowerboundgradoraclepotential}
\end{align*}
Here the first inequality crucially uses the definition of $\tau$ itself. The second inequality follows from $\eta \le \frac3{4C}$. Note as $\theta$ and $\rho_{\Phi}$ are increasing and as $\kappa'>1$, $A_0(\vecW_0)=\theta\prn*{\rho_{\Phi}^{-1}\prn*{\kappa' \rho_{\Phi}\prn*{\Phi(\vecW_0)}}}-\theta\prn*{\Phi(\vecW_0)}>0$.

Now recall that as we condition on $E_1$ from \pref{lem:azumahoeffding}, we have that
\[ Y_{\tau} - Y_0 \le \sqrt{\frac12 \prn*{\sum_{t=0}^{\tau-1} C(\eta, t, d, \beta)^2}\log(10T)}\]
since the above holds for every $1 \le t \le T$, and thus holds for every value $\tau$ could take.
Recall $\nrm*{\vecEps_t}=\sqrt{d}$ always holds. Therefore, we obtain via Triangle Inequality and Young's Inequality that for all $t<\tau$,
\begin{align*}
C(\eta, t, d, \beta) &\le 4\nrm*{-\eta \nabla F(\vecW_t)+\sqrt{\frac{2\eta}{\beta}} \vecEps_t}+ 4\nrm*{-\eta \nabla F(\vecW_t)+\sqrt{\frac{2\eta}{\beta}} \vecEps_t}^2\\
&\le 4 \prn*{B_{\textsc{grad}} \eta+2\sqrt{\eta d / \beta}}+ 8\eta^2 B_{\textsc{grad}}^2+16 \eta d / \beta.
\end{align*}
Here the first inequality follows by definition of $\theta$ in \pref{lem:thirdordersmoothfromregularity}, we can easily check $\theta \le 1$ assuming $\rho_{\Phi}$ is scaled and shifted appropriately by an absolute constant.

Observe that as $\eta \le 1$, we have $\eta d / \beta \le 1$ as $\beta \ge d$. Since in all cases we chose $\eta \le \min\prn*{1,\frac1{B_{\textsc{grad}}}}$, we have for all $t<\tau$ that
\begin{align*}
C(\eta, t, d, \beta) &\le 24 \prn*{B_{\textsc{grad}} \eta+\sqrt{\eta d / \beta} }.
\end{align*}
This implies 
\begin{align*}
Y_{\tau} - Y_0 &\le \sqrt{\frac12 \prn*{\sum_{t=0}^{\tau-1} C(\eta, t, d, \beta)^2}\log(10T)}\\
&\le 12\sqrt{2} \prn*{ B_{\textsc{grad}} \eta+\sqrt{\eta d / \beta} } \sqrt{T\log\prn*{10T}}.\numberthis\label{eq:upperboundgradoraclepotential}
\end{align*}
Putting together these lower and upper bounds \pref{eq:lowerboundgradoraclepotential} and \pref{eq:upperboundgradoraclepotential}, obtain
\begin{align*} 
&\theta\prn*{\rho_{\Phi}^{-1}\prn*{\kappa' \rho_{\Phi}\prn*{\Phi(\vecW_0)}}}-\theta\prn*{\Phi(\vecW_0)}-\prn*{\frac32 C B_{\textsc{grad}}^3 \eta^2 + 2C (\eta d / \beta)^{3/2}}T \\
&\le Y_{\tau}-Y_0 \\
&\le 12\sqrt{2} \prn*{ B_{\textsc{grad}} \eta+\sqrt{\eta d / \beta} } \sqrt{T\log\prn*{10T}}.
\end{align*}
That is, recalling the definition of $A_0(\vecW_0)$,
\begin{align*}
0<A_0(\vecW_0) &= \theta\prn*{\rho_{\Phi}^{-1}\prn*{\kappa' \rho_{\Phi}\prn*{\Phi(\vecW_0)}}}-\theta\prn*{\Phi(\vecW_0)} \\
&\le \prn*{\frac32 C B_{\textsc{grad}}^3 \eta^2 + 2C (\eta d / \beta)^{3/2}}T+12\sqrt{2} \prn*{ B_{\textsc{grad}} \eta+\sqrt{\eta d / \beta} } \sqrt{T\log\prn*{10T}}.\numberthis \label{eq:potentialoracleineq}
\end{align*}
Noting the left hand side is a positive constant, we aim to show with our choice of $\eta$ and $T$ that this gives contradiction. Break into our original cases:
\begin{enumerate}
    \item If $\beta \le \frac{d}{\epsilon/\prn*{\log 1/\epsilon}^2}$: Recall $r(\vecW_0) \le 1$. By our choice of $\eta=r(\vecW_0) \frac{d}{\beta} \le 1$, we have 
    \[ \prn*{\eta d / \beta}^{1/2} \le \frac{(\eta^2)^{1/2}}{r(\vecW_0)^{1/2}} = \frac{\eta}{r(\vecW_0)^{1/2}}, \prn*{\eta d / \beta}^{3/2} = \frac{(\eta^2)^{3/2}}{r(\vecW_0)^{3/2}} \le \frac{\eta^2}{r(\vecW_0)^{3/2}}.\]
    Now, using this we note the right hand side of \pref{eq:potentialoracleineq} is at most
    \begin{align*}
    &A_1(\vecW_0) \prn*{\prn*{\eta^2 + \frac{\eta^2}{r(\vecW_0)^{3/2}}}T + \prn*{\eta + \frac{\eta}{r(\vecW_0)^{1/2}}} \sqrt{T \log (10T)}} \\
    &\le \frac{A_1(\vecW_0)}{r(\vecW_0)^{3/2}} \prn*{2\eta^2 T + 2\eta\sqrt{T\log (10T)}} \\
    &\le \frac{4A_1(\vecW_0)}{r(\vecW_0)^{3/2}} \cdot \eta \sqrt{T\log (10T)} \\
    &\le \frac{4A_1(\vecW_0)}{r(\vecW_0)^{3/2}}\cdot r(\vecW_0) \cdot\frac{d}{\beta} \cdot\frac{A_0(\vecW_0) \cdot r(\vecW_0)^{1/2}}{8A_1(\vecW_0)}\cdot  \frac{\beta}d \\
    &< \frac{A_0(\vecW_0)}2.
    \end{align*}
    The first inequality is because $\eta \le 1$ and definition of $A_1(\vecW_0)$, and because $r(\vecW_0)\le 1$. The second inequality is because $\eta \sqrt{T} \le \frac{d}{\beta} \cdot \frac{\beta}d=1$ (recall $T\ge 1$ else we are done). The third inequality follows recalling the definitions of $\eta$ and $T$ in terms of $A_0(\vecW_0)$, $A_1(\vecW_0)$, $r(\vecW_0)$ and $C(\vecW_0)$ (note $z \log(10z)$ is increasing on $z \ge 1$). The last inequality follows from definition of $C(\vecW_0)$. 
    
    As $A_0(\vecW_0)>0$, this contradicts \pref{eq:potentialoracleineq} which is exactly what we want.
    \item If $\beta \ge \frac{d}{\epsilon/\prn*{\log 1/\epsilon}^2}$: The strategy is similar. This time, we have by the condition that
    \[\frac{d}{\beta} \le \frac{\epsilon}{\prn*{\log 1/\epsilon}^2} = \frac{\eta}{r(\vecW_0)}\]
    hence
    \[ \prn*{\frac{\eta d}{\beta}}^{1/2} \le\frac{\eta}{r(\vecW_0)^{1/2}}.\]
    Thus, we have the right hand side of \pref{eq:potentialoracleineq} is at most
    \begin{align*}
    &A_1(\vecW_0) \prn*{\prn*{\eta^2 + \frac{\eta^2}{r(\vecW_0)^{3/2}}}T + \prn*{\eta + \frac{\eta}{r(\vecW_0)^{1/2}}} \sqrt{T \log (10T)}} \\
    &\le \frac{A_1(\vecW_0)}{r(\vecW_0)^{3/2}} \prn*{2\eta^2 T + 2\eta\sqrt{T\log (10T)}} \\
    &\le \frac{4A_1(\vecW_0)}{r(\vecW_0)^{3/2}} \cdot \eta \sqrt{T\log (10T)} \\
    &\le \frac{4A_1(\vecW_0)}{r(\vecW_0)^{3/2}} \cdot  r(\vecW_0) \cdot \frac{\epsilon}{(\log 1/\epsilon)^2} \cdot \frac{\sqrt{C(\vecW_0)}}{\epsilon} \cdot \log (1/\epsilon) \cdot \sqrt{\log \prn*{\frac{10}{\epsilon^2}\prn*{\log (1/\epsilon)^2}}} \\
    &\le \frac{A_0(\vecW_0)}2 \sqrt{\frac{\log(\frac{\sqrt{10}}{\epsilon}(\log 1/\epsilon))}{\log^2(1/\epsilon)}} \\
    &\le \frac{A_0(\vecW_0)}{2\sqrt{2}}.
    \end{align*}
    The first inequality is because $\eta \le 1$, the definition of $A_1(\vecW_0)$, and because $r(\vecW_0)\le 1$. The second inequality is because $\epsilon \le 1/e$, $T \ge 1$ and so $\eta \sqrt{T} \le \frac{1}{\log\prn*{1/\epsilon}} \le 1$. The third inequality is by definition of $\eta$ and $T$ and as $C(\vecW_0) \le 1$, $T \ge 1$ (note $z \log (10z)$ is increasing on $z \ge 1$). The fourth inequality is by definition of $C(\vecW_0)$. The last inequality is by definition of $\epsilon$ and $C'$. In detail, since $\frac{\log(\sqrt{10}z(\log z))}{\log^2(z)}$ is continuous, decreasing for large enough $z$, and $\lim_{z\rightarrow\infty}\frac{\log(\sqrt{10}z(\log z))}{\log^2(z)}=0$, let $z:=\frac1{C'}$ be the largest solution to $\frac{\log(\sqrt{10}z(\log z))}{\log^2(z)}=1/2$. Thus, as $\epsilon<C'$ we have the last inequality. 
    
    This contradicts \pref{eq:potentialoracleineq} as $A_0(\vecW_0)>0$, which again is exactly what we want.
\end{enumerate}
In all cases we obtain a contradiction conditioned on $E_1$, which occurs with probability at least 0.9 from the earlier discussion. Hence with probability at least 0.9 we have $\rho_{\Phi}(\Phi(\vecW_t)) \le \kappa' \rho_{\Phi}(\Phi(\vecW_0))$ for all $0 \le t \le T-1$ as desired.
\end{proof}

Now, with these parts in hand, we can prove \pref{thm:discretizationsimpleoracleformal}.
\begin{proof}
First, it clearly suffices to prove for small enough $\epsilon$, in particular when $\epsilon \le \min\prn*{1/e, C', \sqrt{\frac{C(\vecW_0)}{\log10}}}$. By the logic in \pref{lem:actuallypotentialsimplesetting}, based on our cases on $\beta$ and $\epsilon$, the $T$ that we choose will always be at least 1. This is as $\epsilon \le 1/e$, so our cutoff of $\min\prn*{1/e, C', \sqrt{\frac{C(\vecW_0)}{\log10}}}$ for $\epsilon$ guarantees that the $\epsilon$ we use is upper bounded by $\sqrt{C(\vecW_0)}$. 
Therefore suppose $\epsilon \le \min\prn*{1/e, C', \sqrt{\frac{C(\vecW_0)}{\log10}}}$. In this case, $\epsilon \le 1$, so we can also assume $\eta \le 1$ by the choice of $\eta$ given in \pref{lem:actuallypotentialsimplesetting}. 

Let $E_1$ be the event that $\rho_{\Phi}(\Phi(\vecW_t)) \le \kappa' \rho_{\Phi}(\Phi(\vecW_0))$ for all $1 \le t \le T$. From \pref{lem:actuallypotentialsimplesetting}, we know $E_1$ holds with probability at least 0.9 for the choice of $\eta, T$ given there. By \pref{ass:iteratesinballassumption}, this means that conditioned on $E_1$, all the $\vecW_t \in \ball(\vecOrigin,R_1)$ for $1 \le t \le T$. By the same derivation as \pref{lem:actuallypotentialsimplesetting}, this means 
\[ \nrm*{\grad F(\vecW_t)} \le B_{\textsc{grad}}\FORALLTEXT1 \le t \le T.\numberthis\label{eq:gradupperboundoracle}\]

As $\theta' \ge 0$ by \pref{lem:thirdordersmoothfromregularity}, $\sum_{t=0}^{T-1} F_{\epsilon}(\vecW_t) \theta'(\Phi(\vecW_t)) \ge 0$, thus we see by Markov's Inequality that with probability at least 0.9,
\[ \sum_{t=0}^{T-1} F_{\epsilon}(\vecW_t) \theta'(\Phi(\vecW_t)) \le 10\mathbb{E}[ \sum_{t=0}^{T-1}  F_{\epsilon}(\vecW_t) \theta'(\Phi(\vecW_t))].\numberthis\label{eq:discretizationoraclemarkov}\]
Let $E_2$ be the event that this above inequality holds. 

Finally, consider $\sum_{t=0}^{T-1} \mathbb{E}\brk*{\nrm*{\nabla F(\vecW_t)}^r}$ for $r \in \{2,3\}$. Note by constructing a martingale from partial sums of the sequence minus the expectation of the partial sum, we can show that with probability at least 0.975 for a given $r \in \{2,3\}$,
\begin{align*}
\sum_{t=0}^{T-1} \nrm*{\nabla F(\vecW_t)}^r &\ge  \sum_{t=0}^{T-1} \mathbb{E}\brk*{\nrm*{\nabla F(\vecW_t)}^r}- 2\sqrt{2}\max_{0 \le t \le T-1}\prn*{\nrm*{\nabla F(\vecW_t)}^r} \cdot \sqrt{T\log 40}.\numberthis\label{eq:discretizationoraclemartingale}
\end{align*}
Let $E_3$ be the intersection of these two events for $r\in\{2,3\}$, so $E_3$ has probability at least 0.95.

The last step we need is the following: summing and telescoping the result from \pref{lem:onesteprecursion}, and using that $\theta\prn*{\Phi(z)}\ge 0$, we obtain
\begin{align*}
\eta \mathbb{E}\brk*{\sum_{t=0}^{T-1}  F_{\epsilon}(\vecW_t) \theta'\prn*{\Phi(\vecW_t)}} &\le \theta\prn*{\Phi(\vecW_0)} + 2C (\eta d / \beta)^{3/2} T \\
&+ \frac12 \eta^2 \sum_{t=0}^{T-1} \mathbb{E}\brk*{\nrm*{\nabla F(\vecW_t)}^2}+\frac{2C}3 \eta^3 \sum_{t=0}^{T-1} \mathbb{E}\brk*{\nrm*{\nabla F(\vecW_t)}^3}.\numberthis\label{eq:discretizationoracletelescope}
\end{align*}
Here, we took full expectations over the noise sequence $\vecEps_t$ in the above.

Let 
\begin{align*}
M_1(\vecW_0) &= 10\max\prn*{\theta\prn*{\Phi(\vecW_0)}, 6C B_{\textsc{grad}}^3},
\end{align*}
which is just a $\vecW_0$-dependent constant. 

Now we put the above steps together and do a Union Bound over $E_1, E_2, E_3$. Let $E=E_1 \cap E_2 \cap E_3$; we have that $E$ occurs with probability at least 0.75. Then conditioned on $E$, we see combining \pref{eq:gradupperboundoracle}, \pref{eq:discretizationoraclemarkov}, \pref{eq:discretizationoraclemartingale}, and \pref{eq:discretizationoracletelescope} that
\begin{align*}
\sum_{t=0}^{T-1}  F_{\epsilon}(\vecW_t) \theta'\prn*{\Phi(\vecW_t)}&\le 10\mathbb{E}\brk*{ \sum_{t=0}^{T-1}  F_{\epsilon}(\vecW_t) \theta'\prn*{\Phi( \vecW_t)}} \\
&\le \frac{\theta\prn*{\Phi(\vecW_0)}}{\eta} + 2C (d / \beta)^{3/2} \eta^{1/2} T + \frac12 \eta \sum_{t=0}^{T-1} \mathbb{E}\brk*{\nrm*{\nabla F(\vecW_t)}^2}+\frac{2C}3 \eta^2 \sum_{t=0}^{T-1} \mathbb{E}\brk*{\nrm*{\nabla F(\vecW_t)}^3}\\
&\le \frac{\theta\prn*{\Phi(\vecW_0)}}{\eta} + 2C (d / \beta)^{3/2} \eta^{1/2} T \\
&\hspace{1in}+\frac12 \eta \prn*{\sum_{t=0}^{T-1} \nrm*{\nabla F(\vecW_t)}^2 + 2\sqrt{2}\max_{0 \le t \le T-1}\prn*{\nrm*{\nabla F(\vecW_t)}^2} \cdot \sqrt{T\log 40}} \\
&\hspace{1in}+\frac{2C}3 \eta^2 \prn*{\sum_{t=0}^{T-1} \nrm*{\nabla F(\vecW_t)}^3 + 2\sqrt{2}\max_{0 \le t \le T-1}\prn*{\nrm*{\nabla F(\vecW_t)}^3} \cdot \sqrt{T\log 40}} \\
&\le M_1(\vecW_0) \prn*{\frac{1}{\eta} + (d / \beta)^{3/2} \eta^{1/2} T + \eta T + \eta^2 T + \sqrt{T}}.
\end{align*}
The last inequality uses \pref{eq:gradupperboundoracle}, $\eta \le 1$, and $B_{\textsc{grad}} \ge 1$.

Note by definition from \pref{lem:thirdordersmoothfromregularity} we know $\theta' > 0$ always holds. Moreover, because $\theta'' < 0$ from \pref{lem:thirdordersmoothfromregularity}, conditioned on $E$ we have 
\[ \theta'\prn*{\Phi(\vecW_t)} \ge \theta'\prn*{\rho_{\Phi}^{-1}\prn*{\kappa' \rho_{\Phi}\prn*{\Phi(\vecW_0)}}}\FORALLTEXT 0 \le t \le T-1. \]
Thus, defining 
\[ M(\vecW_0) = \frac{M_1(\vecW_0)}{\theta'\prn*{\rho_{\Phi}^{-1}\prn*{\kappa' \rho_{\Phi}\prn*{\Phi(\vecW_0)}}}} \in (0,\infty),\]
we see that conditioned on $E$ which occurs with probability at least 0.75 we have
\begin{align*}
\frac1T \sum_{t=0}^{T-1} F_{\epsilon}(\vecW_t) &\le M(\vecW_0) \prn*{\frac{1}{\eta T} + (d / \beta)^{3/2} \eta^{1/2} + \eta + \eta^2 + \frac1{\sqrt{T}}} \\
&\le M(\vecW_0)\prn*{\frac1{\eta T} + \frac{\eta}2 + \frac{(d/\beta)^{3}}2 + 2\eta + \frac{1}{2\eta T} + \frac{\eta}2} \\
&\le 3M(\vecW_0) \prn*{\frac1{\eta T} + \eta + d/\beta}.
\end{align*}
Here we used $\eta \le 1$, $d/\beta \le 1$, and AM-GM. We break into cases based on how we set $\eta, T$ from \pref{lem:actuallypotentialsimplesetting}:
\begin{enumerate}
    \item If $\beta \le \frac{d}{\epsilon/\prn*{\log 1/\epsilon}^2}$: Now recall we set $\eta = r(\vecW_0) \frac{d}{\beta}\le  d/\beta$, and that we let $T$ be the floor of the unique solution to the equation 
    \[ z \log (10z) = C(\vecW_0) \frac{\beta^2}{d^2},\]
    where $C(\vecW_0)$ is defined according to \pref{lem:actuallypotentialsimplesetting}. Recall we had $T \ge 1$ as well as 
    \[ T \le C(\vecW_0)\frac{\beta^2}{d^2} \le \frac{\beta^2}{d^2}.\]
    Since $T \ge 1$, and as $z\log(10z)$ is increasing for $z \ge 1$, it follows via definition of $T$ (note $2\floor{ z } \ge z$ for all $z \ge 1$) that 
    \[ 2T \log(20T) \ge C(\vecW_0) \frac{\beta^2}{d^2},\]
    hence
    \[ \eta T \ge \frac{r(\vecW_0)C(\vecW_0)}{2\log(20T)}\frac{\beta}d.\]
    Thus, we have with probability at least 0.75 that 
    \begin{align*}
    \frac1T \sum_{t=0}^{T-1} F_{\epsilon}(\vecW_t) &\le 3M(\vecW_0) \prn*{\frac{2\log(20 T)}{r(\vecW_0)C(\vecW_0)} + 2}\frac{d}{\beta} \\
    &\le 6M(\vecW_0) \prn*{\frac{\log20+2\log(\beta/d)}{r(\vecW_0)C(\vecW_0)} + 1} \frac{d}{\beta}.
    \end{align*}
    That is, we obtain $\widetilde{O}\prn*{\frac{d}{\beta}}$ suboptimality with at most $\frac{\beta^2}{d^2}$ iterations.
    \item If $\beta \ge \frac{d}{\epsilon/\prn*{\log 1/\epsilon}^2}$: Recalling how we set $\eta$ and the definition of this case, as well as $\epsilon \le 1/e$, we have $\eta, d/\beta \le \epsilon$. Moreover, note $\eta T \ge \frac{r(\vecW_0)C(\vecW_0)}{2\epsilon}$; $T \ge 1$, and $\floor{z} \ge \frac{z}2$ for all $z \ge 1$. Hence, we obtain with probability at least 0.75 that
    \[\frac1T \sum_{t=0}^{T-1} F_{\epsilon}(\vecW_t) \le 3M(\vecW_0) \prn*{\frac{2\epsilon}{r(\vecW_0)C(\vecW_0)}+2\epsilon} = 6M(\vecW_0) \prn*{\frac1{r(\vecW_0)C(\vecW_0)}+1}\epsilon. \]
    That is, we obtain $O(\epsilon)$ suboptimality with at most $T \le \frac1{\epsilon^2}\prn*{\log 1/\epsilon}^2$ iterations.
\end{enumerate}
\end{proof}

\subsection{Proof of \pref{thm:gradestimateFeps}}\label{subsec:gradientestimatediscretize}
In the stochastic gradient oracle setting, we have the following result for optimization of $F_{\epsilon}(\vecW)$. The proof is very similar to the above. Again note this generalizes to cost function $A(\vecW)$ the same way as noted in \pref{rem:costfunction}.

The following is our formal statement of \pref{thm:gradestimateFeps}.
\begin{theorem}\label{thm:discretizationsimpleestimateformal}
Consider $F$ and suppose $F$ is differentiable. Suppose that we have \pref{eq:admissablepotentialF}:
\[ \tri*{ \nabla \Phi(\vecW), \nabla F(\vecW)} \ge F_{\epsilon}(\vecW)+\frac1{\beta} \Delta \Phi(\vecW),\]
for some $\beta>0$. Suppose \pref{ass:holderF}, \ref{ass:polyselfbounding}, and \ref{ass:iteratesinballassumption} hold. Moreover, suppose we have an unbiased stochastic gradient oracle $\nabla f$ that satisfies \pref{ass:gradnoiseassumption}, and that $\beta \ge d\sqrt{\frac{\log^7(20)}{C(\vecW_0)}}$ where $C(\cdot)$ is defined below. 

Define the following quantities:
\[ A_0(\vecW_0) = \theta\prn*{\rho_{\Phi}^{-1}\prn*{\kappa' \rho_{\Phi}\prn*{\Phi(\vecW_0)}}}-\theta\prn*{\Phi(\vecW_0)} >0, A_1(\vecW_0) = 12\sqrt{2} CB_{\textsc{grad}}^3.\]
Here $C$ comes from \pref{lem:onesteprecursion} and $B_{\textsc{grad}}:=2\prn*{L\prn*{R_1+\nrm*{\vecW^{\star}}}^s+\sigma_F}$, $\sigma_F$ comes from \pref{ass:gradnoiseassumption}, $L$ comes from \pref{ass:holderF} and $R_1$ coming from \pref{ass:iteratesinballassumption}. (If necessary take $C\leftarrow\max(C,1)$, $B_{\textsc{grad}}\leftarrow\max(B_{\textsc{grad}},1)$, $\sigma_F \leftarrow \max(\sigma_F,1)$.) In terms of these define (where $\theta$ comes from \pref{lem:thirdordersmoothfromregularity}),
\[ r(\vecW_0) = \min\prn*{1, \frac3{4C}, \frac1{B_{\textsc{grad}}}}, C(\vecW_0) = \min\prn*{1,\frac{A_0(\vecW_0)^2 r(\vecW_0)}{128A_1(\vecW_0)^2}}.\] 
\[M(\vecW_0) = \frac{10\max\prn*{\theta\prn*{\Phi(\vecW_0)}, 6C B_{\textsc{grad}}^3}}{\theta'\prn*{\rho_{\Phi}^{-1}\prn*{\kappa' \rho_{\Phi}\prn*{\Phi(\vecW_0)}}}} \in (0,\infty).\]

Consider running \pref{alg:langevingradoracle} except:
\begin{itemize}
    \item Each instance of $\log10$ is replaced here with $\log^7(20)$.
    \item Instead of using the exact gradient $\nabla F(\vecW_t)$, we use the stochastic gradient oracle $\nabla f(\vecW_t;\vecZ_t)$.
    \item We use $\eta$ given in \pref{lem:actuallypotentialstochasticsetting} (rather than from \pref{lem:actuallypotentialsimplesetting}) and run this process with step size $\eta$ for $T$ steps also given in \pref{lem:actuallypotentialstochasticsetting}.
\end{itemize}

Then our algorithm has the following guarantees. First, its runtime $T$ is as follows:
\begin{enumerate}
    \item If $\epsilon \le \min\prn*{1/e, C', \sqrt{\frac{C(\vecW_0)}{\log^7(20)}}}$: then
    \[ T \le \begin{cases} \frac{\beta^2}{d^2}&\IF \beta \le \frac{d}{\epsilon/\prn*{\log 1/\epsilon}^5} \\
    \frac{1}{\epsilon^2}\prn*{\log 1/\epsilon}^2&\IF\beta \ge \frac{d}{\epsilon/\prn*{\log 1/\epsilon}^5}. \end{cases} \]
    \item If $\epsilon > \min\prn*{1/e, C', \sqrt{\frac{C(\vecW_0)}{\log^7(20)}}}$: then we have the same runtime guarantee as implied by above with $\min\prn*{1/e, C', \sqrt{\frac{C(\vecW_0)}{\log^7(20)}}}$ in place of $\epsilon$.
\end{enumerate}
In terms of error, we have with probability at least 0.75 (where probability is taken over $\{\vecEps_t,\vecZ_t\}_{0\le t \le T-1}$) that 
\begin{align*}
&\frac1T \sum_{t=0}^{T-1}F_{\epsilon}(\vecW_t)\le \begin{cases} 6M(\vecW_0) \prn*{\frac{2^7\log^7(40)+2^{14}\log^7(\beta/d)}{r(\vecW_0)C(\vecW_0)} + 1} \frac{d}{\beta}&\IF \beta \le \frac{d}{\epsilon/\prn*{\log 1/\epsilon}^5} \\
6M(\vecW_0) \prn*{\frac1{r(\vecW_0)C(\vecW_0)}+1}\epsilon \prn*{\log 1/\epsilon}^3&\IF\beta \ge \frac{d}{\epsilon/\prn*{\log 1/\epsilon}^5}. \end{cases}
\end{align*}
Note now that logarithmic boosting tricks proves that a given one of these guarantees can occur with probability at least $1-\delta$ using at most $T \log(1/\delta)$ steps.
\end{theorem}
We prove \pref{thm:discretizationsimpleestimateformal} with a very similar strategy to the proof of \pref{thm:discretizationsimpleoracleformal}.

The first step is to control the error of the gradient estimates to adapt to the stochastic gradient setting; we also use these results to prove \pref{thm:poincareoptformal}.
\begin{lemma}\label{lem:stochasticgradsetting}
Suppose \pref{ass:gradnoiseassumption} holds. Letting $\{\vecW_t\}_{0\le t\le T-1}$ be the sequence of iterates generated by any of the variants of \SGLDTEXTSPACE used in our algorithms on $F$, using stochastic gradient estimates based on $\{\vecZ_t\}_{0\le t\le T-1}$. Then we have 
\[ \mathbb{E}_{\vecZ_t}\brk*{\nrm*{\grad f(\vecW_t;\vecZ_t)-\grad F(\vecW_t)}^2} \le \sigma_F^2,\]
and moreover
\[ \mathbb{E}_{\vecZ_t}\brk*{\nrm*{\nabla f(\vecW_t;\vecZ_t)}^2} \le 2\sigma_F^2+2\nrm*{\nabla F(\vecW_t)}^2, \mathbb{E}_{\vecZ_t}\brk*{\nrm*{\grad f(\vecW_t;\vecZ_t)}^3} \le 8\sigma_F^3 + 4\nrm*{\grad F(\vecW_t)}^3. \]
Also with probability at least $1-\delta$ we have
\[ \nrm*{\grad f(\vecW_t;\vecZ_t)} \le \nrm*{\grad F(\vecW_t)} + \sigma_F \sqrt{\log\prn*{T/\delta}} \FORALLTEXT 0 \le t \le T-1.\]
Here, all probabilities and expectations are taken over the $\vecZ_t$.
\end{lemma}
\begin{proof}
Clearly $\nrm*{\grad f(\vecW_t;\vecZ_t)-\grad F(\vecW_t)}^2$ is non-negative, therefore 
\begin{align*}
\mathbb{E}\brk*{\nrm*{\grad f(\vecW_t;\vecZ_t)-\grad F(\vecW_t)}^2} &= \int_{t=0}^\infty \mathbb{P}\prn*{\nrm*{\grad f(\vecW_t;\vecZ_t)-\grad F(\vecW_t)}^2 \ge t} \DERIV t\\
&=\int_{t=0}^\infty \mathbb{P}\prn*{\nrm*{\grad f(\vecW_t;\vecZ_t)-\grad F(\vecW_t)} \ge \sqrt{t}} \DERIV t\\
&\le \int_{t=0}^\infty e^{-t/\sigma_F^2} \DERIV t\\
&=\sigma_F^2.
\end{align*}
Now by Young's Inequality we have pointwise
\[ \nrm*{\nabla f(\vecW_t;\vecZ_t)}^2 \le 2\nrm*{\nabla f(\vecW_t;\vecZ_t)-\nabla F(\vecW_t)}^2+2\nrm*{\nabla F(\vecW_t)}^2,\]
and combining with the above gives 
\[ \mathbb{E}\brk*{\nrm*{\nabla f(\vecW_t;\vecZ_t)}^2} \le 2\sigma_F^2+2\mathbb{E}\brk*{\nrm*{\nabla F(\vecW_t)}^2}.\]
Analogously, note
\begin{align*}
\mathbb{E}\brk*{\nrm*{\grad f(\vecW_t;\vecZ_t)-\grad F(\vecW_t)}^3} &= \int_{t=0}^\infty \mathbb{P}\prn*{\nrm*{\grad f(\vecW_t;\vecZ_t)-\grad F(\vecW_t)}^3 \ge t} \DERIV t \\
&= \int_{t=0}^\infty \mathbb{P}\prn*{\nrm*{\grad f(\vecW_t;\vecZ_t)-\grad F(\vecW_t)} \ge t^{1/3}} \DERIV t \\
&\le \int_{t=0}^\infty e^{-t^{2/3}/\sigma_F^2} \DERIV t\\
&\le 2\sigma_F^3.
\end{align*}
The inequality $\nrm*{\mathbf{a}+\mathbf{b}}^3 \le 4\nrm*{\mathbf{a}}^3+4\nrm*{\mathbf{b}}^3$ thus yields
\[ \mathbb{E}\brk*{\nrm*{\grad f(\vecW_t;\vecZ_t)}^3} \le 8\sigma_F^3 + 4\mathbb{E}\brk*{\nrm*{\grad F(\vecW_t)}^3}.\]

For a high probability statement, note for any $0 \le t \le T-1$, we have $\nrm*{\grad f(\vecW_t;\vecZ_t)-\grad F(\vecW_t)} \ge \sigma_F \sqrt{\log\prn*{T/\delta}}$ with probability at most $\delta/T$. A Union Bound and Triangle Inequality implies that with probability at least $1-\delta$ we have
\[ \nrm*{\grad f(\vecW_t;\vecZ_t)} \le \nrm*{\grad F(\vecW_t)} + \sigma_F \sqrt{\log\prn*{T/\delta}} \FORALLTEXT 0 \le t \le T-1.\]
\end{proof}

Now analogously as before with \pref{lem:onesteprecursion}, we prove a one-step discretization result. The main difference now is that we have to do the argument in a way that handles the stochasticity of the gradient estimates, but the same idea goes through.
\begin{lemma}\label{lem:onesteprecursionstochastic}
For one iteration of \SGLDTEXTSPACE starting at arbitrary $\vecW_t$,
\begin{align*}
\mathbb{E}_{\vecEps_t,\vecZ_t} \brk*{\theta\prn*{\Phi(\vecW_{t+1})}} &\le \theta\prn*{\Phi(\vecW_t)} - \eta  \theta'\prn*{\Phi(\vecW_t)} F_{\epsilon}(\vecW_t)\\
&+ \frac12 \eta^2 \mathbb{E}_{\vecZ_t} \brk*{\nrm*{\nabla f(\vecW_t;\vecZ_t)}^2}+\frac{2C}3 \eta^3\mathbb{E}_{\vecZ_t} \brk*{\nrm*{\nabla f(\vecW_t;\vecZ_t)}^3} + 2C(\eta d/\beta)^{3/2}
\end{align*}
where $p$ and $C$ are defined from \pref{lem:thirdordersmoothfromregularity}.
\end{lemma}
\begin{proof}
First, apply \pref{lem:thirdordersmoothfromregularity} with $\vecW=\vecW_t$ and $\vecU = -\eta \nabla f(\vecW_t;\vecZ_t)+\sqrt{2\eta/\beta} \vecEps_t$ to obtain
\begin{align*}
\theta\prn*{\Phi(\vecW_{t+1})}
&= \theta\prn*{\Phi(\vecW_t-\eta \nabla f(\vecW_t;\vecZ_t)+\sqrt{2\eta/\beta} \vecEps_t)} \\
&\le\theta\prn*{\Phi(\vecW_t)}+ \theta'\prn*{\Phi(\vecW_t)}\tri*{ \nabla \Phi(\vecW_t),-\eta \nabla f(\vecW_t;\vecZ_t)+\sqrt{2\eta/\beta} \vecEps_t}\\
&+\frac12 \theta'\prn*{\Phi(\vecW_t)} \tri*{ \nabla^2 \Phi(\vecW_t) \prn*{-\eta \nabla f(\vecW_t;\vecZ_t)+\sqrt{2\eta/\beta} \vecEps_t}, -\eta \nabla f(\vecW_t;\vecZ_t)+\sqrt{2\eta/\beta} \vecEps_t } \\
&\hspace{1in}+\frac{C}6 \nrm*{-\eta \nabla f(\vecW_t;\vecZ_t)+\sqrt{2\eta/\beta} \vecEps_t}^3
\end{align*}
where $C$ is defined in the proof of \pref{lem:thirdordersmoothfromregularity}.

We take expectations of this inequality with respect to $\vecEps_t$ and $\vecZ_t$. Let's consider what each term of the upper bound becomes when we take expectations.
\begin{itemize}
\item First order term: Since $\nabla f(\vecW_t;\vecZ_t)$ is unbiased, $\mathbb{E}_{\vecEps_t,\vecZ_t}\brk*{\nabla f(\vecW_t;\vecZ_t)} = \grad F(\vecW_t)$. Thus as $\vecEps_t$ has mean of the 0 vector,
\begin{align*}
&\mathbb{E}_{\vecEps_t,\vecZ_t}\brk*{\theta'\prn*{\Phi(\vecW_t)}\tri*{ \nabla \Phi(\vecW_t),-\eta \nabla f(\vecW_t;\vecZ_t)+\sqrt{2\eta/\beta} \vecEps_t}} \\
&=  \theta'\prn*{\Phi(\vecW_t)} \prn*{\tri*{ \nabla\Phi(\vecW_t), -\eta \mathbb{E}_{\vecEps_t,\vecZ_t}\brk*{\nabla f(\vecW_t;\vecZ_t)}} + \tri*{ \nabla\Phi(\vecW_t),\sqrt{2\eta/\beta}\mathbb{E}_{\vecEps_t,\vecZ_t}\brk*{\vecEps_t} }}\\
&= -\eta \theta'\prn*{\Phi(\vecW_t)}\tri*{ \nabla \Phi(\vecW_t),\nabla F(\vecW_t)}.
\end{align*}
\item Second order term: Note
\begin{align*}
&\mathbb{E}_{\vecEps_t,\vecZ_t} \brk*{\theta'\prn*{\Phi(\vecW_t)} \tri*{ \nabla^2 \Phi(\vecW_t) \prn*{-\eta \nabla f(\vecW_t;\vecZ_t)+\sqrt{2\eta/\beta} \vecEps_t}, -\eta \nabla f(\vecW_t;\vecZ_t)+\sqrt{2\eta/\beta} \vecEps_t}}\\
&= \eta^2\theta'\prn*{\Phi(\vecW_t)} \mathbb{E}_{\vecEps_t,\vecZ_t}\brk*{\tri*{ \nabla^2 \Phi(\vecW_t)\nabla f(\vecW_t;\vecZ_t), \nabla f(\vecW_t;\vecZ_t)}} \\
&\hspace{1in}- 2\eta \prn*{2\eta/\beta}^{1/2} \theta'\prn*{\Phi(\vecW_t)} \tri*{\nabla^2 \Phi(\vecW_t)\mathbb{E}_{\vecEps_t}\brk*{\vecEps_t},\mathbb{E}_{\vecZ_t}\brk*{\nabla f(\vecW_t;\vecZ_t)}}  \\
&\hspace{1in} + \prn*{2\eta/\beta }\theta'\prn*{\Phi(\vecW_t)} \mathbb{E}_{\vecEps_t,\vecZ_t}\brk*{\tri*{\nabla^2 \Phi(\vecW_t) \vecEps_t,\vecEps_t}}\\
&= \eta^2\theta'\prn*{\Phi(\vecW_t)} \mathbb{E}_{\vecZ_t}\brk*{\tri*{ \nabla^2 \Phi(\vecW_t)\nabla f(\vecW_t;\vecZ_t), \nabla f(\vecW_t;\vecZ_t)}} + \prn*{2\eta/\beta} \theta'\prn*{\Phi(\vecW_t)} \Delta\Phi(\vecW_t).
\end{align*}
Here we used that $\vecEps_t$ has zero mean as a vector and that $\vecEps_t, \vecZ_t$ are clearly independent to compute the cross term.
The calculation of 
\[ \mathbb{E}_{\vecEps_t,\vecZ_t}\brk*{\tri*{\nabla^2 \Phi(w_t) \vecEps_t,\vecEps_t}} = \Delta\Phi(\vecW_t)\]
is the same as before. Note this expectation has no $\vecZ_t$ dependence.
\item Third order term: Again we use for all $\mathbf{a},\mathbf{b}\in\mathbb{R}^d$,
\[ \nrm*{\mathbf{a}+\mathbf{b}}^3 \le4\nrm*{\mathbf{a}}^3+4\nrm{\mathbf{b}}^3.\]
Using this inequality pointwise we obtain
\begin{align*}
\mathbb{E}_{\vecEps_t,\vecZ_t}\brk*{\nrm*{-\eta \nabla f(\vecW_t;\vecZ_t)+\sqrt{2\eta/\beta} \vecEps_t}^3} &\le 4\eta^3 \mathbb{E}_{\vecZ_t}\brk*{\nrm*{\nabla f(\vecW_t;\vecZ_t)}^3}+4\prn*{2\eta/\beta}^{3/2} d^{3/2}.
\end{align*}
The last step is because $\nrm*{\vecEps_t}=\sqrt{d}$ always holds deterministically.
\end{itemize}
We put this together, noting $\theta'\prn*{\Phi(\vecW_t)} \ge 0$ from \pref{lem:thirdordersmoothfromregularity} which means we can use the admissability condition \pref{eq:admissablepotentialF} which we use to upper bound the first order term. This gives
\begin{align*}
&\mathbb{E}_{\vecEps_t,\vecZ_t} \brk*{\theta\prn*{\Phi(\vecW_{t+1})}}\\
&\le \theta\prn*{\Phi(\vecW_t)} -\eta \theta'\prn*{\Phi(\vecW_t)}\tri*{ \nabla \Phi(\vecW_t),\nabla F(\vecW_t)}\\
&+\frac12\prn*{\eta^2\theta'\prn*{\Phi(\vecW_t)} \mathbb{E}_{\vecZ_t}\brk*{\tri*{ \nabla^2 \Phi(\vecW_t)\nabla f(\vecW_t;\vecZ_t), \nabla f(\vecW_t;\vecZ_t)}} + \prn*{2\eta/\beta} \theta'\prn*{\Phi(\vecW_t)} \Delta\Phi(\vecW_t)} \\
&\hspace{1in}+\frac{C}6 \prn*{4\eta^3 \mathbb{E}_{\vecZ_t}\brk*{\nrm*{\nabla f(\vecW_t;\vecZ_t)}^3}+4\prn*{2\eta/\beta}^{3/2} d^{3/2}}  \\
&\le \theta\prn*{\Phi(\vecW_t)} -\eta \theta'\prn*{\Phi(\vecW_t)}\prn*{F_{\epsilon}(\vecW_t)+\frac1{\beta}\Delta\Phi(\vecW_t)}\\
&+\frac12\prn*{\eta^2\theta'\prn*{\Phi(\vecW_t)} \nrm*{\grad^2\Phi(\vecW_t)}_{\OPNORM}\mathbb{E}_{\vecZ_t}\brk*{\nrm*{\nabla f(\vecW_t;\vecZ_t)}^2} + \prn*{2\eta/\beta} \theta'\prn*{\Phi(\vecW_t)} \Delta\Phi(\vecW_t) }\\
&\hspace{1in}+\frac{C}6 \prn*{4\eta^3 \mathbb{E}_{\vecZ_t}\brk*{\nrm*{\nabla f(\vecW_t;\vecZ_t)}^3}+4\prn*{2\eta/\beta}^{3/2} d^{3/2}}. 
\end{align*}
The second inequality follows analogously as in the proof of \pref{lem:onesteprecursion}; pointwise we have
\[ \nabla f(\vecW_t;\vecZ_t)^T \nabla^2 \Phi(\vecW_t) \nabla f(\vecW_t;\vecZ_t) \le \nrm*{\nabla f(\vecW_t;\vecZ_t)}^2 ||\nabla^2 \Phi(\vecW_t)||_{\OPNORM},\]
and the fact that
\[ \theta'(z) \le \frac1{\rho_{\Phi,2}(z)}\]
always holds. Also note the terms $\eta \theta'\prn*{\Phi(\vecW_t)} \cdot \frac1{\beta} \Delta \Phi(\vecW_t)$ and $\frac12 \prn*{2\eta/\beta} \theta'\prn*{\Phi(\vecW_t)}\Delta \Phi(\vecW_t)$ cancel out. Thus our above bound becomes
\begin{align*}
\mathbb{E}_{\vecEps_t,\vecZ_t} \brk*{\theta\prn*{\Phi(\vecW_{t+1})}} &\le \theta\prn*{\Phi(\vecW_t)} -\eta \theta'\prn*{\Phi(\vecW_t)}F_{\epsilon}(\vecW_t)\\
&+\frac12\eta^2\mathbb{E}_{\vecZ_t}\brk*{\nrm*{\nabla f(\vecW_t;\vecZ_t)}^2} +\frac{2C}3 \eta^3 \mathbb{E}_{\vecZ_t}\brk*{\nrm*{\nabla f(\vecW_t;\vecZ_t)}^3}+2C\prn*{\eta/\beta}^{3/2} d^{3/2}. 
\end{align*}
\end{proof}

Next, analogously as \pref{lem:actuallypotentialsimplesetting}, we prove that $\Phi$ indeed is a potential.
\begin{lemma}\label{lem:actuallypotentialstochasticsetting}
Suppose $\beta \ge d$ and $\epsilon < 1/e$. Additionally suppose $\epsilon < C'$ where $z=\frac1{C'}$ is the largest solution to $\frac{\log^7(\sqrt{20}z(\log z))}{\log^8(z)}=1/2$ (the existence of such a $C'>0$ is obvious). 

With probability at least $0.9$, we have that
\[ \rho_{\Phi}\prn*{\Phi(\vecW_t)} \le \kappa' \rho_{\Phi}\prn*{\Phi(\vecW_0)},\]
for all $0 \le t \le T-1$ if $T, \eta$ are chosen as follows, where $\kappa'$ comes from \pref{ass:iteratesinballassumption}. First define 
\[ A_0(\vecW_0) = \theta\prn*{\rho_{\Phi}^{-1}\prn*{\kappa' \rho_{\Phi}\prn*{\Phi(\vecW_0)}}}-\theta\prn*{\Phi(\vecW_0)} >0, A_1(\vecW_0) = 12\sqrt{2} CB_{\textsc{grad}}^3.\]
Here $C$ comes from \pref{lem:onesteprecursion} and $B_{\textsc{grad}}:=2\prn*{L\prn*{R_1+\nrm*{\vecW^{\star}}}^s+\sigma_F}$, $\sigma_F$ comes from \pref{ass:gradnoiseassumption}, $L$ comes from \pref{ass:holderF} and $R_1$ comes from \pref{ass:iteratesinballassumption}. (Again if necessary take $C\leftarrow\max(C,1)$, $B_{\textsc{grad}}\leftarrow\max(B_{\textsc{grad}},1)$, $\sigma_F \leftarrow \max(\sigma_F, 1)$.)

Also define
\[ r(\vecW_0) = \min\prn*{1, \frac3{4C}, \frac1{B_{\textsc{grad}}}}, C(\vecW_0) = \min\prn*{1,\frac{A_0(\vecW_0)^2 r(\vecW_0)}{128A_1(\vecW_0)^2}}.\] 

Now we choose $T, \eta$ based on cases:
\begin{enumerate}
\item If $\beta \le \frac{d}{\epsilon/\prn*{\log 1/\epsilon}^5}$: Take $T, \eta$ such that
\[ \eta = r(\vecW_0) \frac{d}{\beta},\]
and $T$ is the unique solution to the equation 
\[ z \log^7(20z) = C(\vecW_0) \frac{\beta^2}{d^2}.\]
Existence and uniqueness is clear since $z \log^7(20z)$ is surjective on $\mathbb{R}_{\ge 0}$ and for every positive real $t$, exactly one positive real $z$ is such that $z \log^7 (20z)=t$. By the same argument as in \pref{lem:actuallypotentialsimplesetting}, this means $T \le \frac{\beta^2}{d^2}$. Also note this means 
\[ \eta \le \min\prn*{1, \frac3{4C}, \frac1{B_{\textsc{grad}}},  \frac{d}{\beta}}.\]
\item If $\beta \ge \frac{d}{\epsilon/\prn*{\log 1/\epsilon}^5}$: Take $T, \eta$ such that 
\[ \eta = r(\vecW_0)\cdot\frac{\epsilon}{\prn*{\log 1/\epsilon}^5}, T = \floor{ \frac{C(\vecW_0)}{\epsilon^2}\cdot\prn*{\log 1/\epsilon}^2}.\]
Note this implies $T \le C(\vecW_0) \cdot \frac{1}{\epsilon^2}\prn*{\log 1/\epsilon}^2 \le \frac{1}{\epsilon^2}\prn*{\log 1/\epsilon}^2$.
\end{enumerate}
Note as $\epsilon \le 1/e$, if $\beta \ge d\sqrt{\log^7 (20)/C(\vecW_0)}$ and $\epsilon \le \sqrt{C(\vecW_0)}$, then $T \ge 1$. 

Also note in all cases that $\eta \le \min\prn*{1, r(\vecW_0)}$, since $\epsilon \le 1$, $\beta \ge d$.
\end{lemma}
\begin{proof}
Define $\mathfrak{F}_t$ by the natural filtration with respect to $\vecW_j, \vecEps_j,\vecZ_j$ for all $j \le t$. Again, let 
\[ \tau := \min\{T, \inf\{t:\rho_{\Phi}\prn*{\Phi(\vecW_t)} > \kappa' \rho_{\Phi}\prn*{\Phi(\vecW_0)}\},\] 
where $\kappa'$ comes from \pref{ass:iteratesinballassumption}.

Again, define a stochastic process $Y_t$ by 
\[ Y_t := \begin{cases}\theta\prn*{\Phi(\vecW_t)}+\sum_{j=0}^{t-1} \prn*{\eta F_{\epsilon}(\vecW_j) - R(\vecW_j, \Phi, \eta, \beta, d)}&\IF t \le \tau \\ Y_{\tau}&\OTHERWISEIF t > \tau, \end{cases}\]
where now we have
\[ R(\vecW_j, \Phi, \eta, \beta, d) := \frac12 \eta^2 \prn*{2\sigma_F^2+2\nrm*{\nabla F(\vecW_t)}^2}+\frac{2C}3 \eta^3\prn*{8\sigma_F^3 + 4\nrm*{\grad F(\vecW_t)}^3} + 2C(\eta d/\beta)^{3/2}.\]
Note from \pref{lem:stochasticgradsetting} that 
\[ R(\vecW_j, \Phi, \eta, \beta, d) \ge \frac12 \eta^2 \nrm*{\nabla f(\vecW_t;\vecZ_t)}^2+\frac{2C}3 \nrm*{\grad f(\vecW_t;\vecZ_t)}^3 + 2C(\eta d/\beta)^{3/2}.\]
The same derivation as in the proofs of \pref{lem:Ytsupermartingalesimplesetting} (now using \pref{lem:onesteprecursionstochastic}) and \pref{lem:azumahoeffding} give the following two results adapted to this setting:
\begin{lemma}\label{lem:Ytsupermartingalestochasticgradsetting}
$Y_t$ is a supermartingale with respect to $\mathfrak{F}_t$.
\end{lemma}
\begin{lemma}\label{lem:azumahoeffdingstochasticgradsetting}
With probability at least $1-\delta$, we have
\[ Y_t - Y_0 \le \sqrt{\frac12 \prn*{\sum_{t=0}^{T-1} C(\eta, t, d, \beta)^2}\log(T/\delta)}\]
for all $1 \le t \le T$, where 
\[C(\eta, t, d, \beta) = 4\sqrt{\theta\prn*{\rho_{\Phi}^{-1}\prn*{\kappa \rho_{\Phi}\prn*{\Phi(w_0)}}}} \cdot \nrm*{-\eta \nabla f(\vecW_t;\vecZ_t)+\sqrt{\frac{2\eta}{\beta}} \vecEps_t}+ 4\nrm*{-\eta \nabla f(\vecW_t;\vecZ_t)+\sqrt{\frac{2\eta}{\beta}} \vecEps_t}^2.\]
\end{lemma}
Note we need to have high-probability control over the $\nabla f(\vecW_t;\vecZ_t)$, rather than just control over their moments, to upper bound the $C(\eta, t, d, \beta)$ in the above.

Denote the event from \pref{lem:azumahoeffdingstochasticgradsetting} with $\delta=0.05$ by $E_1'$, which occurs with probability at least $0.95$. 

Also denote the event from \pref{lem:stochasticgradsetting} with $\delta=0.05$ by $E_1''$, which occurs with probability at least $0.95$. 

Now define $E_1=E_1' \cap E_1''$, which occurs with probability at least 0.9. We claim that if $E_1$ occurs, then for all $0 \le t \le T-1$ we have $\rho_{\Phi}\prn*{\Phi(\vecW_t)} \le\kappa' \rho_{\Phi}\prn*{\Phi(\vecW_0)}$. This clearly finishes the proof.

Suppose for the sake of contradiction that there exists $0 \le t \le T-1$ where $\rho_{\Phi}\prn*{\Phi(\vecW_t)} > \kappa' \rho_{\Phi}\prn*{\Phi(\vecW_0)}$. Hence we have $\tau < T$ and for that $\tau$, $\rho_{\Phi}\prn*{\Phi(\vecW_{\tau})} > \kappa' \rho_{\Phi}\prn*{\Phi(\vecW_0)}$. First note if $T=0$ this is not possible, so suppose $T \ge 1$ from now on. Thus $\log(20T) > 1$.

Then by \pref{ass:iteratesinballassumption}, for all $t < \tau$ we have $\vecW_t \in \ball(\vecOrigin,R_1')$. Hence for all $t<\tau$ we have by \pref{ass:holderF} that
\[ \nrm*{\grad F(\vecW_t)}=\nrm*{\grad F(\vecW_t)-\grad F(\vecW^{\star})} \le L\nrm*{\vecW_t-\vecW^{\star}}^s \le L\prn*{R_1+\nrm*{\vecW^{\star}}}^s.\]
Since $E_1''$ holds as we condition on $E_1$, by \pref{lem:stochasticgradsetting} gives that for all $0 \le t \le T-1$ we have
\[ \nrm*{\grad f(\vecW_t;\vecZ_t)} \le \nrm*{\grad F(\vecW_t)}+\sigma_F\sqrt{\log(20T)} \le B_{\textsc{grad}}\sqrt{\log(20T)}.\numberthis\label{eq:upperboundgradhighprob}\]
We assumed without loss of generality that $L, R_1 \ge 1$, thus the above upper bound is at least 1. (Note compared to the proof of \pref{lem:actuallypotentialsimplesetting} that the definition of $B_{\textsc{grad}}$ changed.)

\pref{lem:azumahoeffding} gives us a way to upper bound $Y_{\tau}-Y_0$ (since we condition on $E_1$), so now let's derive a lower bound on $Y_{\tau}-Y_0$. We will then show that these upper and lower bounds are contradictory to complete the proof.

Note 
\begin{align*}
R(\vecW_j, \Phi, \eta, \beta, d) &= \frac12 \eta^2 \prn*{2\sigma_F^2+2\nrm*{\nabla F(\vecW_t)}^2}+\frac{2C}3 \eta^3\prn*{8\sigma_F^3 + 4\nrm*{\grad F(\vecW_t)}^3} + 2C(\eta d/\beta)^{3/2} \\
&\le \frac12\eta^2 B_{\textsc{grad}}^2 + \frac{2C}3 \eta^3 B_{\textsc{grad}}^3 + 2C(\eta d/\beta)^{3/2} \\
&\le \frac{3}{2} C B_{\textsc{grad}}^3+ 2C(\eta d/\beta)^{3/2}.
\end{align*}
The above uses $\eta \le 1$, our assumption we made without loss of generality that $C \ge 1$, the definition of $B_{\textsc{grad}}$, and that $B_{\textsc{grad}} \ge 1$. 

Now by definition of $\tau$ and as $Y_0=\theta\prn*{\Phi(\vecW_0)}$, similarly to the proof of \pref{lem:actuallypotentialsimplesetting} we get 
\begin{align*}
Y_{\tau}-Y_0 &> \theta\prn*{\rho_{\Phi}^{-1}\prn*{\kappa' \rho_{\Phi}\prn*{\Phi(\vecW_0)}}}-\theta\prn*{\Phi(\vecW_0)}-\prn*{\frac32 C B_{\textsc{grad}}^3 \eta^2 + 2C (\eta d / \beta)^{3/2}}T. \numberthis\label{eq:lowerboundgradestimatepotential}
\end{align*}
Now we use \pref{lem:azumahoeffdingstochasticgradsetting} to upper bound $Y_{\tau}-Y_0$. Again recall $\nrm{\vecEps_t}=\sqrt{d}$ always holds, as well as $\theta\le1$. Thus via the same derivation as in the proof of \pref{lem:actuallypotentialsimplesetting}, using \pref{eq:upperboundgradhighprob}, we obtain for all $t<\tau$ that
\begin{align*}
C(\eta, t, d, \beta) &\le 24 \prn*{B_{\textsc{grad}} \eta+\sqrt{\eta d / \beta} }\log^6(20T).
\end{align*}

This implies 
\begin{align*}
Y_{\tau} - Y_0 &\le \sqrt{\frac12 \prn*{\sum_{t=0}^{\tau-1} C(\eta, t, d, \beta)^2}\log^7(20T)}\\
&\le 12\sqrt{2} \prn*{ B_{\textsc{grad}} \eta+\sqrt{\eta d / \beta} } \sqrt{T\log^7\prn*{20T}}.\numberthis\label{eq:upperboundgradestimatepotential}
\end{align*}
Similarly as before, putting together our lower and upper bounds \pref{eq:lowerboundgradestimatepotential} and \pref{eq:upperboundgradestimatepotential} on $Y_{\tau}-Y_0$, now we aim to show the following cannot hold:
\begin{align*}
0<A_0(\vecW_0) &= \rho_{\Phi}^{-1}\prn*{\kappa' \rho_{\Phi}\prn*{\Phi(\vecW_0)}}-\Phi(\vecW_0) \\
&\le \prn*{\frac32 C B_{\textsc{grad}}^3 \eta^2 + 2C (\eta d / \beta)^{3/2}}T + 12\sqrt{2} \prn*{ B_{\textsc{grad}} \eta+\sqrt{\eta d / \beta} } \sqrt{T\log^7\prn*{20T}}. \numberthis \label{eq:potentialestimateineq}
\end{align*}
Again note the left hand side is a positive constant. We aim to show with our choice of $\eta$ and $T$ that this gives contradiction. Break into our original cases:
\begin{enumerate}
    \item If $\beta \le \frac{d}{\epsilon/\prn*{\log 1/\epsilon}^5}$: Once more by our choice of $\eta$, we have $\eta \le 1, \frac{d}{\beta}=\frac{\eta}{r(\vecW_0)}$, thus
    \[  \prn*{\eta d / \beta}^{1/2}, \prn*{\eta d / \beta}^{3/2} \le \frac{\eta}{r(\vecW_0)^{3/2}}. \]
    Therefore, an analogous derivation as in the proof of \pref{lem:actuallypotentialsimplesetting} gives that the right hand side of \pref{eq:potentialestimateineq} is at most 
    \begin{align*}  
    \frac{A_1(\vecW_0)}{r(\vecW_0)^{3/2}} \prn*{2\eta^2 T + 2\eta\sqrt{T\log^7 (20T)}} &\le \frac{4A_1(\vecW_0)}{r(\vecW_0)^{3/2}} \eta \sqrt{T\log^7 (20T)} \\
    &\le \frac{4A_1(\vecW_0)}{r(\vecW_0)^{3/2}} \cdot r(\vecW_0) \cdot \frac{d}{\beta} \cdot\frac{A_0(\vecW_0) r(\vecW_0)^{1/2}}{8A_1(\vecW_0)} \cdot \frac{\beta}d \\
    &< \frac{A_0(\vecW_0)}2.
    \end{align*}
    The first inequality follows as $T \ge 1$ and $\eta\sqrt{T} \le 1$. The fourth inequality follows recalling the definitions of $\eta$ and $T$ in terms of $A_0(\vecW_0)$, $A_1(\vecW_0)$, $r(\vecW_0)$ and $C(\vecW_0)$ (note $z \log^7(20z)$ is increasing on $x \ge 1$). The last inequality follows from definition of $C(\vecW_0)$. As $A_0(\vecW_0)>0$, this contradicts \pref{eq:potentialestimateineq} which is exactly what we want.
    \item If $\beta \ge \frac{d}{\epsilon/\prn*{\log 1/\epsilon}^5}$: The strategy is similar. This time, we have by the condition that
    \[ \frac{d}{\beta} \le \frac{\epsilon}{\prn*{\log 1/\epsilon}^5} = \frac{\eta}{r(\vecW_0)},\]
    which implies 
    \[\prn*{\frac{\eta d}{\beta}}^{1/2} \le \frac{\eta}{r(\vecW_0)^{1/2}}.\]
    Therefore, an analogous derivation as in the proof of \pref{lem:actuallypotentialsimplesetting} gives that the right hand side of \pref{eq:potentialestimateineq} is at most 
    \begin{align*}  
    &\frac{A_1(\vecW_0)}{r(\vecW_0)^{3/2}} \prn*{2\eta^2 T + 2\eta\sqrt{T\log^7 (20T)}} \\
    &\le \frac{4A_1(\vecW_0)}{r(\vecW_0)^{3/2}} \eta \sqrt{T\log^7 (20T)} \\
    &\le \frac{4A_1(\vecW_0)}{r(\vecW_0)^{3/2}} \cdot r(\vecW_0) \cdot \frac{\epsilon}{\prn*{\log 1/\epsilon}^5} \cdot \frac{\sqrt{C(\vecW_0)}}{\epsilon} \cdot (\log 1/\epsilon)\cdot  \sqrt{\log^7 \prn*{\frac{20}{\epsilon^2}\prn*{\log \prn*{1/\epsilon}^2}}} \\
    &\le \frac{A_0(\vecW_0)}2 \sqrt{\frac{\log^7\prn*{\frac{\sqrt{20}}{\epsilon}\prn*{\log 1/\epsilon}}}{\log^8\prn*{1/\epsilon}}} \\
    &\le \frac{A_0(\vecW_0)}{2\sqrt{2}}.
    \end{align*}
    The first inequality follows as $T \ge 1$, $\epsilon \le 1/e$ and so $\eta\sqrt{T} \le 1$. The second inequality is by definition of $\eta$ and $T$ and as $C(\vecW_0) \le 1$, $T \ge 1$ (note $z \log^7 (20z)$ is increasing on $z \ge 1$). The third inequality is by definition of $C(\vecW_0)$. The last inequality is by definition of $\epsilon$ and $C'$. In detail, since $\frac{\log^7(\sqrt{20}z(\log z))}{\log^8(z)}$ is continuous, decreasing for large enough $z$, and $\lim_{z\rightarrow\infty}\frac{\log^7(\sqrt{20}z(\log z))}{\log^8(z)}=0$, let $z:=\frac1{C'}$ be the largest solution to $\frac{\log^7(\sqrt{20}z(\log z))}{\log^8(z)}=1/2$. Thus, as $\epsilon<C'$ we have the last inequality. This contradicts \pref{eq:potentialestimateineq} as $A_0(\vecW_0)>0$, which again is exactly what we want.
\end{enumerate}
In all cases we obtain a contradiction conditioned on $E_1$, which occurs with probability at least 0.9 from the earlier discussion. Hence with probability at least 0.9 we have $\rho_{\Phi}(\Phi(\vecW_t)) \le \kappa' \rho_{\Phi}(\Phi(\vecW_0))$ for all $0 \le t \le T-1$ as desired.
\end{proof}

Finally, we can conclude again similarly as the proof of \pref{thm:discretizationsimpleoracleformal} to prove \pref{thm:discretizationsimpleestimateformal}.
\begin{proof}
Again note by the logic in \pref{lem:actuallypotentialstochasticsetting}, based on our cases on $\beta$ and $\epsilon$, the $T$ that we choose will always be at least 1. Moreover, in the same way as in the proof of \pref{thm:discretizationsimpleoracleformal}, we can reduce to proving the case when $\epsilon \le \min\prn*{1/e, C', \sqrt{\frac{C(\vecW_0)}{\log^7(20)}}}$ where $C'$ is defined from \pref{lem:actuallypotentialstochasticsetting}. We also have $\epsilon, \eta \le 1$ as a consequence.

Let $E_1$ be the event that $\rho_{\Phi}(\Phi(\vecW_t)) \le \kappa' \rho_{\Phi}(\Phi(\vecW_0))$ for all $1 \le t \le T$. From \pref{lem:actuallypotentialstochasticsetting}, we know $E_1$ holds with probability at least 0.9 for the choice of $\eta, T$ given there. By \pref{ass:iteratesinballassumption}, this means that conditioned on $E_1$, all the $\vecW_t \in \ball(\vecOrigin,R_1)$ for $1 \le t \le T$. By the same derivation as \pref{lem:actuallypotentialsimplesetting}, this means 
\[ \nrm*{\grad F(\vecW_t)} \le L\prn*{R_1+\nrm*{\vecW^{\star}}}^s \FORALLTEXT 1 \le t \le T.\numberthis\label{eq:gradupperboundestimate}\]
Once more, by Markov's Inequality, with probability at least 0.9,
\[ \sum_{t=0}^{T-1} F_{\epsilon}(\vecW_t) \theta'\prn*{\Phi(\vecW_t)} \le 10\mathbb{E}\brk*{ \sum_{t=0}^{T-1} F_{\epsilon}(\vecW_t) \theta'\prn*{\Phi(\vecW_t)}}.\numberthis\label{eq:markovestimate}\]
Let $E_2$ be the event that this above inequality holds.

Summing and telescoping from \pref{lem:onesteprecursionstochastic} and using that $\theta\prn*{\Phi(z)} \ge 0$, we obtain
\begin{align*}
\eta \mathbb{E}\brk*{\sum_{t=0}^{T-1} F_{\epsilon}(\vecW_t) \theta'\prn*{\Phi(\vecW_t)}} &\le \theta(\Phi(\vecW_0)) + 2C \prn*{\eta d / \beta}^{3/2} T \\
&+ \frac12 \eta^2 \sum_{t=0}^{T-1} \mathbb{E}\brk*{\nrm*{\nabla f(\vecW_t;\vecZ_t)}^2}+\frac{2C}3 \eta^3 \sum_{t=0}^{T-1} \mathbb{E}\brk*{\nrm*{\nabla f(\vecW_t;\vecZ_t)}^3}.
\end{align*}
Here, we took full expectations over $\{\vecEps_t,\vecZ_t\}_{0\le t \le T-1}$ in the above.

Using \pref{lem:stochasticgradsetting} and taking full expectations, we have
\[ \mathbb{E}\brk*{\nrm*{\nabla f(\vecW_t;\vecZ_t)}^2} \le 2\sigma_F^2+2\mathbb{E}\brk*{\nrm*{\nabla F(\vecW_t)}^2}.\]
\[ \mathbb{E}\brk*{\nrm*{\nabla f(\vecW_t;\vecZ_t)}^3} \le 8\sigma_F^3+4\mathbb{E}\brk*{\nrm*{\nabla F(\vecW_t)}^3}.\]
Using these in the above we see that
\begin{align*}
\eta \mathbb{E}\brk*{\sum_{t=0}^{T-1} F_{\epsilon}(\vecW_t) \theta'\prn*{\Phi(\vecW_t)}} &\le \theta(\Phi(\vecW_0)) + 2C \prn*{\eta d / \beta}^{3/2} T \\
&\hspace{1in}+ \frac12 \eta^2 \sum_{t=0}^{T-1} \prn*{2\sigma_F^2+2\mathbb{E}\brk*{\nrm*{\nabla F(\vecW_t)}^2}}\\
&\hspace{1in}+\frac{2C}3 \eta^3 \sum_{t=0}^{T-1} \prn*{8\sigma_F^3+4\mathbb{E}\brk*{\nrm*{\nabla F(\vecW_t)}^3}}. \numberthis\label{eq:expupperboundestimate}
\end{align*}

Finally, consider $\sum_{t=0}^{T-1} \mathbb{E}\brk*{\nrm*{\nabla F(\vecW_t)}^r}$ for $r \in \{2,3\}$. The same logic as the proof of \pref{thm:discretizationsimpleoracleformal} gives that with probability at least 0.975 for a given $r \in \{2,3\}$
\begin{align*}
\sum_{t=0}^{T-1} \nrm*{\nabla F(\vecW_t)}^r &\ge  \sum_{t=0}^{T-1} \mathbb{E}\brk*{\nrm*{\nabla F(\vecW_t)}^r}- 2\sqrt{2}\max_{0 \le t \le T-1}\prn*{\nrm*{\nabla F(\vecW_t)}^r} \cdot \sqrt{T\log 40}.\numberthis\label{eq:discretizationestimatemartingale}
\end{align*}
Let $E_3$ be the intersection of these two events for $r\in\{2,3\}$, so $E_3$ has probability at least 0.95.

Let 
\begin{align*}
M_1(\vecW_0) &= 10\max\prn*{\theta\prn*{\Phi(\vecW_0)}, 6C B_{\textsc{grad}}^3},
\end{align*}
which is just a $\vecW_0$-dependent constant. (Note compared to the proof of \pref{thm:discretizationsimpleoracleformal} that the definition of $B_{\textsc{grad}}$ changed.)

Now we put these steps together and do a Union Bound over $E_1, E_2, E_3$. Let $E=E_1 \cap E_2 \cap E_3$; we have that $E$ occurs with probability at least 0.75. 

Then conditioned on $E$, combining \pref{eq:gradupperboundestimate}, \pref{eq:markovestimate}, \pref{eq:expupperboundestimate}, \pref{eq:discretizationestimatemartingale} in the same manner we used to prove \pref{thm:discretizationsimpleoracleformal}, we see
\begin{align*}
\sum_{t=0}^{T-1}  F_{\epsilon}(\vecW_t) \theta'\prn*{\Phi(\vecW_t)}&\le 10\mathbb{E}\brk*{ \sum_{t=0}^{T-1}  F_{\epsilon}(\vecW_t) \theta'\prn*{\Phi( \vecW_t)}} \\
&\le M_1(\vecW_0) \prn*{\frac{1}{\eta} + (d / \beta)^{3/2} \eta^{1/2} T + \eta T + \eta^2 T + \sqrt{T}}.
\end{align*}
This uses that $B_{\textsc{grad}} = L\prn*{R_1+\nrm*{\vecW^{\star}}}^s+\sigma_F$ and straightforward estimates.

As before, conditioned on $E$ we have 
\[ \theta'\prn*{\Phi(\vecW_t)} \ge \theta'\prn*{\rho_{\Phi}^{-1}\prn*{\kappa' \rho_{\Phi}\prn*{\Phi(\vecW_0)}}}\FORALLTEXT 0 \le t \le T-1. \]
Thus, defining 
\[ M(\vecW_0) = \frac{M_1(\vecW_0)}{\theta'\prn*{\rho_{\Phi}^{-1}\prn*{\kappa' \rho_{\Phi}\prn*{\Phi(\vecW_0)}}}} \in (0,\infty)\]
we see that conditioned on $E$ which occurs with probability at least 0.75 we have, via identical steps as before,
\begin{align*}
\frac1T \sum_{t=0}^{T-1} F_{\epsilon}(\vecW_t) &\le 3M(\vecW_0) \prn*{\frac1{\eta T} + \eta + d/\beta}.
\end{align*}

We break into cases based on how we set $\eta, T$:
\begin{enumerate}
    \item If $\beta \le \frac{d}{\epsilon/\prn*{\log 1/\epsilon}^5}$: Now recall we set $\eta = r(\vecW_0) \frac{d}{\beta}\le  d/\beta$, and that we let $T$ be the floor of the unique solution to the equation 
    \[ z \log^7 \prn*{20z} = C(\vecW_0) \frac{\beta^2}{d^2},\]
    where $C(\vecW_0)$ is defined according to \pref{lem:actuallypotentialstochasticsetting}. Recall we had $T \ge 1$ as well as 
    \[ T \le C(\vecW_0) \frac{\beta^2}{d^2} \le \frac{\beta^2}{d^2}.\]
    Since $T \ge 1$, and as $z\log^7(20z)$ is increasing for $z \ge 1$, it follows via definition of $T$ (note $2\floor{z} \ge z$ for all $z \ge 1$) that 
    \[ 2T \log^7(40T) \ge C(\vecW_0) \frac{\beta^2}{d^2}, \]
    hence
    \[ \eta T \ge \frac{r(\vecW_0)C(\vecW_0)}{2\log^7 (40T)}\frac{\beta}d.\]
    Thus, we have with probability at least 0.75 that 
    \begin{align*}
    \frac1T \sum_{t=0}^{T-1} F_{\epsilon}(\vecW_t) &\le 6M(\vecW_0) \left(\frac{2^7\log^7(40)+2^{14}\log^7(\beta/d)}{r(\vecW_0)C(\vecW_0)} + 1\right) \frac{d}{\beta}
    \end{align*}
    that is, we obtain $\widetilde{O}(d/\beta)$ suboptimality with at most $\frac{\beta^2}{d^2}$ iterations. (This step uses the inequality $(a+b)^7 \le 2^7(a^7+b^7)$ for $a,b\ge 0$.)
    \item If $\beta \ge \frac{d}{\epsilon/(\log 1/\epsilon)^5}$: Recalling how we set $\eta$ and the definition of this case, we have $\eta, d/\beta \le \epsilon$. Moreover, note $\eta T \ge \frac{r(\vecW_0)C(\vecW_0)}{2\epsilon \prn*{\log 1/\epsilon}^3}$ by analogous logic as in the proof of \pref{thm:discretizationsimpleoracleformal}. Hence, we obtain with probability at least 0.75 that
    \begin{align*}
    \frac1T \sum_{t=0}^{T-1} F_{\epsilon}(\vecW_t) &\le 3M(\vecW_0) \prn*{\frac{2\epsilon (\log 1/\epsilon)^3}{r(\vecW_0)C(\vecW_0)}+2\epsilon} \\
    &\le 6M(\vecW_0) \prn*{\frac1{r(\vecW_0)C(\vecW_0)}+1}\epsilon \prn*{\log 1/\epsilon}^3.
    \end{align*}
    That is, we obtain $\widetilde{O}(\epsilon)$ suboptimality with at most $T \le \frac1{\epsilon^2}\prn*{\log 1/\epsilon}^2$ iterations.
\end{enumerate}
\end{proof}

\subsection{Proof of \pref{thm:discretizationloosenoracle}}\label{subsec:loosenconditionproofs}
We formally state \pref{thm:discretizationloosenoracle} and the algorithm in question as follows.
\begin{theorem}\label{thm:discretizationloosenoracleformal}
Suppose the geometric property \pref{eq:loosenedcondition} holds:
\[ \tri*{ \nabla \Phi(\vecW), \nabla F(\vecW)} \ge F_{\epsilon}(\vecW) +\min\prn*{0, \frac1{\beta} \Delta\Phi(\vecW)}.\]
Consider running \pref{alg:sgldmodlaplacian}, following the same $\eta, T$ as well as cutoff for $\epsilon$ as from \pref{lem:actuallypotentialsimplesetting}. Then we have the same runtime and error guarantees as \pref{thm:discretizationsimpleoracleformal}.
\end{theorem}

\begin{algorithm}[h!]
\caption{Modified Langevin Dynamics using Gradient Domination Information} 
\label{alg:sgldmodlaplacian}
\begin{algorithmic}
\State Initialize at $\vecW_0$. 
\For{each $t \ge 0$} 
\State Compute $\Delta_t:=\tri*{ \nabla \Phi(\vecW), \nabla F(\vecW)} - F_{\epsilon}(\vecW)$. 
    \If{$\Delta_t \ge 0$}
        \State $\vecW_{t+1} \leftarrow \vecW_t - \eta \nabla F(\vecW_t)$.
    \Else{If} 
        \State $\vecW_{t+1} \leftarrow \vecW_t - \eta \nabla F(\vecW_t) + \sqrt{\eta/\beta} \vecEps_t$ where $\vecEps_t \sim \sqrt{d}\mathcal{S}^{d-1}$ uniformly. (Recall $\mathcal{S}^{d-1}$ is the unit sphere in $\mathbb{R}^d$ so $\vecEps_t$ is philosophically a Gaussian.)
    \EndIf
\EndFor 
\end{algorithmic}
\end{algorithm} 

To prove \pref{thm:discretizationloosenoracleformal}, in fact the exact same proof of \pref{thm:discretizationsimpleoracleformal} will suffice. The main idea is that in all of our bounds involving $\nrm*{\vecEps_t}$, we use either Triangle Inequality or Young's Inequality to bound $\nrm*{-\eta \nabla F(\vecW_t)+\sqrt{\eta / \beta} \vecEps_t}$ and use that $\nrm*{\vecEps_t} \le \sqrt{d}$. However these results still hold when $\vecEps_t=0$, when no noise is added. Moreover, this result holds if $F_{\epsilon}$ is replaced with $A$ and $A$ is query-able, in the same way as described in \pref{rem:costfunction}.

Again we break the proof into similar parts, starting with the one-step discretization bound.
\begin{lemma}\label{lem:onesteprecursionloosen}
For one iteration starting at arbitrary $\vecW_t$,
\begin{align*}
\mathbb{E}_{\vecEps_t} \brk*{\theta\prn*{\Phi(\vecW_{t+1})}} &\le \theta\prn*{\Phi(\vecW_t)} - \eta  \theta'\prn*{\Phi(\vecW_t)} F_{\epsilon}(\vecW_t)\\
&\hspace{1in}+ \frac12 \eta^2 \nrm*{\nabla F(\vecW_t)}^2+1_{\cN_t}\prn*{\frac{2C}3 \eta^3\nrm*{\nabla F(\vecW_t)}^3 + 2C(\eta d/\beta)^{3/2}},
\end{align*}
where $p$ and $C$ are defined from \pref{lem:thirdordersmoothfromregularity}, and where $\cN_t$ is the indicator of if noise was added on round $t$.
\end{lemma}
\begin{proof}
In rounds where we add noise, we have $\tri*{ \nabla\Phi(\vecW_t), \nabla F(\vecW_t)} <F_{\epsilon}(\vecW_t)$. By our condition, this implies we must have 
\[ F_{\epsilon}(\vecW_t)+\frac1{\beta} \Delta \Phi(\vecW_t)\le \tri*{ \nabla\Phi(\vecW_t), \nabla F(\vecW_t)}.\]
Thus, in these rounds this result follows immediately from \pref{lem:onesteprecursion}. 

Otherwise if we do not add noise, we have $\tri*{\nabla\Phi(\vecW_t), \nabla F(\vecW_t)} \ge F_{\epsilon}(\vecW_t)$. The proof now is the same as in \citet{priorpaper}. Applying \pref{lem:thirdordersmoothfromregularity} again, this time for the expansion to at most second order, and using that $\theta' \ge 0$, we get from this condition that
\begin{align*}
\theta\prn*{\Phi\prn*{\vecW_{t+1}}} &\le \theta\prn*{\Phi\prn*{\vecW_t}} + \theta'\prn*{\Phi\prn*{\vecW_t}} \tri*{ \nabla \Phi\prn*{\vecW_t}, -\eta \nabla F\prn*{\vecW_t}}+ \frac{\eta^2}2 \nrm*{\nabla F(\vecW_t)}^2 \\
&\le \theta\prn*{\Phi(\vecW_t)} -\eta \theta'\prn*{\Phi(\vecW_t)} F_{\epsilon}(\vecW_t) + \frac{\eta^2}2 \nrm*{\nabla F(\vecW_t)}^2,
\end{align*}
implying the result.
\end{proof}

This result gives us a way to upper bound $\theta'\prn*{\Phi(\vecW_t)}$. To control this, we will again need to control the $\Phi(\vecW_t)$ which we do as follows.
\begin{lemma}\label{lem:actuallypotentialloosenedsimplesetting}
Follow the same notation, assumptions, and choice of $\eta, T$ as in \pref{lem:actuallypotentialsimplesetting}. Then with probability at least $0.9$, we have that
\[ \rho_{\Phi}\prn*{\Phi(\vecW_t)} \le \kappa' \rho_{\Phi}\prn*{\Phi(\vecW_0)}\FORALLTEXT 0 \le t \le T-1.\]
\end{lemma}
\begin{proof}
Note defining $R\prn*{\vecW_j, \Phi, \eta, \beta, d}$ identically as in \pref{lem:actuallypotentialsimplesetting}, we still have from \pref{lem:onesteprecursionloosen} that
\begin{align*}
\mathbb{E}_{\vecEps_t} \brk*{\theta\prn*{\Phi(\vecW_{t+1})}} &\le \theta\prn*{\Phi(\vecW_t)} - \eta  \theta'\prn*{\Phi(\vecW_t)} F_{\epsilon}(\vecW_t)\\
&\hspace{1in}+ \frac12 \eta^2 \nrm*{\nabla F(\vecW_t)}^2+\frac{2C}3 \eta^3\nrm*{\nabla F(\vecW_t)}^3 + 2C(\eta d/\beta)^{3/2}.
\end{align*}
Therefore, \pref{lem:Ytsupermartingalesimplesetting} still holds.

Moreover, following the same proof we see that \pref{lem:azumahoeffding} still holds here too, except now $\vecEps_t$ is 0 when $t \in \cN_t$.

Define $\tau$ analogously as in the proof of \pref{lem:actuallypotentialsimplesetting}. The same derivation as earlier establishes
\begin{align*}
Y_{\tau}-Y_0 &\ge \rho_{\Phi}^{-1}\prn*{\kappa' \rho_{\Phi}\prn*{\Phi(\vecW_0)}}-\Phi(\vecW_0)-\prn*{\frac32 C B_{\textsc{grad}}^3 \eta^2 + 2C (\eta d / \beta)^{3/2}}T.
\end{align*}
Now we use \pref{lem:azumahoeffding}, which still holds there, to upper bound $Y_{\tau}-Y_0$. Denote the event from \pref{lem:azumahoeffding} with $\delta=0.1$ by $E_1$. Conditioned on $E_1$, which occurs with probability at least $0.9$, we have that
\[ Y_{\tau} - Y_0 \le \sqrt{\frac12 \prn*{\sum_{t=0}^{\tau-1} C(\eta, t, d, \beta)^2}\log(10T)}\]
where
\[C(\eta, t, d, \beta) = 4\sqrt{\theta\prn*{\rho_{\Phi}^{-1}\prn*{\kappa' \rho_{\Phi}\prn*{\Phi(\vecW_0)}}}} \cdot \nrm*{-\eta \nabla F(\vecW_t)+\sqrt{\frac{2\eta}{\beta}} \vecEps_t}+ 4\nrm*{-\eta \nabla F(\vecW_t)+\sqrt{\frac{2\eta}{\beta}} \vecEps_t}^2.\]
To control $C(\eta, t, d, \beta, \rho)$, note in the proof of \pref{lem:actuallypotentialsimplesetting}, we used Triangle Inequality or Young's Inequality to split up each of the $\nrm*{-\eta \nabla F(\vecW_t)+\sqrt{2\eta / \beta} \vecEps_t}$ or this quantity squared and isolated the $\nrm*{\vecEps_t}$. Note now that we still have $\nrm*{\vecEps_t}\le\sqrt{d}$ always ($\nrm*{\vecEps_t}=0,\sqrt{d}$), so the same upper bound for $Y_{\tau}-Y_0$ holds. 

Thus the same steps as in the proof of \pref{lem:actuallypotentialsimplesetting}, with the same choice of $T$ and $\eta$, allow us to conclude.
\end{proof}

Now, with these parts in hand, we can prove \pref{thm:discretizationloosenoracleformal}.
\begin{proof}
The proof is nearly identical to the finish of the proof of \pref{thm:discretizationsimpleoracleformal}. Again, we can reduce to proving the main case. The only difference is that we now use \pref{lem:onesteprecursionloosen}, but we can still upper bound the discretization error from that step as $\frac12 \eta^2 \nrm*{\nabla F(\vecW_t)}^2+\frac{2C}3 \eta^3\nrm*{\nabla F(\vecW_t)}^3 + 2C(\eta d/\beta)^{3/2}$. Hence we have the same inequality obtained from telescoping as in the proof of \pref{thm:discretizationsimpleoracleformal}, and the same proof finishes.
\end{proof}

\subsection{Additional Discussion}\label{subsec:loosenconditiondiscussion}
We discuss how \pref{thm:discretizationloosenoracle} implies optimization for P\LCHAR functions and also K\LCHAR functions. In particular, they satisfy this condition \pref{eq:loosenedcondition}: when $\Phi(\vecW)=\lambda F(\vecW)$, we obtain Polyak-\L ojasiewicz (P\L) functions \citep{polyak1963gradient, lojasiewicz1963topological}. 
\begin{definition}[Polyak-\L ojasiewicz (P\L)]
A differentiable function $F$ is Polyak-\L ojasiewicz (P\L) if $\nrm*{\grad F(\vecW)}^2 \ge \lambda F(\vecW)$ for some $\lambda>0$, for all $\vecW\in\mathbb{R}^d$.
\end{definition}
P\LCHAR functions are a classic class of functions for which GD and SGD can be proved to succeed as a strategy for global optimization, but are not necessarily convex. It is well known $\alpha$-strong convexity of $F$ implies that $F$ satisfies the P\LCHAR inequality with parameter $2\alpha$, but \textit{not} vice-versa \citep{chewi21analysis}. Additional examples of P\LCHAR functions have been found in recent literature, for example transformers \citep{zhang2024trained}.

Moreover, \pref{eq:loosenedcondition} can even handle the case of Kurdyka-\L ojasiewicz (K\L) functions \citep{kurdyka1998gradients}, which are a generalization of P\LCHAR functions.
\begin{definition}[Kurdyka-\L ojasiewicz (K\L)]
A differentiable function $F$ is Kurdyka-\L ojasiewicz (K\L) if $\nrm*{\grad F(\vecW)}^2 \ge \lambda F(\vecW)^{1+\theta}$ for some $\lambda>0$ and $\theta \in [0,1)$ for all $\vecW\in\mathbb{R}^d$.
\end{definition}
This can be seen by taking $\Phi(\vecW)=\frac{1}{\lambda(1-\theta)} F(\vecW)^{1-\theta}$ (note \pref{eq:loosenedcondition} can actually handle any $\theta<1$, even if $\theta$ is negative). There are many examples of K\LCHAR functions from generalized linear models \citep{mei2021leveraging} to reinforcement learning \citep{agarwal2021theory, mei2020global, yuan2022general} to over-parametrized nueral networks \citep{zeng2018global, allen2019convergence} to low-rank matrix recovery \citep{bi2022local} to optimal control \citep{bu2019lqr, fatkhullin2021optimizing}. 

In all these cases it is reasonable to assume we have query access to $\tri*{ \nabla \Phi(\vecW), \nabla F(\vecW)} - F_{\epsilon}(\vecW)$, since $\Phi$ solely depends on $F$. Thus we obtain the following corollary:
\begin{corollary}\label{corr:PLloosenoracle}
Suppose $F$ is P\LCHAR/K\LCHAR and satisfies \pref{ass:holderF}, \pref{ass:polyselfbounding}, and \pref{ass:iteratesinballassumption} (the latter two assumptions with $F$ in place of $\Phi$). Then running \pref{alg:sgldmodlaplacian} with constant step size, we can optimize P\LCHAR and K\LCHAR functions to any precision $\epsilon \ge \widetilde{\Omega}\prn*{\frac{d}{\beta}}$ in $\widetilde{O}(\frac1{\epsilon^2})$ iterations.
\end{corollary}
However, note \pref{eq:loosenedcondition} is much looser than P\LCHAR and K\LCHAR functions. P\LCHAR and K\LCHAR functions here correspond to when we \textit{never} add noise in \pref{alg:sgldmodlaplacian}. Thus we believe \pref{eq:loosenedcondition} encompasses many more non-convex optimization problems of interest. 

\section{Proofs for \pref{sec:contributions}}\label{sec:PIoptproofs}
In this section, as with \pref{sec:corediscreteproofs}, we state all guarantees with constant probability. To obtain those results with probability $1-\delta$, one can simply use the standard log-boosting trick.
\subsection{Proofs of \pref{thm:poincareoptlipschitz}, \ref{thm:poincareoptlipschitzestimate}, and \ref{thm:smoothdissipativesettingpoincareopt}}
Here we formally state and prove \pref{thm:poincareoptlipschitz}, \ref{thm:poincareoptlipschitzestimate}, and \ref{thm:smoothdissipativesettingpoincareopt}, which are all subsumed by the following result.
\begin{theorem}\label{thm:poincareoptformal}
Suppose that $F$ satisfies \pref{ass:holderF} and \pref{ass:weakdissipation}. 
Suppose $\mu_{\beta}$ has second moment $S<\infty$ and satisfies a Poincar\'e Inequality with constant $\CPI(\mu_{\beta})$ with $\beta = \widetilde{\Theta}\prn*{\frac{d}{\epsilon}}$, namely
\[ \epsilon \ge \frac{2d}{\beta} \log(4\pi e \beta L d S).\] 
Suppose $\Phi$ (from \pref{thm:maingeometriccondition}) satisfies \pref{ass:selfboundingPhipoly} with $0 \le p\le 1$, and define $\rho_{\Phi} = \max\prn*{\rho_{\Phi,1}, \rho_{\Phi,2}, \rho_{\Phi,3}}$. We can assume without loss of generality that $\rho_{\Phi}(z)=A(z+1)^p$ for some constant $A>0$.

Then consider running either \GLDTEXT, or \SGLDTEXTSPACE using a stochastic gradient oracle $\grad f$ satisfying \pref{ass:gradnoiseassumption} and \pref{ass:stochasticgradcontrol}, with constant step size for $T$ iterations. 
We will reach a $\vecW$ in $\{\vecW:F(\vecW) \le \epsilon\}$ with probability at least 0.8 in at most $T$ gradient (for \GLDTEXT) or stochastic gradient (for \SGLDTEXT) evaluations respectively, where we set
\begin{align*}
T &\le 8^3 C_0 \max\crl*{1, \frac{4L^2}m, \frac{4\max(L,B)}{m}, \frac{4B^2}m, 6B, 120^2 A^2 B^2 M^2, 120 AC_0 M}\\
&\hspace{1in}\cdot \max\crl*{\beta \max\prn*{\CPI(\mu_{\beta}),2}, d^3 \max\prn*{\CPI(\mu_{\beta}),2}^3, \beta^{2+s/2}\max\prn*{\CPI(\mu_{\beta}),2}^{2+s/2}}.
\end{align*}
(An explicit expression can be found in our proof.)

Here we define the above constants as follows:
\[ L_2 := \prn*{\nrm*{\vecW_0}^4+8\prn*{\frac{4\prn*{m+b+\frac{4d+2}{\beta}}}{m \wedge 1}}^{\frac{1+\gamma}{\gamma} \lor 2}}^{s/2}, L_3 := \prn*{\nrm*{\vecW_0}^4+8\prn*{\frac{4\prn*{m+b+\frac{4d+2}{\beta}}}{m \wedge 1}}^{\frac{1+\gamma}{\gamma} \lor 2}}^{3s/4},\]
\[ B = \max\prn*{L\max(1,\nrm{\vecW^{\star}}), \sigma_F}, C_0 = 50A\theta(\Phi(\vecW_0)) \lor 1, C = \frac{4A^2 p(p+1)+2Ap+1}3,\]
\[ M = \max\prn*{\frac12, 2C} \cdot \prn*{8\sigma_F^3+16\max(L,B)^3 (\max\prn*{L_2,L_3}+1)}.\]
Here $\theta = \frac1{\rho_{\Phi}}$, as defined in \pref{lem:thirdordersmoothfromregularity}. (Take $L \leftarrow \max(1,L)$, $\sigma_F \leftarrow \max(\sigma_F,1)$ if necessary.)

Moreover, this generalizes to $s=0,1$ as follows:
\begin{enumerate}
    \item In the case when $s=0$, this result holds with no dependence on $L_2, L_3$ and instead we have
\[ M = \max\prn*{\frac12, 2C} \cdot \prn*{8\sigma_F^3+16\max(L,B)^3}.\]
    \item In the case when $s=1$, we no longer need to make assumptions on $\mu_{\beta}$: as $2s \le \gamma \le s+1$, $s=1$ forces $\gamma=1$, the setting of $F$ being $L$-smooth and $(m,b)$ dissipative from \citet{raginsky2017non}, \citet{xu2018global}, and \citet{zou2021faster}. As shown in \citet{raginsky2017non}, these conditions imply $\mu_{\beta}$ satisfies a Poincar\'e Inequality for $\beta \ge \frac{2}m$, and also that $\mu_{\beta}$ has finite second moment $S\le \frac{b+d/\beta}m$.

    Moreover, our guarantees improve as follows. Instead letting
\[ L_2 := \nrm*{\vecW_0}^2+\frac2m\prn*{b+ 2B^2+ \frac{d}{\beta}}, L_3 := \prn*{\nrm*{\vecW_{0}}^4 + C'' \lor \frac{2C''}{m}}^{3/4},\]
where
\[ C'' = \frac{4}{m}\prn*{2C'^2\prn*{4+\frac1m} \lor \frac1m\prn*{3mB+C'}^2}, C'=m+b+\frac{4d+2}{\beta},\]
we have a runtime guarantee of 
\begin{align*}
T &\le 8^3 C_0 \max\crl*{1, \frac{4L^2}m, \frac{4\max(L,B)}{m}, \frac{4B^2}m, 6B, 120^2 A^2 B^2 M^2, 120 AC_0 M}\\
&\hspace{1in}\cdot \max\crl*{\beta \max\prn*{\CPI(\mu_{\beta}),2}, d^3 \max\prn*{\CPI(\mu_{\beta}),2}^3, \beta^{2}\max\prn*{\CPI(\mu_{\beta}),2}^{2}}.
\end{align*}
\end{enumerate}
\end{theorem}
First, note from our assumption that $\epsilon \ge \frac{2d}{\beta} \log(4\pi e \beta L d S)$, we may apply \pref{thm:maingeometriccondition} (in particular, by combining \pref{thm:maingeometricconditionformal}, \pref{lem:measureoflargeF}) to obtain $\Phi$ satisfying the properties described in \pref{thm:maingeometriccondition}. 

Next, we show a Lemma showing the iterates of GLD and SGLD are controlled. We will need this only when $s > 0$. Similar results have been shown in \cite{raginsky2017non} and \cite{balasubramanian2022towards}.
\begin{lemma}\label{lem:smoothdissipativeinball}
Suppose $F$ satisfies \pref{ass:holderF} and \pref{ass:weakdissipation}. Consider the $\{\vecW_t\}_{t \ge 0}$ generated by \GLDTEXTSPACE / \SGLDTEXTSPACE (for \SGLDTEXTSPACE we need \pref{ass:stochasticgradcontrol}), run for $T$ iterations for $T<\infty$ (we only use this for the $T$ we set later). If the step size $\eta \in (0, 1 \wedge \frac{m}{4L^2} \wedge \frac{m}{4\max(L,B)} \wedge \frac{m}{4B^2} \wedge \frac{1}{6B})$, then we have the following bounds:
\[ \mathbb{E}\brk*{\nrm*{\vecW_t}^{2s}} \le L_2\max(\eta T, 1)^{s/2}, \mathbb{E}\brk*{\nrm*{\vecW_t}^{3s}}\le L_3\max(\eta T, 1)^{3s/4},\]
where we define
\[ L_2 := \prn*{\nrm*{\vecW_0}^4+8\prn*{\frac{4\prn*{m+b+\frac{4d+2}{\beta}}}{m \wedge 1}}^{\frac{1+\gamma}{\gamma} \lor 2}}^{s/2}, L_3 := \prn*{\nrm*{\vecW_0}^4+8\prn*{\frac{4\prn*{m+b+\frac{4d+2}{\beta}}}{m \wedge 1}}^{\frac{1+\gamma}{\gamma} \lor 2}}^{3s/4}.\]
Here $B = \max\prn*{L\max(1,\nrm{\vecW^{\star}}), \sigma_F}$, where $\sigma_F$ comes from \pref{ass:gradnoiseassumption}. (Recall we took $L\leftarrow\max\prn*{L,1}$ if necessary earlier in the statement of \pref{thm:poincareoptformal}.)

Moreover, if $s=1$ (which implies $F$ is $L$-smooth and $(m,b)$ dissipative), we have the following uniform bounds:
\[ \mathbb{E}\brk*{\nrm*{\vecW_{t}}^2} \le L_2, \mathbb{E}\brk*{\nrm*{\vecW_{t}}^3} \le L_3 \]
for
\[ L_2 := \nrm*{\vecW_0}^2+\frac2m\prn*{b+ 2B^2+ \frac{d}{\beta}}, L_3 := \prn*{\nrm*{\vecW_{0}}^4 + C'' \lor \frac{2C''}{m}}^{3/4}\]
where
\[ C'' = \frac{4}{m}\prn*{2C'^2\prn*{4+\frac1m} \lor \frac1m\prn*{3mB+C'}^2}, C'=m+b+\frac{4d+2}{\beta}.\]
\end{lemma}
\begin{proof}
Our goal is to use Proposition 14 of \citet{balasubramanian2022towards} to control the second and fourth moments of the $\nrm*{\vecW_t}$. Intuitively, our result should be the same as theirs except their $V$ is replaced with $\beta F$, and then the relevant parameters change (except for $d$, the rest of them are all scaled by $\beta$). However, this gives some unnecessary $\beta$ dependence which arises for technical reasons in their analysis (intuitively, they should cancel), so we need to modify their proof slightly to improve this dependence.

As done in the sampling literature \citep{chewi2024log}, define the continuous-time interpolation of \pref{eq:SGLDiterates} by 
\[ \vecW_r = \vecW_{t} - (r-t \eta) \grad F(\vecW_{t}) + \sqrt{\frac{2}{\beta}}\prn*{\vecB(r)-\vecB(t \eta)} \FORALLTEXT r \in [t \eta, (t+1)\eta).\]
This appears somewhat different than the interpolation defined in the literature, but it is actually the same. Our process \pref{eq:SGLDiterates} with step size $\eta$ is equivalent to theirs with their $V=\beta F$ and their step size $h=\frac{\eta}{\beta}$. They index by `time' where the subscript $th$ corresponds to the $t$-th iterate whereas we index iterates simply by the iteration count (which is at `time' $t \eta$ in the above interpolation) and are indexing time by $r$ to avoid confusion.\footnote{Using this correspondence one can actually carefully track the proof of Proposition 14 of \citet{balasubramanian2022towards} to show a similar result to what we show here.}

For the stochastic gradient case, this will be instead
\[ \vecW_r = \vecW_{t} - (r-t \eta) \grad f(\vecW_{t};\vecZ_t) + \sqrt{\frac{2}{\beta}}\prn*{\vecB(r)-\vecB(t \eta)} \FORALLTEXT r \in [t \eta, (t+1)\eta).\]
Note for both these interpolations, all functions of quantities at time $t\eta$/iteration count $t$ are constant (including $\vecZ_t$), the randomness being over the Brownian motion $\vecB(r)-\vecB(t \eta)$. 

We will do the proofs in the stochastic gradient case, and the proofs in the exact gradient case are the exact same. 

First, we control the second moment. Let $\mathfrak{F}_t$ be defined identically as in \pref{sec:corediscreteproofs}. Analogously to the proof of Proposition 14 of \cite{balasubramanian2022towards}, It\^{o}'s Lemma applied to $\nrm*{\vecW}^2$ conditioned on $\mathfrak{F}_t$ yields for all $r \in [t\eta, (t+1)\eta]$,
\begin{align*}
\frac{\DERIV}{\DERIV r} \mathbb{E}\brk*{\nrm*{\vecW_r}^2 |\mathfrak{F}_t} &= 2\mathbb{E}\brk*{\tri*{\vecW_r,-\grad f(\vecW_t;\vecZ_t)}|\mathfrak{F}_t} + \frac12 \cdot \sqrt{\frac{2}{\beta}}^2 \cdot 2\text{tr}(\bbI_d) \\
&= -2\mathbb{E}\brk*{\tri*{\vecW_t-(r-t \eta) \grad f(\vecW_t;\vecZ_t)+\sqrt{\frac2{\beta}}\prn*{\vecB(r)-\vecB(t\eta)}, \grad f(\vecW_t;\vecZ_t)}|\mathfrak{F}_t}+ \frac{2d}{\beta} \\
&= -2\mathbb{E}\brk*{\tri*{\vecW_t-(r-t \eta) \grad f(\vecW_t;\vecZ_t), \grad f(\vecW_t;\vecZ_t}|\mathfrak{F}_t}+ \frac{2d}{\beta} \\
&\le 2b-2m\nrm*{\vecW_t}^{\gamma}+2(r-t \eta) \nrm*{\grad f(\vecW_t;\vecZ_t)}^2+ \frac{2d}{\beta} \\
&\le 2b-2m\nrm*{\vecW_t}^{\gamma}+4\eta \cdot L^2 \max(1,\nrm*{\vecW^{\star}})^{2s} \prn*{\nrm*{\vecW_t}^{2s}+1} + \frac{2d}{\beta} \\
&\le 4m+2b+\frac{2d}{\beta}. \numberthis\label{eq:secondmomentitobound}
\end{align*}
In the above we use that $\mathbb{E}\brk*{\tri*{\prn*{\vecB(r)-\vecB(t\eta)}, \grad f(\vecW_t;\vecZ_t)}|\mathfrak{F}_t}=0$, \pref{ass:stochasticgradcontrol}, \pref{lem:upperboundgradFholder}, $\gamma \ge 2s$, $\eta \le \frac{m}{2B^2}$, and $r-t \eta \le \eta$. Integrating this over $r \in [t\eta, (t+1)\eta]$ and iterating yields 
\[ \mathbb{E}\brk*{\nrm*{\vecW_t}^2} \le \nrm*{\vecW_0}^2+\prn*{4m+2b+\frac{2d}{\beta}} \cdot \eta t \le \prn*{\nrm*{\vecW_0}^2+4m+2b+\frac{2d}{\beta}} \max(\eta T, 1).\]
We now control the fourth moment with the same idea. Applying It\^{o}'s Lemma to $\nrm*{\vecW}^4=\prn*{\nrm*{\vecW}^2}^2$, we obtain
\begin{align*}
&\frac{\DERIV}{\DERIV r} \mathbb{E}\brk*{\nrm*{\vecW_r}^4 |\mathfrak{F}_t} \\
&= -4\mathbb{E}\brk*{\nrm*{\vecW_r}^2\tri*{\vecW_r,\grad f(\vecW_t;\vecZ_t)}|\mathfrak{F}_t} + \frac12 \cdot \sqrt{\frac{2}{\beta}}^2 \cdot (4d+2) \mathbb{E}\brk*{\nrm*{\vecW_r}^2 | \mathfrak{F}_t} \\
&= -4\mathbb{E}\brk*{\nrm*{\vecW_r}^2\tri*{\vecW_t-(r-t \eta) \grad f(\vecW_t;\vecZ_t)+\sqrt{\frac2{\beta}}\prn*{\vecB(r)-\vecB(t\eta)},\grad f(\vecW_t;\vecZ_t)}|\mathfrak{F}_t} + \frac{4d+2}{\beta} \mathbb{E}\brk*{\nrm*{\vecW_r}^2 | \mathfrak{F}_t}.
\end{align*}
Let $\mathbf{x} = \frac{\vecB(r)-\vecB(t\eta)}{\sqrt{r-t\eta}}$ be a standard Gaussian vector. Using Gaussian Integration by Parts on $h(\mathbf{x}) = \nrm*{\vecW_t-(r-t\eta) \grad f(\vecW_t;\vecZ_t)+\sqrt{\frac{2}{\beta}} \cdot \sqrt{r-t\eta} \mathbf{x}}^2=\nrm*{\vecW_r}^2$, we have 
\begin{align*}
&\mathbb{E}\brk*{\nrm*{\vecW_r}^2 \tri*{\sqrt{\frac2{\beta}}\prn*{\vecB(r)-\vecB(t\eta)},\grad f(\vecW_t;\vecZ_t)}|\mathfrak{F}_t} \\
&= \sqrt{\frac2{\beta}} \cdot \sqrt{r-t\eta} \cdot \tri*{\mathbb{E}\brk*{\mathbf{x} h(\mathbf{x}) |\mathfrak{F}_t}, \grad f(\vecW_t;\vecZ_t)} \\
&= \sqrt{\frac2{\beta}} \cdot \sqrt{r-t\eta} \cdot \tri*{\mathbb{E}\brk*{\grad h(\mathbf{x})|\mathfrak{F}_t}, \grad f(\vecW_t;\vecZ_t)} \\
&= \frac{4}{\beta}(r-t\eta)\tri*{\vecW_t - (r-t \eta)\grad f(\vecW_t;\vecZ_t), \grad f(\vecW_t;\vecZ_t)}.
\end{align*}
The above follows since $\grad h(\mathbf{x}) = \sqrt{\frac2{\beta}} \cdot \sqrt{r-t\eta} \cdot \prn*{\vecW_t-(r-t\eta) \grad f(\vecW_t;\vecZ_t)+\sqrt{\frac{2}{\beta}} \cdot \sqrt{r-t\eta} \mathbf{x}}$ and as $\mathbf{x}$ is independent of $\mathfrak{F}_t$ and has mean of the 0 vector.

Hence, we have for all $r \in [t\eta, (t+1)\eta]$,
\begin{align*}
\frac{\DERIV}{\DERIV r} \mathbb{E}\brk*{\nrm*{\vecW_r}^4 |\mathfrak{F}_t} &=4\mathbb{E}\brk*{\nrm*{\vecW_r}^2 | \mathfrak{F}_t} \prn*{-\tri*{\vecW_t,\grad f(\vecW_t;\vecZ_t)}+(r-t\eta) \nrm*{\grad f(\vecW_t;\vecZ_t)}^2 + \frac{4d+2}{\beta}} \\
&\hspace{1in}- \frac{16}{\beta}(r-t\eta)\tri*{\vecW_t - (r-t \eta)\grad f(\vecW_t;\vecZ_t), \grad f(\vecW_t;\vecZ_t)} \\
&\le 4 \prn*{\mathbb{E}\brk*{\nrm*{\vecW_r}^2 | \mathfrak{F}_t}+\frac{4}{\beta}(r-t\eta)} \prn*{-\tri*{\vecW_t,\grad f(\vecW_t;\vecZ_t)}+(r-t\eta) \nrm*{\grad f(\vecW_t;\vecZ_t)}^2 + \frac{4d+2}{\beta}} \\
&\le 4 \prn*{\mathbb{E}\brk*{\nrm*{\vecW_r}^2 | \mathfrak{F}_t}+\frac{4}{\beta}(r-t\eta)} \prn*{-m\nrm*{\vecW_t}^{\gamma} + b+2\eta \cdot L^2 \max(1,\nrm*{\vecW^{\star}})^{2s} \prn*{\nrm*{\vecW_t}^{2s}+1} + \frac{4d+2}{\beta}} \\
&\le 4 \prn*{\mathbb{E}\brk*{\nrm*{\vecW_r}^2 | \mathfrak{F}_t}+\frac{4}{\beta}(r-t\eta)}\prn*{-\frac{m}2\nrm*{\vecW_t}^{\gamma}+m+b+\frac{4d+2}{\beta}}.\numberthis\label{eq:fourthmomentitobound}
\end{align*}
The above follows as $r \ge t\eta$ and so the first factor in the above is always non-negative, as well as $\eta \le \frac{m}{4B^2}$ and $\gamma \ge 2s$.

Define $C' := m+b+\frac{4d+2}{\beta}$ for convenience. If $\nrm*{\vecW_t} \ge \prn*{\frac{2C'}{m}}^{1/\gamma}$, this means $\frac{\DERIV}{\DERIV r} \mathbb{E}\brk*{\nrm*{\vecW_r}^4 |\mathfrak{F}_t} \le 0$. Otherwise if $\nrm*{\vecW_t} \le \prn*{\frac{2C'}{m}}^{1/\gamma}$, using our upper bound on $\frac{\DERIV}{\DERIV r} \mathbb{E}\brk*{\nrm*{\vecW_r}^2 |\mathfrak{F}_t}$ gives
\[ \mathbb{E}\brk*{\nrm*{\vecW_r}^2 | \mathfrak{F}_t}+\frac{4}{\beta}(r-t\eta) \le \nrm*{\vecW_t}^2+\prn*{4m+2b+\frac{2d}{\beta}} (r-t\eta)+\frac{4}{\beta}(r-t\eta) \le \prn*{\frac{2C'}{m}}^{1/\gamma} + 4C',\]
as $r-t\eta \le \eta \le 1$. Note $\nrm*{\vecW_t} \le \prn*{\frac{2C'}{m}}^{1/\gamma}$ implies the second factor in \pref{eq:fourthmomentitobound} is non-negative, and the second factor is at most $C'$ clearly. Thus, in this case we have
\[ \frac{\DERIV}{\DERIV r} \mathbb{E}\brk*{\nrm*{\vecW_r}^4 |\mathfrak{F}_t} \le 4C'\prn*{\prn*{\frac{2C'}{m}}^{1/\gamma} + 4C'} \le 8\prn*{\frac{4C'}{m \wedge 1}}^{\frac{1+\gamma}{\gamma} \lor 2}. \]
Hence the above is an upper bound on $\frac{\DERIV}{\DERIV r} \mathbb{E}\brk*{\nrm*{\vecW_r}^4 |\mathfrak{F}_t}$ in all cases, and iterating this gives the desired fourth moment bound
\[ \mathbb{E}\brk*{\nrm*{\vecW_t}^4} \le \nrm*{\vecW_0}^4+8\prn*{\frac{4C'}{m \wedge 1}}^{\frac{1+\gamma}{\gamma} \lor 2} \eta t \le \prn*{\nrm*{\vecW_0}^4+8\prn*{\frac{4C'}{m \wedge 1}}^{\frac{1+\gamma}{\gamma} \lor 2}}\max(\eta T, 1).\]
From here, to obtain the desired conclusion, use monotonicity of moments (as $s \le 1$):
\[ \mathbb{E}\brk*{\nrm*{\vecW_t}^{2s}} \le \mathbb{E}\brk*{\nrm*{\vecW_t}^{4}}^{2s/4} \le \prn*{\nrm*{\vecW_0}^4+8\prn*{\frac{4C'}{m \wedge 1}}^{\frac{1+\gamma}{\gamma} \lor 2}}^{s/2}\max(\eta T, 1)^{s/2}.\]
\[ \mathbb{E}\brk*{\nrm*{\vecW_t}^{3s}}\le \mathbb{E}\brk*{\nrm*{\vecW_t}^{4}}^{3s/4} \le \prn*{\nrm*{\vecW_0}^4+8\prn*{\frac{4C'}{m \wedge 1}}^{\frac{1+\gamma}{\gamma} \lor 2}}^{3s/4}\max(\eta T, 1)^{3s/4}.\]
When $s=1$ and hence $\gamma=2$ (which implies $(m,b)$ dissipativeness), we can be tighter in the above analysis. First, using \pref{lem:stochasticgradsetting}, we have
\[ \mathbb{E}_{\vecZ_t}\brk*{\nrm*{\grad f(\vecW_t;\vecZ_t)-\grad F(\vecW_t)}^2} \le \sigma_F^2 \le B^2.\]
With the above, identically as the steps of the proof of Lemma 3 of \citet{raginsky2017non}, using $(m,b)$ dissipativeness and our constant upper bound on $\eta$, we can show a uniform bound on the second moment for both exact and stochastic gradients:
\[ \mathbb{E}\brk*{\nrm*{\vecW_{t}}^2} \le \nrm*{\vecW_0}^2+\frac2m\prn*{b+ 2B^2+ \frac{d}{\beta}}.\]
We also claim we have a uniform upper bound on the fourth moment. We break into two cases, both using a similar strategy. 
\begin{enumerate}
    \item $\nrm*{\vecW_t} \le \prn*{\frac{2C'}{m}}^{1/2}$: In this case, the second factor in \pref{eq:fourthmomentitobound} is non-negative. Recall the upper bound we showed from \pref{eq:secondmomentitobound}:
\[ \frac{\DERIV}{\DERIV r} \mathbb{E}\brk*{\nrm*{\vecW_r}^2 |\mathfrak{F}_t} \le 4m+2b+\frac{2d}{\beta}. \]
This implies
\[ \mathbb{E}\brk*{\nrm*{\vecW_r}^2 | \mathfrak{F}_t} \le \nrm*{\vecW_t}^2 + \prn*{4m+2b+\frac{2d}{\beta}}(r-t\eta).\]
Now as the second factor in \pref{eq:fourthmomentitobound} is non-negative, we obtain using $r-t\eta \le \eta \le1$,
\begin{align*}
\frac{\DERIV}{\DERIV r} \mathbb{E}\brk*{\nrm*{\vecW_r}^4 |\mathfrak{F}_t} &\le 4 \prn*{\mathbb{E}\brk*{\nrm*{\vecW_r}^2 | \mathfrak{F}_t}+\frac{4}{\beta}(r-t\eta)}\prn*{-\frac{m}2\nrm*{\vecW_t}^{2}+m+b+\frac{4d+2}{\beta}} \\
&\le 4\prn*{\nrm*{\vecW_t}^2+\prn*{4m+2b+\frac{2d}{\beta}} (r-t\eta)+\frac{4}{\beta}(r-t\eta)}\prn*{-\frac{m}2\nrm*{\vecW_t}^{2}+m+b+\frac{4d+2}{\beta}}\\
&\le 4\prn*{\nrm*{\vecW_t}^2+4C'}\prn*{-\frac{m}2\nrm*{\vecW_t}^{2}+C'} \\
&\le 4\prn*{-\frac{m}2\nrm*{\vecW_t}^4 + C' \nrm*{\vecW_t}^2+4C'^2} \\
&\le 4\prn*{-\frac{m}4\nrm*{\vecW_t}^4+C'^2\prn*{4+\frac{1}m}} = -m\nrm*{\vecW_t}^4+4C'^2\prn*{4+\frac{1}m}.
\end{align*}
The last step uses AM-GM.
\item $\nrm*{\vecW_t} > \prn*{\frac{2C'}{m}}^{1/2}$: This time, the second factor in \pref{eq:fourthmomentitobound} is negative, so we aim to lower bound $\mathbb{E}\brk*{\nrm*{\vecW_r}^2 | \mathfrak{F}_t}$. Recalling the intermediate steps in \pref{eq:secondmomentitobound}, we have 
\begin{align*}
\frac{\DERIV}{\DERIV r} \mathbb{E}\brk*{\nrm*{\vecW_r}^2 |\mathfrak{F}_t} &= -2\mathbb{E}\brk*{\tri*{\vecW_t-(r-t \eta) \grad f(\vecW_t;\vecZ_t), \grad f(\vecW_t;\vecZ_t}|\mathfrak{F}_t}+ \frac{2d}{\beta} \\
&\ge -2\tri*{\vecW_t, \grad f(\vecW_t;\vecZ_t)} \\
&\ge -2\nrm*{\vecW_t}\nrm*{\grad f(\vecW_t;\vecZ_t)} \\
&\ge -2B\nrm*{\vecW_t} \prn*{\nrm*{\vecW_t}+1} \\
&\ge -3B\prn*{\nrm*{\vecW_t}^2+1},
\end{align*}
where we upper bound $\grad f(\vecW_t;\vecZ_t)$ via \pref{lem:upperboundgradFholder} and use AM-GM in the last step. 

This implies
\[ \mathbb{E}\brk*{\nrm*{\vecW_r}^2 | \mathfrak{F}_t} \ge \nrm*{\vecW_t}^2 -  3B\prn*{\nrm*{\vecW_t}^2+1}(r-t\eta) \ge \frac12\nrm*{\vecW_t}^2-3B,\]
since $r-t\eta \le \eta$, $\eta \le \frac1{6B}\le 1$.

The second factor in \pref{eq:fourthmomentitobound} is negative, so we may apply this in \pref{eq:fourthmomentitobound} to give
\begin{align*}
\frac{\DERIV}{\DERIV r} \mathbb{E}\brk*{\nrm*{\vecW_r}^4 |\mathfrak{F}_t} &\le 4 \prn*{\mathbb{E}\brk*{\nrm*{\vecW_r}^2 | \mathfrak{F}_t}+\frac{4}{\beta}(r-t\eta)}\prn*{-\frac{m}2\nrm*{\vecW_t}^{2}+m+b+\frac{4d+2}{\beta}} \\
&\le 4\prn*{\frac12\nrm*{\vecW_t}^2-3B}\prn*{-\frac{m}2\nrm*{\vecW_t}^{2}+C'} \\
&= 4\prn*{-\frac{m}4\nrm*{\vecW_t}^4 + \prn*{\frac{3mB}2 + \frac{C'}2}\nrm*{\vecW_t}^2 - 3BC'} \\
&\le 4\prn*{-\frac{m}8 \nrm*{\vecW_t}^4 + \frac{(3mB+C')^2}{2m}} = -\frac{m}2 \nrm*{\vecW_t}^4 + \frac2{m}\cdot (3mB+C')^2.
\end{align*}
Again, the last step uses AM-GM.
\end{enumerate}
From the above we see that in either case we have
\[ \frac{\DERIV}{\DERIV r} \mathbb{E}\brk*{\nrm*{\vecW_r}^4 |\mathfrak{F}_t} \le -\frac{m}2 \nrm*{\vecW_t}^4 + C''\]
where $C'' = 4C'^2\prn*{4+\frac{1}m} \lor \frac2{m}\cdot (3mB+C')^2$. 

Iterating the above for one step and then taking full expectation yields the recursion
\[ \mathbb{E}\brk*{\nrm*{\vecW_{t+1}}^4 } \le \mathbb{E}\brk*{\nrm*{\vecW_{t}}^4 }+\eta\prn*{-\frac{m}2\mathbb{E}\brk*{\nrm*{\vecW_t}^4}+C''} = \prn*{1-\frac{\eta m}2}\mathbb{E}\brk*{\nrm*{\vecW_{t}}^4 }+\eta C''.\]
If $1-\frac{\eta m}2 \le 0$ we obtain $\mathbb{E}\brk*{\nrm*{\vecW_{t}}^4 } \le C''$, and otherwise if $1-\frac{\eta m}2 \in (0,1)$, iterating the above and summing the resulting geometric series gives 
\[ \mathbb{E}\brk*{\nrm*{\vecW_{t}}^4 } \le \prn*{1-\frac{\eta m}2}^t\nrm*{\vecW_{0}}^4+\frac{2\eta C''}{\eta m} \le \nrm*{\vecW_{0}}^4 + \frac{2C''}{m}.\] 
The desired upper bound on the third moment in this case now just comes from monotonicity of moments.
\end{proof}

We now are ready to prove \pref{thm:poincareoptformal}. We do the proof when $0<s \le 1$ (when $s>0$, $\gamma\ge 2s>0$ so we can certainly use \pref{lem:smoothdissipativeinball}), and we discuss the simple extension to $s=0$ and the tighter results when $s=1$ at the end.
\begin{proof}
Consider $\theta$ and $C = \frac{Ap^2+4Ap+1}6$ defined in terms of $\rho_{\Phi}$ in \pref{lem:thirdordersmoothfromregularity} for the $p \le 1$ case. 

We set 
\[ C_0 = 50A\theta(\Phi(\vecW_0)) \lor 1, M = \max\prn*{\frac12, 2C} \cdot \prn*{8\sigma_F^3+16\max(L,B)^3 (\max\prn*{L_2,L_3}+1)},\]
\[ \eta = \min\prn*{1,\frac{m}{4L^2}, \frac{m}{4\max(L,B)}, \frac{m}{4B^2},\frac{1}{6B},\frac{1}{120^2 A^2 B^2 M^2} \cdot \frac{\beta^3 \lambda^2}{d^3}, \frac{\lambda^{1+s/2}}{120AC_0 M}},\]
\[ T = \frac{C_0}{\eta \lambda}.\]
Here $\lambda \in \brk*{\frac1{8\beta} \min\prn*{\frac1{{\CPI}(\mu_{\beta})}, \frac12}, \frac1{4\beta} \min\prn*{\frac1{{\CPI}(\mu_{\beta})}, \frac12}}$, as with $\Phi$, comes from \pref{thm:maingeometriccondition}. Thus, using 
\[ T = C_0 \max\crl*{\frac1{\lambda}\max\prn*{1, \frac{4L^2}m, \frac{4\max(L,B)}m, \frac{4B^2}m, 6B}, 120^2 A^2 B^2 M^2 \cdot \frac{d^3}{\beta^3 \lambda^3}, 120AC_0 M \cdot \frac{1}{\lambda^{2+s/2}}} \]
and
\[ \frac1{\lambda} \le 8\beta \max\prn*{\CPI(\mu_{\beta}), 2},\]
we see that our definition of $T$ above is consistent with the statement of \pref{thm:poincareoptformal}. Moreover, note $\eta T=\frac{C_0}{\lambda} \ge 1$. 

As with before let $\mathfrak{F}_t$ be the natural filtration with respect to $\vecEps_{t'},\vecZ_{t'}$ for all $0 \le t' \le t$ in the SGLD case, and with respect to $\vecEps_{t'}$ for all $0 \le t' \le t$ in the GLD case.   

Define
\[ \tau_{\cA_{\epsilon}, T}(\vecW_0) = \min\prn*{\tau_{\cA_{\epsilon}}(\vecW_0), T},\]
where in a slight abuse of notation, $\tau_{\cA_{\epsilon}}$ now denotes the hitting time of discrete-time \GLDTEXT/\SGLDTEXTSPACE to $\cA_{\epsilon}$ with the choice of $\eta$ above. Note $\tau_{\cA_{\epsilon}, T}(\vecW_0)$ is a stopping time that is at most $T<\infty$.

Consider $\vecW_t$ for $t < \tau_{\cA_{\epsilon}, T}(\vecW_0)$, thus $\vecW_t \in \cA_{\epsilon}^c$. By \pref{thm:maingeometriccondition}, this implies for this $\vecW_t$, \pref{eq:poincaregeomconditioninformal} holds:
\[ \tri*{\grad F(\vecW), \grad \Phi(\vecW)} \ge\lambda\Phi(\vecW) + \frac1{\beta} \Delta\Phi(\vecW).\]
Recall $\theta' > 0$ from \pref{lem:thirdordersmoothfromregularity}, including in this case where $p \le 1$. Analogously to the proof of \pref{lem:onesteprecursion}, and using the geometric condition \pref{eq:lyapunovbaby}, we obtain
\begin{align*}
\mathbb{E}_{\vecEps_t,\vecZ_t} \brk*{\theta(\Phi(\vecW_{t+1}))} &\le \theta\prn*{\Phi(\vecW_t)} - \eta  \theta'\prn*{\Phi(\vecW_t)} \cdot \lambda \Phi(\vecW_t)\\
&\hspace{1in}+ \frac12 \eta^2 \nrm*{\grad F(\vecW_t)}^2+C \eta^3 \nrm*{\grad F(\vecW_t)}^3 + 2C\prn*{\frac{\eta d}{\beta}}^{3/2}.
\end{align*}
This uses \pref{lem:thirdordersmoothfromregularity} in the $p \le 1$ case. 

In the stochastic gradient case we analogously have via the same logic as \pref{lem:onesteprecursionstochastic} that
\begin{align*}
\mathbb{E}_{\vecEps_t,\vecZ_t} \brk*{\theta(\Phi(\vecW_{t+1}))} &\le \theta\prn*{\Phi(\vecW_t)} - \eta  \theta'\prn*{\Phi(\vecW_t)} \cdot \lambda \Phi(\vecW_t)\\
&\hspace{1in}+ \frac12 \eta^2 \mathbb{E}_{\vecZ_t}\brk*{\nrm*{\grad f(\vecW_t;\vecZ_t)}^2}+C \eta^3 \mathbb{E}_{\vecZ_t}\brk*{\nrm*{\grad f(\vecW_t;\vecZ_t)}^3} \\
&\hspace{1in}+ 2C\prn*{\frac{\eta d}{\beta}}^{3/2}.
\end{align*}
(Note these results can be proved for either $\vecEps_t \sim \cS^{d-1}$ or $\vecEps_t \sim \mathcal{N}(0,\bbI_d)$ by the exact same proof as \pref{lem:onesteprecursion}, \pref{lem:onesteprecursionstochastic}.)

Applying \pref{lem:stochasticgradsetting} and then \pref{lem:upperboundgradFholder}, Young's Inequality, and $\nrm*{\mathbf{a}+\mathbf{b}}^3 \le 4\nrm*{\mathbf{a}}^3+4\nrm*{\mathbf{b}}^3$, and noting $\sigma_F \ge 0$, we see in both the \GLDTEXTSPACE and \SGLDTEXTSPACE cases that 
\begin{align*}
\mathbb{E}_{\vecEps_t,\vecZ_t} \brk*{\theta(\Phi(\vecW_{t+1}))} &\le \theta\prn*{\Phi(\vecW_t)} - \eta  \theta'\prn*{\Phi(\vecW_t)}\cdot \lambda \Phi(\vecW_t)\\
&\hspace{1in}+ \frac12 \eta^2 \prn*{2\sigma_F^2+2\nrm*{\grad F(\vecW_t)}^2}+C \eta^3 \prn*{8\sigma_F^3+4\nrm*{\grad F(\vecW_t}^3}\\
&\hspace{1in}+ 2C\prn*{\frac{\eta d}{\beta}}^{3/2} \\
&\le \theta\prn*{\Phi(\vecW_t)} - \eta  \theta'\prn*{\Phi(\vecW_t)} \cdot \lambda \Phi(\vecW_t)\\
&\hspace{1in}+ \frac12 \eta^2 \prn*{2\sigma_F^2+4\max(L,B)^2 \prn*{\nrm*{\vecW_t}^{2s}+1}}\\
&\hspace{1in}+C \eta^3 \prn*{8\sigma_F^3+16\max(L,B)^3 \prn*{\nrm*{\vecW_t}^{3s}+1}} + 2C\prn*{\frac{\eta d}{\beta}}^{3/2}.
\end{align*}
Recall that $\theta'(z)=\frac{1}{A(z+1)^p}$ where $p \le 1$, which is increasing on $z \ge 0$. Therefore, $z \theta'(z) = \frac{z}{A(z+1)^p} \ge \frac1{2 A}$ for $z \ge 1$. Recall $\Phi(\vecW_t) \ge 1$ from \pref{rem:philowerboundremark}, because $t < \tau_{\cA_{\epsilon}, T}(\vecW_0)$ and so $\vecW_t \in \cA_{\epsilon}^c$. Thus, $\Phi(\vecW_t) \theta'\prn*{\Phi(\vecW_t)} \ge \frac1{2A}$. Therefore we can rearrange the above as 
\begin{align*}
\mathbb{E}_{\vecEps_t,\vecZ_t} \brk*{\theta(\Phi(\vecW_{t+1}))} &\le \theta\prn*{\Phi(\vecW_t)}- \eta  \theta'\prn*{\Phi(\vecW_t)} \cdot \lambda \Phi(\vecW_t)+ \frac12 \eta^2 \prn*{2\sigma_F^2+4\max(L,B)^2 \prn*{\nrm*{\vecW_t}^{2s}+1}}\\
&\hspace{1in}+C \eta^3 \prn*{8\sigma_F^3+16\max(L,B)^3 \prn*{\nrm*{\vecW_t}^{3s}+1}} + 2C\prn*{\frac{\eta d}{\beta}}^{3/2} \\
&\le \theta\prn*{\Phi(\vecW_t)}-\frac{\eta \lambda}{2A}+ \frac12 \eta^2 \prn*{2\sigma_F^2+4\max(L,B)^2 \prn*{\nrm*{\vecW_t}^{2s}+1}}\\
&\hspace{1in}+C \eta^3 \prn*{8\sigma_F^3+16\max(L,B)^3 \prn*{\nrm*{\vecW_t}^{3s}+1}} + 2C\prn*{\frac{\eta d}{\beta}}^{3/2} \\
&= \theta\prn*{\Phi(\vecW_t)}-\frac{\eta \lambda}{2A}+\text{err}(\vecW_t), \numberthis\label{eq:generalonesteprecursion}
\end{align*}
where we define
\[ \text{err}(\vecW) := \frac12 \eta^2 \prn*{2\sigma_F^2+4\max(L,B)^2 \prn*{\nrm*{\vecW}^{2s}+1}}+C \eta^3 \prn*{8\sigma_F^3+16\max(L,B)^3 \prn*{\nrm*{\vecW}^{3s}+1}} + 2C\prn*{\frac{\eta d}{\beta}}^{3/2} > 0.\]

Now with \pref{eq:generalonesteprecursion}, the idea is to sum and telescope this relations over $\tau_{\cA_{\epsilon},T+1}$ time steps, as discussed in \pref{subsec:strategy}. The way to do this is using discrete-time Dynkin's Formula, stated in Theorem 11.3.1 of \cite{meyn2012markov}:
\begin{theorem}[Theorem 11.3.1 of \cite{meyn2012markov}]\label{thm:discretetimedynkin}
Let $Z_t$ be any $\mathfrak{F}_t$-measurable function of $\vecW_0, \ldots, \vecW_t$. Consider any stopping time $\tau$ and define $\tau^n := \min\crl*{n, \tau, \inf\prn*{t \ge 0: \vecZ_t \ge n}}$. Then we have for all $n \ge 0$ and $\vecW_0 \in\mathbb{R}^d$ that 
\[ \mathbb{E}\brk*{Z_{\tau^n}}=\mathbb{E}\brk*{Z_{0}}+\mathbb{E}\brk*{\sum_{t=1}^{\tau^n} \prn*{\mathbb{E}\brk*{Z_t|\mathfrak{F}_t}- Z_{t-1}}}.\]
\end{theorem}
As a simple corollary of \pref{thm:discretetimedynkin}, we have the following, Proposition 11.3.2 of \cite{meyn2012markov}. Unlike the above, it holds for \textit{any stopping time}.
\begin{corollary}[Proposition 11.3.2 of \cite{meyn2012markov}]\label{corr:discretetimedynkinconvenient}
Suppose there exists non-negative functions $s_t, f_t$\footnote{The result in \cite{meyn2012markov} states this for positive $s_t,f_t$, but it is clear their proof still works when the functions are non-negative.} such that 
\[ \mathbb{E}\brk*{Z_{t+1} | \mathfrak{F}_t} \le Z_t - f_t\prn*{\vecW_t} + s_t\prn*{\vecW_t}.\numberthis\label{eq:onesteprecursionfordynkin}\]
Then for any $\vecW_0 \in\mathbb{R}^d$ and any stopping time $\tau$,
\[ \mathbb{E}\brk*{\sum_{t=0}^{\tau-1} f_t\prn*{\vecW_t}} \le Z_0 + \mathbb{E}\brk*{\sum_{t=0}^{\tau-1} s_t\prn*{\vecW_t}}.\]
\end{corollary}
Apply \pref{corr:discretetimedynkinconvenient} for the stopping time $\tau=\tau_{\cA_{\epsilon},T+1}$, $Z_t = \theta\prn*{\Phi\prn*{\vecW_t}}$, and the functions $f_t, s_t$ defined as follows. Take
\[ f_t(\vecW) = \begin{cases}\frac{\eta \lambda}{2A}&\IF \vecW \in \cA_{\epsilon}^c \\ 0&\OTHERWISE\end{cases}. \]
In the \GLDTEXTSPACE case take
\[ s_t(\vecW) = \begin{cases} \text{err}(\vecW)&\IF \vecW \in \cA_{\epsilon}^c \\ \mathbb{E}_{\vecEps, \vecZ}\brk*{\theta\prn*{\Phi\prn*{\vecW-\eta \grad F(\vecW)+\sqrt{\frac{2\eta}{\beta}} \vecEps}}}&\OTHERWISE\end{cases},\]
and in the \SGLDTEXTSPACE case take
\[ s_t(\vecW) = \begin{cases} \text{err}(\vecW)&\IF \vecW \in \cA_{\epsilon}^c \\ \mathbb{E}_{\vecEps, \vecZ}\brk*{\theta\prn*{\Phi\prn*{\vecW-\eta \grad f(\vecW;\vecZ)+\sqrt{\frac{2\eta}{\beta}} \vecEps}}}&\OTHERWISE\end{cases}.\]
where $\vecEps \sim \mathcal{N}(0,\bbI_d)$ and $\vecZ$ is an arbitrary data sample. Note the $\{f_t\}$, as well as the $\{s_t\}$, are the same function for all $t$. Since $\theta \ge 0$, the $f_t$ and $s_t$ are non-negative. As $\mathbb{E}_{\vecEps_t, \vecZ_t}\brk*{\cdot}$ is the same as $\mathbb{E}\brk*{\cdot | \mathfrak{F}_t}$, \pref{eq:generalonesteprecursion} proves that \pref{eq:onesteprecursionfordynkin} holds if $\vecW_t\in\cA_{\epsilon}^c$, and \pref{eq:onesteprecursionfordynkin} holds for $\vecW_t\in\cA_{\epsilon}$ as the $Z_t \ge 0$ and as the $s_t(\vecW_t)=\mathbb{E}\brk*{Z_{t+1}|\mathfrak{F}_t}$\footnote{But this is not relevant, since we apply \pref{corr:discretetimedynkinconvenient} with $\tau=\tau_{\cA_{\epsilon},T+1}$.}. Thus, \pref{corr:discretetimedynkinconvenient} yields
\[ \mathbb{E}\brk*{\sum_{t=0}^{\tau_{\cA_{\epsilon}, T}\prn*{\vecW_0}-1}\frac{\eta \lambda}{2A}}=\mathbb{E}\brk*{\sum_{t=0}^{\tau_{\cA_{\epsilon}, T}\prn*{\vecW_0}-1} f_t\prn*{\vecW_t}} \le Z_0 + \mathbb{E}\brk*{\sum_{t=0}^{\tau_{\cA_{\epsilon}, T}\prn*{\vecW_0}-1} s_t\prn*{\vecW_t}}=\theta\prn*{\Phi(\vecW_0)}+\mathbb{E}\brk*{\sum_{t=0}^{\tau_{\cA_{\epsilon}, T}\prn*{\vecW_0}-1} \text{err}(\vecW_t)},\]
since $\vecW_t \in \cA_{\epsilon}^c$ for all $t \le \tau_{\cA_{\epsilon}, T}\prn*{\vecW_0}-1$, and using the definition of $f_t, s_t$ in that case.

Clearly we can simplify the left hand side as $\frac{\eta \lambda}{2A} \mathbb{E}\brk*{\tau_{\cA_{\epsilon}, T}\prn*{\vecW_0}}$. For the right hand side, note pointwise we have $\sum_{t=0}^{\tau_{\cA_{\epsilon}, T}\prn*{\vecW_0}-1}\text{err}(\vecW_t) \le \sum_{t=0}^{T-1}\text{err}(\vecW_t)$ by definition of $\tau_{\cA_{\epsilon}, T}\prn*{\vecW_0}$ and as the $\text{err}(\vecW) \ge 0$. Moreover, all the relevant expectations are finite (by \pref{lem:smoothdissipativeinball} and as $\tau_{\cA_{\epsilon},T+1} \le T < \infty$). Therefore we see
\[ \frac{\eta \lambda}{2A} \mathbb{E}\brk*{\tau_{\cA_{\epsilon}, T}\prn*{\vecW_0}} \le \theta\prn*{\Phi(\vecW_0)}+\mathbb{E}\brk*{\sum_{t=0}^{T-1}\text{err}(\vecW_t)}.\]
We now show that the random variable $\tau_{\cA_{\epsilon}, T}$ is well-controlled.
\begin{lemma}\label{lem:expectedvaluesmall}
We have 
\[ \mathbb{E}\brk*{\tau_{\cA_{\epsilon}, T}\prn*{\vecW_0}} < \frac{T}{10}.\]
\end{lemma}
\begin{proof}
Suppose otherwise that $\mathbb{E}\brk*{\tau_{\cA_{\epsilon}, T}\prn*{\vecW_0}} \ge \frac{T}{10} > 0$. Rearranging the above gives 
\begin{align*}
\frac{\lambda}{2A} &\le \frac{\theta\prn*{\Phi(\vecW_0)}}{\eta \mathbb{E}\brk*{\tau_{\cA_{\epsilon}, T}\prn*{\vecW_0}}} + \frac1{\eta \mathbb{E}\brk*{\tau_{\cA_{\epsilon}, T}(\vecW_0)}} \mathbb{E}\brk*{\sum_{t=0}^{T-1}\text{err}(\vecW_t)} \\
&= \frac{\theta\prn*{\Phi(\vecW_0)}}{\eta \mathbb{E}\brk*{\tau_{\cA_{\epsilon}, T}\prn*{\vecW_0}}} + \frac1{\eta \mathbb{E}\brk*{\tau_{\cA_{\epsilon}, T}(\vecW_0)}}\sum_{t=0}^{T-1}\mathbb{E}\brk*{\text{err}(\vecW_t)}.\numberthis\label{eq:contradictionineq}
\end{align*}
By \pref{lem:smoothdissipativeinball}, which we may apply as our choice of $\eta$ is small enough, we have 
\[ \mathbb{E}\brk*{\nrm*{\vecW_t}^{2s}} \le L_2\max(\eta T, 1)^{s/2}, \mathbb{E}\brk*{\nrm*{\vecW_t}^{3s}}\le L_3\max(\eta T, 1)^{3s/4}.\]
Therefore,
\begin{align*}
\mathbb{E}\brk*{\text{err}(\vecW_t)} &\le \frac12 \eta^2 \prn*{2\sigma_F^2+4\max(L,B)^2 (L_2 \max(\eta T, 1)^{s/2}+1)}\\
&\hspace{1in}+C \eta^3 \prn*{8\sigma_F^3+16\max(L,B)^3 (L_3 \max(\eta T, 1)^{3s/4}+1)} + 2C\prn*{\frac{\eta d}{\beta}}^{3/2} \\
&\le M \prn*{ (\eta d / \beta)^{3/2}+ \eta^{2} \cdot (\eta T)^{s/2} + \eta^3  \cdot (\eta T)^{3s/4}}.
\end{align*}
The last line follows as $\eta T \ge 1$ and from definition of $M$ (recall we took $\sigma_F \leftarrow \max(\sigma_F,1)$ if necessary); recall 
\[ M = \max\prn*{\frac12, 2C} \cdot \prn*{8\sigma_F^3+16\max(L,B)^3 (\max\prn*{L_2,L_3}+1)}.\]
Recall our choice of $T$ such that $\eta T = 
\frac{C_0}{\lambda}$, and also our choice of $C_0 = 50A\theta(\Phi(\vecW_0)) \lor 1$. Therefore, \pref{eq:contradictionineq} becomes
\begin{align*}
\frac{\lambda}{2A} &\le \frac{\theta\prn*{\Phi(\vecW_0)}}{\eta \mathbb{E}\brk*{\tau_{\cA_{\epsilon}, T}\prn*{\vecW_0}}} + \frac1{\eta \mathbb{E}\brk*{\tau_{\cA_{\epsilon}, T}(\vecW_0)}}\sum_{t=0}^{T-1}\mathbb{E}\brk*{\text{err}(\vecW_t)} \\
&\le \frac{10\theta\prn*{\Phi(\vecW_0)}}{\eta T} + \frac{10}{\eta T} \cdot T \cdot M \prn*{ (\eta d / \beta)^{3/2}+ \eta^{2} \cdot (\eta T)^{s/2} + \eta^3  \cdot (\eta T)^{3s/4}} \\
&= 10\prn*{\frac{\theta(\Phi(\vecW_0)) \lambda}{C_0} + M \prn*{ \eta^{1/2} (d / \beta)^{3/2}+ \eta \cdot \frac{C_0^{s/2}}{\lambda^{s/2}} + \eta^2 \cdot \frac{C_0^{3s/4}}{\lambda^{3s/4}}}} \\
&< 10\prn*{\frac{\lambda}{40A} + MC_0 \prn*{ \eta^{1/2} (d / \beta)^{3/2}+ \frac{\eta}{\lambda^{s/2}} + \frac{\eta^{2}}{\lambda^{3s/4}}}} \\
&<10 \cdot \frac{\lambda}{20A} = \frac{\lambda}{2A}.
\end{align*}
In the second inequality we use $\mathbb{E}\brk*{\tau_{\cA_{\epsilon}, T}\prn*{\vecW_0}} \ge \frac{T}{10}$ which we are supposing for contradiction. The last ienquality uses 
\[ \eta \le \min\prn*{\frac{1}{120^2 A^2 B^2 M^2} \cdot \frac{\beta^3 \lambda^2}{d^3}, \frac{\lambda^{1+s/2}}{120AC_0 M}}. \]
Note as we have $\lambda \le 1$ and $A \ge 1$, $C_0 \ge 1$, $M \ge \frac12$, this implies 
\[ \frac{\lambda^{1+s/2}}{120AC_0 M} \le \frac{\lambda^{\frac12+\frac{3s}8}}{(120 AC_0 M)^{1/2}},\]
which we also use to show $MC_0 \cdot \frac{\eta^2}{\lambda^{3s/4}} \le \frac{\lambda}{120A}$. This yields contradiction, and so we have the Lemma.
\end{proof}

With \pref{lem:expectedvaluesmall}, the finish is straightforward. By Markov's Inequality, with probability at least 0.8, 
\[ \tau_{\cA_{\epsilon}, T}\prn*{\vecW_0} \le 5\mathbb{E}\brk*{\tau_{\cA_{\epsilon}, T}\prn*{\vecW_0}} < \frac{T}2.\]
However, $\tau_{\cA_{\epsilon}, T}\prn*{\vecW_0} < T$ implies $\tau_{\cA_{\epsilon}, T}\prn*{\vecW_0} = \tau_{\cA_{\epsilon}}\prn*{\vecW_0}$. Thus, with probability at least 0.8, we have $\tau_{\cA_{\epsilon}}\prn*{\vecW_0} < T$. That is, with probability at least 0.8 we hit $\cA_{\epsilon} = \{\vecW:F(\vecW) \le \epsilon\}$ within $T$ steps.

When $s=0$ and $\gamma=0$, we cannot use \pref{lem:smoothdissipativeinball} anymore. But just note whenever $s=0$, we can use the upper bound $\mathbb{E}\brk*{\nrm*{\grad F(\vecW_t)}^p} \le L^p \le L^3$ for $p=2,3$ in our upper bound of $\mathbb{E}_{\vecEps_t, \vecZ_t}\brk*{\theta\prn*{\Phi\prn*{\vecW_{t+1}}}}$. Defining instead
\[ \text{err}(\vecW) := \frac12 \eta^2 \prn*{2\sigma_F^2+4\max(L,B)^2}+C \eta^3 \prn*{8\sigma_F^3+16\max(L,B)^3} + 2C\prn*{\frac{\eta d}{\beta}}^{3/2} > 0,\]
we see the rest of the proof goes through the same, with no use of \pref{lem:smoothdissipativeinball}.

The tighter results in the case when $s=1$ (the $L$-smooth and $(m,b)$-dissipative setting) are also proved identically. They follow from plugging in the uniform moment bounds from \pref{lem:smoothdissipativeinball} rather than the general ones into the proof of \pref{lem:expectedvaluesmall}. Then, $L_2, L_3$ (which are different in this case) appear in the proof of \pref{lem:expectedvaluesmall} with no $\max(\eta T, 1)$ term present, and again we finish the same as above.
\end{proof}

\subsection{Details for Comparison to Literature}\label{subsec:poincareoptcomparisontoliteratureproofs}
Here, we discuss how we derived optimization results using sampling results from literature, that we discussed in \pref{sec:introduction}. As mentioned there, we assume an $O(1)$ warm-start for all of the literature, which is the least favorable for us. Consider as an example how we obtained results for SGLD the smooth and dissipative case from \citet{raginsky2017non}, \citet{xu2018global}, and \citet{zou2021faster}.

Theorem 1 of \citet{raginsky2017non} requires gradient noise $\delta$ to be exponentially small in $d$, which does not make sense (we only require gradient noise of constant order, which is more realistic). Theorem 3.6, Corollary 3.7, and Remark 3.9 of \citet{xu2018global} reports an iteration count of $K=\widetilde{O}\prn*{\frac{d}{\epsilon \lambda_{*}}}$ where $\lambda_{*}$ is spectral gap of the discrete-time Markov Chain given by \pref{eq:SGLDiterates}, however they do not count the iteration count $B$ to compute each stochastic gradient from $B$ data samples. Either they also require exponentially small gradient noise, or $B = \widetilde{O}\prn*{\frac{d^6}{\epsilon^4 \lambda_{*}^4}}$, and their total gradient complexity should be 
\[ K \cdot B = \widetilde{O}\prn*{\frac{d^7}{\epsilon^5 \lambda_{*}^5}}.\] 
Similarly, for the same paper's claimed runtime for Stochastic Variance Reduced Gradient Langevin Dynamics (SVRG-LD) in Theorem 3.10 and Corollary 3.11, noting the correct runtime should be $K \cdot B$, we obtain a runtime of
\[ \widetilde{O}\prn*{\frac{Ld^5}{\lambda_{*}^4 \epsilon^4}}\ge\widetilde{O}\prn*{\frac{d^5}{\lambda_{*}^4 \epsilon^4}}. \]
The last step simply follows from noting their $L \ge 1$, being the length of an inner loop. 

This accounting must also for the result Theorem 4.5 and Corollary 4.7 of \citet{zou2021faster}. Accounting for $K \cdot B$, they obtain a rate of at least $\widetilde{O}\prn*{\frac{d^4 \beta^2}{\rho^4 \epsilon^2}}$, where $\rho$ is the Cheeger constant of $\mu_{\beta}$, to obtain a TV distance of $\epsilon$ to the Gibbs measure. By Cheeger's Inequality, we have $\frac1{\rho^4} \ge \CPI(\mu_{\beta})^2$. 
However to convert from TV distance results to optimization results using Corollary 4.8 of their same paper \citet{zou2021faster}, we need a TV distance of $\frac{\epsilon}{d}$ (and this is necessary due to dissipativeness) to obtain an optimization result, which leads to additional dimension dependence. Combined with noting $\beta$ is (at least) on the same order as $\frac{d}{\epsilon}$ up to $\log$ factors, this gives a rate of at least
\[ \widetilde{O}\prn*{\frac{d^{8}\CPI(\mu_{\beta})^2}{\epsilon^4}}.\]
for optimizing $F$ to $\widetilde{O}\prn*{\frac{d}{\beta}+\epsilon} = \widetilde{O}\prn*{\epsilon}$ tolerance..

We now discuss how we obtained results from the rest of literature. Generally the rest of literature handles exact gradients and so does not have the problem of those above two works. One point of note is that in some of the sampling literature, such as \citet{vempala2019rapid,balasubramanian2022towards,huang2024faster}, sampling is done from $e^{-f}/Z$. That is, sampling is presumed to be done at constant temperature, a different setting than optimization. In our setting $f=\beta F$, and the smoothness parameter $L$ or condition number in these works is that of $f$. Thus their smoothness parameter $L$ scales like $\tilde{\Omega}\prn*{\frac{d}{\epsilon}}$. The rest of the rates from literature were then derived by converting KL divergence guarantees into TV distance guarantees via Pinkser's Inequality, and then using Corollary 4.8 of \citet{zou2021faster}, analogously to the above example. In more detail, by Pinkser's Inequality, if $F$ is $s$-H\"{o}lder continuous we need KL divergence to be at most $\frac{\epsilon^2}{d^{s+1}}$.

Following \pref{rem:cleversamplingtoopt}, it follows that the $\epsilon$ in the sampling results can be taken to be $\Theta(1)$. However, where $\epsilon$ denotes the desired optimization tolerance, the smoothness parameter $L$ still scales like $\tilde{\Omega}\prn*{\frac{d}{\epsilon}}$. Plugging in these choices, we obtained the results from \pref{sec:introduction}.

As another example, we mention how we derived a rate from Corollary 19 of \citet{balasubramanian2022towards} (which still requires exact knowledge of gradient) in the GLD, Poincar\'e, and Lipschitz case. Taking $s=0$ in Corollary 19 of \citet{balasubramanian2022towards}, and even supposing a warm start of $K_0 = O(1)$ is possible, we see they obtain a TV distance of $\sqrt{\epsilon}$ in 
\[ \widetilde{O}\prn*{\frac{\beta^6 d^{3}\CPI(\mu_{\beta})^3}{\epsilon^{5}}}.\]
However, since $F$ is Lipschitz, we require a TV distance of $\frac{\epsilon}{\sqrt{d}}$, the dimensionality again coming from Remark 4.6 of \citet{zou2021faster}. This yields a rate of 
\[ \widetilde{O}\prn*{\frac{\beta^6 d^{8}\CPI(\mu_{\beta})^3}{\epsilon^{10}}}.\]
We must have $\beta = \widetilde{\Omega}\prn*{\frac{d}{\epsilon}}$, so in this case this gives a rate of at least $\widetilde{O}\prn*{\frac{\CPI(\mu_{\beta})^3 d^{14}}{\epsilon^{16}}}$ for optimizing $F$ to $\widetilde{O}\prn*{\epsilon}$ tolerance. We can derive a faster rate from this result using \pref{rem:cleversamplingtoopt}, which is also mentioned in \pref{sec:introduction}.

Finally, we mention that we can compare the above results from \citet{zou2021faster}, \citet{chewi21analysis}, and \citet{balasubramanian2022towards} for general $\beta = \widetilde{\Omega}\prn*{\frac{d}{\epsilon}}$; as mentioned in \pref{sec:introduction}, to use the results of \citet{chewi21analysis} and \citet{balasubramanian2022towards} of optimization, their $\beta$ dependence will be their stated dependence on the smoothness parameter $L$. Our dependence on $\beta$ is always better than that of \citet{zou2021faster}, and using $\beta = \widetilde{\Omega}\prn*{\frac{d}{\epsilon}}$, we see for any such $\beta$ our dependence in all parameters is better than that of \citet{chewi21analysis, balasubramanian2022towards} when $s \le \frac12$.

\section{Additional Proofs}\label{sec:additionaldiscretetimeproofs}
\subsection{Potential Argument Details}
These are Lemmas from the proof of \pref{thm:gradexactFeps} deferred here for the ease of presentation.
\begin{lemma}\label{lem:Ytsupermartingalesimplesetting}
$Y_t$ is a supermartingale with respect to $\mathfrak{F}_t$.
\end{lemma}
\begin{proof}
This is obvious if $t+1 > \tau$ as then we take $Y_t=Y_{\tau}$. Else, suppose $t+1 \le \tau$. By \pref{lem:onesteprecursion}, we have the inequality 
\begin{align*}
\mathbb{E}_{\vecEps_t}\brk*{\theta\prn*{\Phi(\vecW_{t+1})} |\mathfrak{F}_{t}}\le \theta\prn*{\Phi(\vecW_t)} - \eta F_{\epsilon}(\vecW_t) + R(\vecW_t, \Phi, \eta, \beta, d).
\end{align*}
This means, since $t+1 \le \tau$,
\begin{align*}
\mathbb{E}_{\vecEps_t}\brk*{Y_{t+1}|\mathfrak{F}_{t}} &= \mathbb{E}_{\vecEps_t}\brk*{\theta\prn*{\Phi(\vecW{}_{t+1})}|\mathfrak{F}_{t}}+\sum_{j=0}^t \prn*{\eta F_{\epsilon}(\vecW_j) - R(\vecW_j, \Phi, \eta, \beta, d)}\\
&\le \theta\prn*{\Phi(\vecW_t)} - \eta F_{\epsilon}(\vecW_t) + R(\vecW_t, \Phi, \eta, \beta, d)+\sum_{j=0}^t \prn*{\eta F_{\epsilon}(\vecW_j) - R(\vecW_j, \Phi, \eta, \beta, d)}\\
&= Y_t,
\end{align*}
proving this part.
\end{proof}
\begin{lemma}\label{lem:azumahoeffding}
With probability at least $1-\delta$, we have
\[ Y_t - Y_0 \le \sqrt{\frac12 \left(\sum_{t=0}^{T-1} C(\eta, t, d, \beta)^2\right)\log(T/\delta)}\]
for all $1 \le t \le T$, where 
\[C(\eta, t, d, \beta) := 4\sqrt{\theta\prn*{\rho_{\Phi}^{-1}\prn*{\kappa' \rho_{\Phi}\prn*{\Phi(\vecW_0)}}}} \cdot \nrm*{-\eta \nabla F(\vecW_t)+\sqrt{\frac{2\eta}{\beta}} \vecEps_t}+ 4\nrm*{-\eta \nabla F(\vecW_t)+\sqrt{\frac{2\eta}{\beta}} \vecEps_t}^2.\]
\end{lemma}
\begin{proof}
We aim to apply Azuma-Hoeffding. When $t > \tau$ then $Y_{t+1}-Y_t=0$, so suppose $t \le \tau$ in the following. Define $C_t = \eta F_{\epsilon}(\vecW_t) - R(\vecW_t, \Phi, \eta, \beta, d)$ which is $\mathfrak{F}_t$-measurable and note
\[ Y_{t+1}-Y_t = \theta\prn*{\Phi(\vecW_{t+1})}-\theta\prn*{\Phi(\vecW_t)} + C_t.\]
Let's now bound $\theta\prn*{\Phi(\vecW_{t+1})}-\theta\prn*{\Phi(\vecW_t)}$ from above and below. The idea here is to not Taylor expand to third order but rather second order to obtain simpler estimates; we used Taylor expansion to third order in \pref{lem:onesteprecursion} to use the admissibility condition, but to establish these bounds we don't need said condition. This is a very similar strategy as in the proof of Lemma 11 of \citet{priorpaper}.

For an upper bound, applying the second part of \pref{lem:thirdordersmoothfromregularity} with $\vecW = \vecW_t$ and $\mathbf{u} = -\eta \nabla F(\vecW_t)+\sqrt{2\eta/\beta} \vecEps_t$ gives by definition of $\vecW_{t+1}-\vecW_t$ that
\begin{align*}
\theta\prn*{\Phi(\vecW_{t+1})}-\theta\prn*{\Phi(\vecW_t)} &\le \theta'\prn*{\Phi(\vecW_t)}\tri*{ \nabla \Phi(\vecW_t), -\eta \nabla F(\vecW_t)+\sqrt{2\eta/\beta} \vecEps_t}\\
&\hspace{1in}+\frac12 \nrm*{-\eta \nabla F(\vecW_t)+\sqrt{2\eta/\beta} \vecEps_t}^2 \\
&\le \theta'\prn*{\Phi(\vecW_t)}\nrm*{ \nabla \Phi(\vecW_t)} \nrm*{-\eta \nabla F(\vecW_t)+ \sqrt{2\eta/\beta}\vecEps_t} \\
&\hspace{1in}+\frac12 \nrm*{-\eta \nabla F(\vecW_t)+\sqrt{2\eta/\beta} \vecEps_t}^2\\
&\le \sqrt{2\theta\prn*{\Phi(\vecW_{t})}}\nrm*{-\eta \nabla F(\vecW_t)+ \sqrt{2\eta/\beta}\vecEps_t}\\
&\hspace{1in}+\frac12 \nrm*{-\eta \nabla F(\vecW_t)+\sqrt{2\eta/\beta} \vecEps_t}^2.
\end{align*}
The last step uses that from \pref{lem:thirdordersmoothfromregularity},
\[ \nrm*{\nabla \Phi(\vecW_{t})}\theta'\prn*{\Phi(\vecW_{t})} \le \rho_2\prn*{\Phi(\vecW_{t})}\theta'\prn*{\Phi(\vecW_{t})} \sqrt{2\theta\prn*{\Phi(\vecW_{t})}} \le \sqrt{2\theta\prn*{\Phi(\vecW_{t})}},\]
as $\rho_2(z) \le \rho(z)$ always holds for $z \ge 0$ and by definition of $\theta'=\frac1{\rho}$. 
This upper bound on $\theta\prn*{\Phi(\vecW_{t+1})}-\theta\prn*{\Phi(\vecW_t)}$ is clearly $\mathfrak{F}_t$ measurable. 

Similarly, again applying the second part of \pref{lem:thirdordersmoothfromregularity} with $\vecW = \vecW_{t+1}$ and $\vecU = \eta \nabla F(\vecW_t)-\sqrt{2\eta/\beta} \vecEps_t$ gives 
\begin{align*}
\theta\prn*{\Phi(\vecW_t)}-\theta\prn*{\Phi(\vecW_{t+1})} &\le \theta'\prn*{\Phi(\vecW_{t+1})} \tri*{ \nabla \Phi(\vecW_{t+1}), \eta \nabla F(\vecW_t)-\sqrt{2\eta/\beta} \vecEps_t} \\
&\hspace{1in}+\frac12 \nrm*{-\eta \nabla F(\vecW_t)+\sqrt{2\eta/\beta} \vecEps_t}^2 \\
&\le \theta'\prn*{\Phi(\vecW_{t+1})} \nrm*{\nabla \Phi(\vecW_{t+1})}\nrm*{-\eta \nabla F(\vecW_t)+\sqrt{2\eta/\beta} \vecEps_t} \\
&\hspace{1in}+\frac12 \nrm*{-\eta \nabla F(\vecW_t)+\sqrt{2\eta/\beta} \vecEps_t}^2.
\end{align*}
To make this $\mathfrak{F}_t$ measurable and relate this to the upper bound by employing a similar strategy as in the proof of Lemma 11 of \citet{priorpaper}. Again we use from \pref{lem:thirdordersmoothfromregularity} that
\[ \nrm*{\nabla \Phi(\vecW_{t+1})}\theta'\prn*{\Phi(\vecW_{t+1})} \le \rho_2\prn*{\Phi(\vecW_{t+1})}\theta'\prn*{\Phi(\vecW_{t+1})} \sqrt{2\theta\prn*{\Phi(\vecW_{t+1})}} \le \sqrt{2\theta\prn*{\Phi(\vecW_{t+1})}}.\]
Since $\sqrt{a+b} \le \sqrt{\abs{a}+b} \le \sqrt{a}+\sqrt{b}$ for all reals $a$ and $b\ge 0$ with $a+b\ge 0$, we obtain
\begin{align*}
\nrm*{\nabla \Phi(\vecW_{t+1})}\theta'\prn*{\Phi(\vecW_{t+1})} &\le   \sqrt{2\theta\prn*{\Phi(\vecW_{t+1})}}\\
&= \sqrt{2}\cdot \sqrt{\theta\prn*{\Phi(\vecW_{t+1})}-\theta\prn*{\Phi(\vecW_t)}+\theta\prn*{\Phi(w_t)}} \\
&\le \sqrt{2}\prn*{\sqrt{\theta\prn*{\Phi(w_t)}}+\sqrt{\abs*{\theta\prn*{\Phi(\vecW_{t+1})}-\theta\prn*{\Phi(\vecW_t)}}}}.
\end{align*}
Using this we have
\begin{align*}
\theta\prn*{\Phi(\vecW_{t+1})}-\theta\prn*{\Phi(w_t)} &\ge -\theta'\prn*{\Phi(\vecW_{t+1})} \nrm*{\nabla \Phi(\vecW_{t+1})}\nrm*{-\eta \nabla F(\vecW_t)+\sqrt{2\eta/\beta} \vecEps_t}\\
&\hspace{1in}-\frac12\nrm*{-\eta \nabla F(\vecW_t)+\sqrt{2\eta/\beta} \vecEps_t}^2 \\
&\ge -\sqrt{2\abs*{\theta\prn*{\Phi(\vecW_{t+1})}-\theta\prn*{\Phi(\vecW_t)}}}\nrm*{-\eta \nabla F(\vecW_t)+\sqrt{2\eta/\beta} \vecEps_t}\\
&\hspace{1in}-\sqrt{2\theta\prn*{\Phi(\vecW_t)}}\nrm*{-\eta \nabla F(\vecW_t)+\sqrt{2\eta/\beta} \vecEps_t}\\
&\hspace{1in}-\frac12\nrm*{-\eta \nabla F(\vecW_t)+\sqrt{2\eta/\beta} \vecEps_t}^2. 
\end{align*}
By AM-GM we have
\begin{align*}
&\sqrt{2\abs*{\theta\prn*{\Phi(\vecW_{t+1})}-\theta\prn*{\Phi(\vecW_t)}}}\nrm*{-\eta \nabla F(\vecW_t)+\sqrt{2\eta/\beta} \vecEps_t} \\
\le &\frac{\abs*{\theta\prn*{\Phi(\vecW_{t+1})}-\theta\prn*{\Phi(\vecW_t)}}+2\nrm*{-\eta \nabla F(\vecW_t)+\sqrt{2\eta/\beta} \vecEps_t}^2}{2}.
\end{align*}
Using this gives
\begin{align*}
\theta\prn*{\Phi(\vecW_{t+1})}-\theta\prn*{\Phi(\vecW_t)} &\ge-\frac{\abs*{\theta\prn*{\Phi(\vecW_{t+1}}-\theta\prn*{\Phi(\vecW_t)}}}{2}\\
&\hspace{1in}-\sqrt{2\theta\prn*{\Phi(\vecW_t)}}\nrm*{-\eta \nabla F(\vecW_t)+\sqrt{2\eta/\beta} \vecEps_t}\\
&\hspace{1in}-\frac32\nrm*{-\eta \nabla F(\vecW_t)+\sqrt{2\eta/\beta} \vecEps_t}^2.
\end{align*}
Doing cases on the sign of $\theta\prn*{\Phi(\vecW_{t+1})}-\theta\prn*{\Phi(\vecW_t)}$, we get that in all cases
\begin{align*}\theta\prn*{\Phi(\vecW_{t+1})}-\theta\prn*{\Phi(\vecW_t)} &\ge -2\sqrt{2\theta\prn*{\Phi(\vecW_t)}}\nrm*{-\eta \nabla F(\vecW_t)+\sqrt{2\eta/\beta} \vecEps_t}\\
&\hspace{1in}-3\nrm*{-\eta \nabla F(\vecW_t)+\sqrt{2\eta/\beta} \vecEps_t}^2.\end{align*}
This yields a lower bound on $\theta\prn*{\Phi(\vecW_{t+1})}-\theta\prn*{\Phi(\vecW_t)}$ that is $\mathfrak{F}_t$-measurable. 

Now, to finish the setup for the concentration bound via Azuma-Hoeffding, we just need to upper bound the difference between these bounds. The above shows that this difference is at most
\begin{align*}
\theta\prn*{\Phi(\vecW_{t+1})}-\theta\prn*{\Phi(\vecW_t)} &\le 4\sqrt{\theta\prn*{\Phi(\vecW_t)}} \nrm*{-\eta \nabla F(\vecW_t)+\sqrt{2\eta/\beta} \vecEps_t}+ 4\nrm*{-\eta \nabla F(\vecW_t)+\sqrt{2\eta/\beta} \vecEps_t}^2\\
&\le4\sqrt{\theta\prn*{\rho_{\Phi}^{-1}\prn*{\kappa' \rho_{\Phi}\prn*{\Phi(\vecW_0)}}}} \nrm*{-\eta \nabla F(\vecW_t)+\sqrt{2\eta/\beta} \vecEps_t}\\
&\hspace{1in}+ 4\nrm*{-\eta \nabla F(\vecW_t)+\sqrt{2\eta/\beta} \vecEps_t}^2.
\end{align*}
The last step follows since $t \le \tau$ and since $\rho_{\Phi}$ is increasing, we have by definition of $\tau$ that
\[ \Phi(\vecW_t) \le \rho_{\Phi}^{-1}\prn*{\kappa' \rho_{\Phi}\prn*{\Phi(\vecW_0)}}. \]
Now by \pref{lem:thirdordersmoothfromregularity}, we know $\theta$ is increasing so
\[ \sqrt{\theta\prn*{\Phi(\vecW_t)}} \le \sqrt{\theta\prn*{\rho_{\Phi}^{-1}\prn*{\kappa \rho_{\Phi}\prn*{\Phi(\vecW_0)}}}}.\]
As defined earlier, we denote the above expression by $C(\eta, t, d, \beta, \rho) \ge 0$ for convenience. This serves as an upper bound regardless of whether $t \le \tau$ by the initial discussion. Thus, applying Azuma Hoeffding, which we may apply by \pref{lem:Ytsupermartingalesimplesetting} and since $Y_{t+1}-Y_t=\theta\prn*{\Phi\prn*{\vecW_{t+1}}}-\theta\prn*{\Phi\prn*{\vecW_t}}$, gives with probability at least $1-\frac{\delta}T$ we have that
\[ Y_t - Y_0 \le \sqrt{\frac12 \prn*{\sum_{t=0}^{t-1} C(\eta, t, d, \beta, \rho)^2}\log\prn*{T/\delta}},\]
and we conclude via Union Bound.
\end{proof}
\subsection{Additional Helper Results}
Here we establish many of the results we used in the main discretization proofs.

\begin{lemma}\label{lem:upperboundgradFholder}
Suppose $F$ satisfies \pref{ass:holderF}. Then for all $\vecW \in \mathbb{R}^d$, 
\[ \nrm*{\grad F(\vecW)} \le L\max(1,\nrm*{\vecW^{\star}})^s \prn*{\nrm*{\vecW}^s+1},\]
where $\vecW^{\star}$ is any global minima of $F$. Moreover, if \pref{ass:stochasticgradcontrol} holds, the above also holds for the stochastic gradient estimates $\nrm*{\grad f(\vecW;\vecZ)}$.
\end{lemma}
\begin{proof}
Note $\grad F(\vecW^{\star})=0$. By Triangle Inequality and \pref{ass:holderF},
\begin{align*}
\nrm*{\grad F(\vecW)} &= \nrm*{\grad F(\vecW)-\grad F(\vecW^{\star})} \\
&\le L\nrm*{\vecW-\vecW^{\star}}^s\\
&\le L\prn*{\nrm*{\vecW}+\nrm*{\vecW^{\star}}}^s \\
&\le L\max(1,\nrm*{\vecW^{\star}})^s \prn*{\nrm*{\vecW}^s+1}.
\end{align*}
The last two steps used the following elementary inequalities:
\[ (az+b)^s \le \max(a,b)^s (z+1)^s \FORALLTEXT a, b, z\ge 0.\]
\[(z+1)^{1/s'} \le z^{1/s'} + 1 \iff z+1 \le (z^{1/s'} + 1)^{s'} \FORALLTEXT s' \ge 1.\]
The extension to stochastic gradients given \pref{ass:stochasticgradcontrol} is immediate.
\end{proof}

The following result is used to control the values of $F$ using \pref{ass:holderF}.
\begin{lemma}\label{lem:upperboundFholder}
Suppose $F$ satisfies \pref{ass:holderF}. Then for all $\vecW \in \mathbb{R}^d$, 
\[ F(\vecW) \le L\nrm*{\vecW-\vecW^{\star}}^{s+1}.\]
\end{lemma}
\begin{proof}
The proof is very similar to Lemma 3.4 of \citet{bubeck2015convex}. Let $\vecW^{\star}$ be any global minima of $F$, thus $F(\vecW^{\star})=0$ and $\grad F(\vecW^{\star})=0$. We see from calculus and Cauchy-Schwartz that
\begin{align*}
F(\vecW)&=\abs*{F(\vecW)-F(\vecW^{\star}) - \tri*{\grad F(\vecW^{\star}), \vecW-\vecW^{\star}}} \\
&= \abs*{\int_{t=0}^1 \tri*{\grad F(\vecW^{\star}+t(\vecW-\vecW^{\star})) - \grad F(\vecW^{\star}), \vecW-\vecW^{\star}} \DERIV t} \\
&\le \abs*{\int_{t=0}^1 \nrm*{\grad F(\vecW^{\star}+t(\vecW-\vecW^{\star})) - \grad F(\vecW^{\star})}\nrm*{\vecW-\vecW^{\star}} \DERIV t} \\
&\le \abs*{\int_{t=0}^1 Lt^s \nrm*{\vecW-\vecW^{\star}}^s\nrm*{\vecW-\vecW^{\star}} \DERIV t } \\
&\le L\nrm*{\vecW-\vecW^{\star}}^{s+1},
\end{align*}
where we apply Cauchy-Schwartz to obtain the first inequality and \pref{ass:holderF} for the second.
\end{proof}

We also need the following simple integral to prove \pref{lem:measureoflargeF}.
\begin{lemma}\label{lem:extensionofgaussianintegral}
We have for any $0 \le s \le 1$ and $M \ge 0$ that 
\[ \int_{\mathbb{R}^d} e^{-M\nrm*{\vecW}^{s+1}} \DERIV \vecW=\frac{2\pi^{d/2}}{\Gamma(d/2)} \cdot \frac1{s+1} \cdot M^{-\frac{d}{s+1}}\cdot \Gamma\prn*{\frac{d}{s+1}}.\]
\end{lemma}
\begin{proof}
The surface area of $\cS^{d-1}$ is $\frac{2\pi^{d/2}}{\Gamma(d/2)}$, which scales by $r^{d-1}$ for an arbitrary radius $r$. Consider partitioning $\mathbb{R}^d$ into spheres of radius $r$: upon making this change of variables, which formally is $\DERIV \vecW = \frac{2\pi^{d/2}}{\Gamma(d/2)} r^{d-1} \DERIV r$, we obtain
\[ \int_{\mathbb{R}^d} e^{-M\nrm*{\vecW}^{s+1}} \DERIV \vecW=\frac{2\pi^{d/2}}{\Gamma(d/2)}\int_{0}^\infty e^{-Mr^{s+1}} r^{d-1} \DERIV r. \]
Let $u = r^{s+1}$, therefore $r = u^{\frac1{s+1}}$ and $\DERIV r = \frac1{s+1} u^{-\frac{s}{s+1}} \DERIV u$. Thus
\begin{align*}
\int_{\mathbb{R}^d} e^{-M\nrm*{\vecW}^{s+1}} \DERIV \vecW&=\frac{2\pi^{d/2}}{\Gamma(d/2)} \cdot \frac1{s+1} \int_{0}^\infty e^{-M u} u^{\frac{d-1-s}{s+1}} \DERIV u \\
&= \frac{2\pi^{d/2}}{\Gamma(d/2)} \cdot \frac1{s+1} \cdot M^{-\frac{d}{s+1}}\cdot \Gamma\prn*{\frac{d}{s+1}}.
\end{align*}
Here, the last equality is a well known integral essentially following from definition of the Gamma function, specifically
\[ M^{-t} \Gamma(t) = \int_0^{\infty} e^{-Mu} u^{t-1} \DERIV u.\]
It follows since $d \ge 1 \ge s$ and $s \ge 0$, hence $\frac{d-1-s}{s+1}=\frac{d}{s+1}-1 \ge -1$, so we may apply these results regarding the Gamma function.
\end{proof}

The last lemma is used to upper bound $z^p+1$ for all $z \ge 0$ and any $p \ge 0$.
\begin{lemma}\label{lem:upperboundpower}
For all $z \ge 0$ and any $p \ge 0$, $z^p+1 \le 2(z+1)^p$. 
\end{lemma}
\begin{proof}
First suppose $p \ge 1$. Here we show $z^p+1 \le (z+1)^p$, which clearly suffices. Letting $f(z) = (z+1)^p-(z^p+1)$, we see $f'(z) \ge 0$ always. Therefore $f(z) \ge f(0)=0$, proving this case.

Now suppose $0 \le p < 1$. Let $f(z)=\frac{(z+1)^p}{z^p+1}$. Then, 
\[ f'(z) = \frac{p(z+1)^{p-1} \cdot (z^p+1)-(z+1)^p \cdot pz^{p-1}}{(z^p+1)^2}=\frac{p(z+1)^{p-1}\prn*{1-z^{p-1}}}{(z^p+1)^2}.\]
Therefore $f'(z) \le 0$ for $z \in [0,1]$ and $f'(z) \ge 0$ for $z \in [1,\infty)$, so $f(z)$ is minimized on $[0,\infty)$ when $z=1$. Hence, $f(z) \ge f(1) = 2^{p-1}$. Thus, $z^p+1 \le 2^{1-p} (z+1)^p \le 2(z+1)^p$ as desired.
\end{proof}

\clearpage 

\end{document}